\documentclass[twoside,11pt]{article}

\usepackage[abbrvbib, preprint]{jmlr2e}

\usepackage{amsmath, amssymb, amsfonts, amsbsy,  latexsym, epsfig}
\DeclareMathOperator{\Tr}{tr}
\usepackage{MnSymbol}
\usepackage{subfigure}
\usepackage{xcolor}
\usepackage{url}

\usepackage{jmlr2e}

\definecolor{darkblue}{rgb}{0.0, 0.0, 0.55}

\newcommand{\cL}{\mathcal{L}}

\newcommand{\RR}{\mathbb{R}}

\usepackage{lastpage}

\ShortHeadings{Good regularity creates large learning rate implicit biases}{Yuqing Wang, Zhenghao Xu, Tuo Zhao, Molei Tao}
\firstpageno{1}

\begin{document}

\title{Good regularity creates large learning rate implicit biases: edge of stability, balancing, and catapult}

\author{\name Yuqing Wang \email ywang3398@gatech.edu\\ 
       \addr School of Mathematics
       \AND
       \name Zhenghao Xu \email zhenghaoxu@gatech.edu\\
       \addr H. Milton Stewart School of Industrial \& Systems Engineering
        \AND
        \name Tuo Zhao \email tourzhao@gatech.edu \\
       \addr H. Milton Stewart School of Industrial \& Systems Engineering
        \AND
        Molei Tao \email mtao@gatech.edu \\
        \addr School of Mathematics\\
        Georgia Institute of Technology\\
        Atlanta, GA 30332}

\editor{My editor}

\newcommand{\fix}{\marginpar{FIX}}
\newcommand{\new}{\marginpar{NEW}}

\def\month{MM}  
\def\year{YYYY} 
\def\openreview{\url{https://openreview.net/forum?id=XXXX}} 


\maketitle

\begin{abstract}
    Large learning rates, when applied to gradient descent for nonconvex optimization, yield various implicit biases including the edge of stability \citep{cohen2021gradient}, balancing \citep{wang2022large}, and catapult \citep{lewkowycz2020large}. 
    These phenomena cannot be well explained by classical optimization theory. Though significant theoretical progress has been made in understanding these implicit biases, it remains unclear for which objective functions would they be more likely.
    This paper provides an initial step in answering this question and also shows that these implicit biases are in fact various tips of the same iceberg. To establish these results, we develop a global convergence theory under large learning rates, for a family of nonconvex functions without globally Lipschitz continuous gradient, which was typically assumed in existing convergence analysis. Specifically, these phenomena are more likely to occur when the optimization objective function has good regularity. This regularity, together with gradient descent using a large learning rate that favors flatter regions, results in these nontrivial dynamical behaviors.      
    Another corollary is the first non-asymptotic convergence rate bound for large-learning-rate gradient descent optimization of nonconvex functions.
    Although our theory only applies to specific functions so far, the possibility of extrapolating it to neural networks is also experimentally validated, for which different choices of loss, activation functions, and other techniques such as batch normalization can all affect regularity significantly and lead to very different training dynamics.
\end{abstract}
\section{Introduction}

Large learning rates are often employed in deep learning practices, which is believed to improve training efficiency and generalization \citep{seong2018towards,lewkowycz2020large,yue2023salr,smith2018disciplined}. However, a comprehensive theoretical foundation supporting such a choice has been largely absent, as a large body of the theory literature primarily focused on small or infinitesimal learning rates (gradient flow), leaving the empirical successes achieved with large learning rates unexplained. Only recently have researchers redirected their efforts toward investigating the benefits of larger learning rates.

A prevalent conjecture posits that large learning rates result in a preference toward `flat minima' which in turn leads to better generalization. This notion has served as a catalyst for many novel insights examining the `sharpness' of the solution under large learning rates, i.e., the largest eigenvalue of the Hessian matrix associated with the objective function. A notable example can be found in \citet{cohen2021gradient}, where the phenomenon of the \emph{Edge of Stability} (EoS) is examined in detail (see Section~\ref{sec:eos}). Specifically, they demonstrate that during the training of neural networks using Gradient Descent (GD), a stage of progressive sharpening transpires prior to the stabilization of sharpness near $2/h$,
where $h$ is the Learning Rate. Furthermore, another significant example comes from the study by \citet{wang2022large}, unveiling a \emph{balancing} phenomenon (see Section~\ref{sec:balancing}). This phenomenon pertains to the convergence of weights in two-layer linear diagonal networks, showcasing a tendency for the weights to approach each other as compared to their initial values. This is in stark contrast to the behavior observed under small learning rates, where weight discrepancies remain relatively constant \citep{NEURIPS_Du2018algorithmic,ye2021global,ma2021beyond}. Consequently, large learning rates bias GD towards flat minima, even when the optimization is initiated arbitrarily close to a sharp minimum (see Corollary 3.3 in \citet{wang2022large} and Section~\ref{sec:balancing}). It is worth noting that the initial phase of the balancing effect gives rise to \emph{loss catapulting} \citep{lewkowycz2020large}, another intriguing empirical observation—during this phase, the loss experiences an initial increase (GD's escaping from the neighborhood of a sharp minimum) before subsequently decreasing. Collectively, these phenomena will be referred to as {\it large learning rate phenomena} thereafter.

Remarkable theoretical studies have been conducted to analyze EoS and balancing concerning specific objective functions \citep[for example,][]{chen2023edge,macdonald2023progressive,ahn2022learning,damian2022self,zhu2023understanding,wang2022analyzing,kreisler2023gradient,arora2022understanding,ahn2022understanding,lyu2022understanding,song2023trajectory}. Nevertheless, the following questions remain not yet fully resolved:
\begin{center}
    {\it When and why do EoS and other large learning rate phenomena occur?}
\end{center}
These questions serve as the primary motivation for our paper. Specifically, we generalize the objectives examined in recent theoretical work, for example, \citet{wang2022large,zhu2023understanding,ahn2022learning}, and construct a family of functions that can be compared as fairly as possible, despite their different regularities. This allows us to investigate the comparative question of which function makes EoS more likely. Beyond that, we also use these functions to demonstrate that many large learning rate phenomena are implicit biases rooted in a unifying mechanism.
While still being simplified models, they are designed to encapsulate key features of practical neural networks, including the composition of loss and activation functions, and are made amenable to quantitative analysis.
Within this family, we then prove that the combined influence of good regularity and the preferences of large-learning-rate GD toward flatter minima leads to balancing, and makes EoS more likely. These rigorous results lead us to conjecture that regularity will similarly impact large learning rate phenomena for more general functions, and this conjecture is tested positively on empirical experiments with neural networks.

Before further explanation of our results, let us begin with some visualization of our claims. Figure~\ref{fig:intuition_phenomenoa_appears} illustrates multiple instances of the aforementioned large learning rate phenomena, including the balancing phenomenon (Figure~\ref{fig:intuition_phenomenoa_appears}(b)), the catapult phenomenon in the loss decay (Figure~\ref{fig:intuition_phenomenoa_appears}(a)), and the edge of stability (EoS) shown in Figure~\ref{fig:intuition_phenomenoa_appears}(c) and (d). It is important to note that large learning rates do not always lead to these phenomena. Figure~\ref{fig:intuition_phenomenoa_disappears} depicts a scenario where we use an objective function with the same global minima but worse regularity compared to Figure~\ref{fig:intuition_phenomenoa_appears}. In this case, neither EoS nor balancing (including catapult) occurs, despite the use of large learning rates. The landscapes of these two functions are visualized in Figure~\ref{fig:heatmap_good_bad}.

\begin{figure}[htb!]
    \centering
    \includegraphics[width=\textwidth]{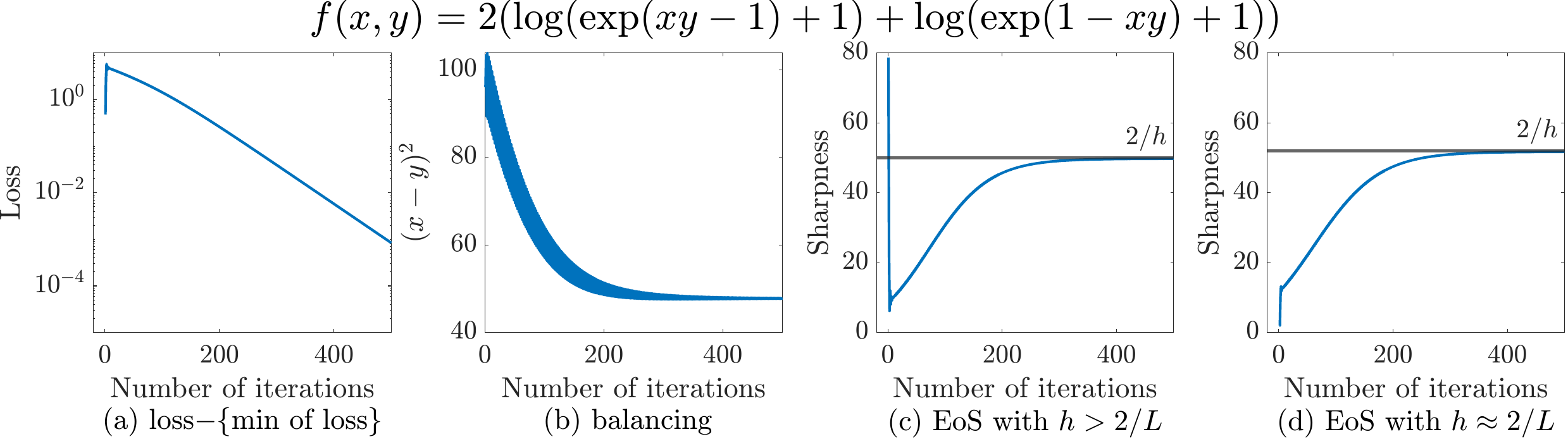}
\caption{\it Large learning rate phenomena. The learning rate is chosen to be $h=4/(x_0^2+y_0^2)$ for all the figures. (a)(b)(c) share the same initial condition $x=0.2,y=10$, while for (d), $x=2,y=10$. (a) shows the decay of loss, i.e., $f-\min f$, and exhibits a catapult phenomenon at the beginning of the iterations. (b) shows the balancing phenomenon. (c) and (d) illustrate EoS under different initial conditions: if GD starts at large sharpness (c), the learning rate $h>2/L$, where $L$ is the local Lipschitz constant of the gradient in a bounded region where the GD trajectory belongs to, and there is a sharp decay at the beginning of the iterations before progressive sharpening; if it starts at small sharpness, the learning rate $h\approx 2/L$, and there is no de-sharpening stage (see more explanation in Section~\ref{sec:eos}).}    \label{fig:intuition_phenomenoa_appears}
\end{figure}

\begin{figure}
    \centering
    \includegraphics[width=0.75\textwidth]{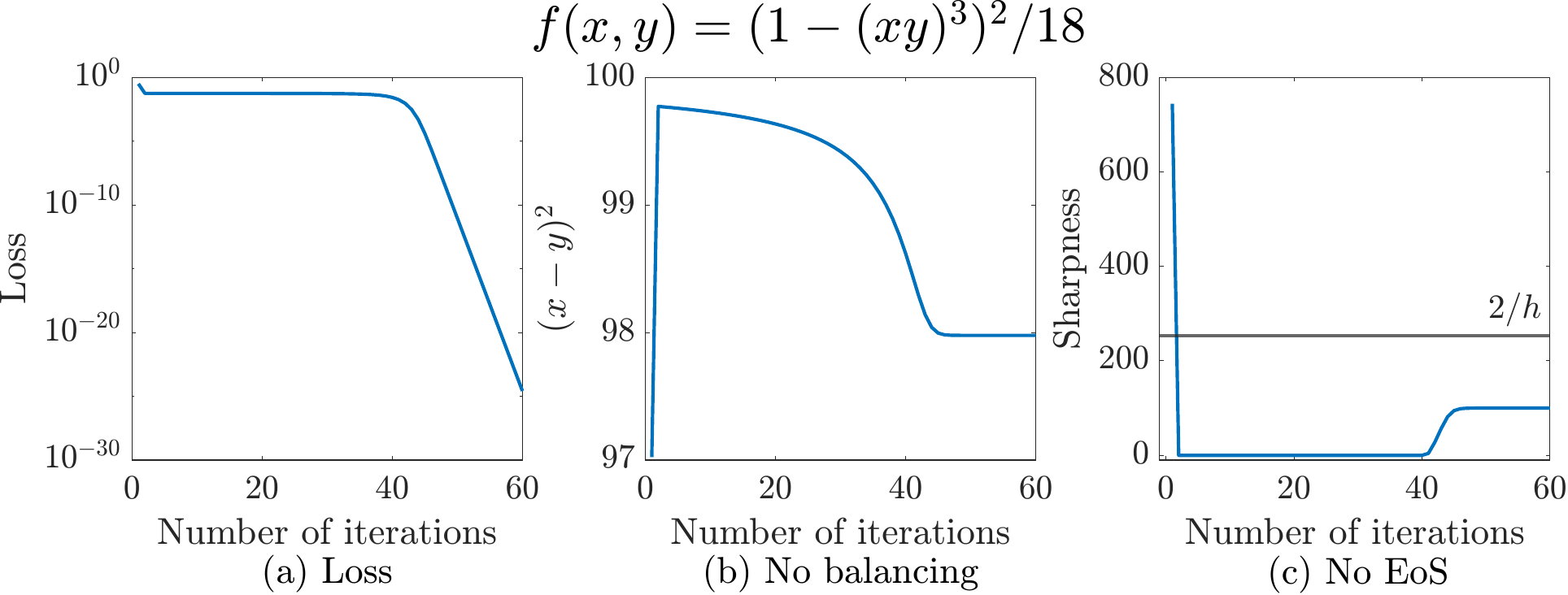}
    \caption{\it Disappearance of large learning rate phenomena.  
    All the three figures share the same initial condition $x_0=0.15,y_0=10$ with the learning rate to be $h=4/(x_0^2+y_0^2)/(x_0y_0)^{2b-2}$. It can be seen from (c) that the learning rate is indeed large, $>2/L$, where $L$ is the local Lipschitz constant of this trajectory. However, none of the large learning rate phenomena appear, i.e., no EoS, no balancing, and no catapult of loss.
    }
    \label{fig:intuition_phenomenoa_disappears}
\end{figure}
\begin{figure}[ht]
    \centering
    \subfigure[values of a good regularity function]{\includegraphics[width=0.234\textwidth]{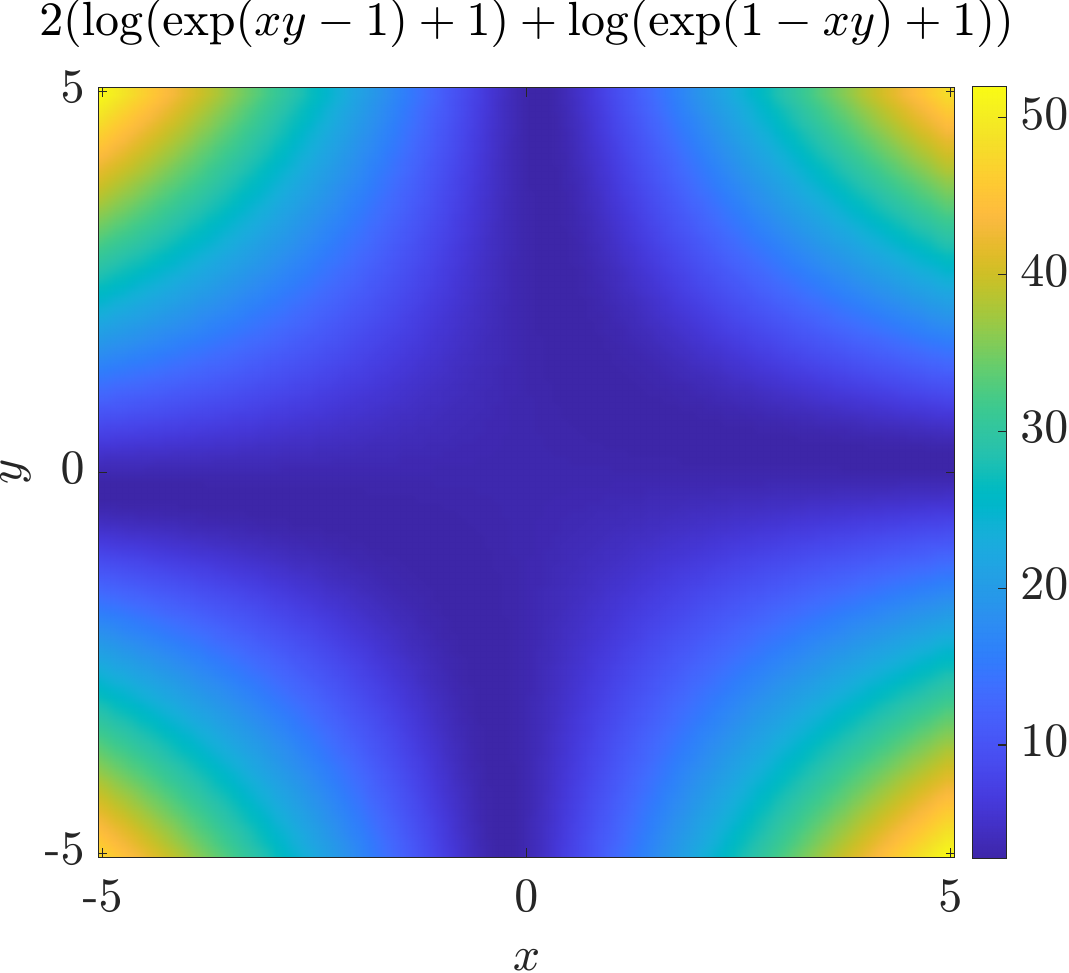}}
    \subfigure[values of a bad regularity function]{\includegraphics[width=0.24\textwidth]{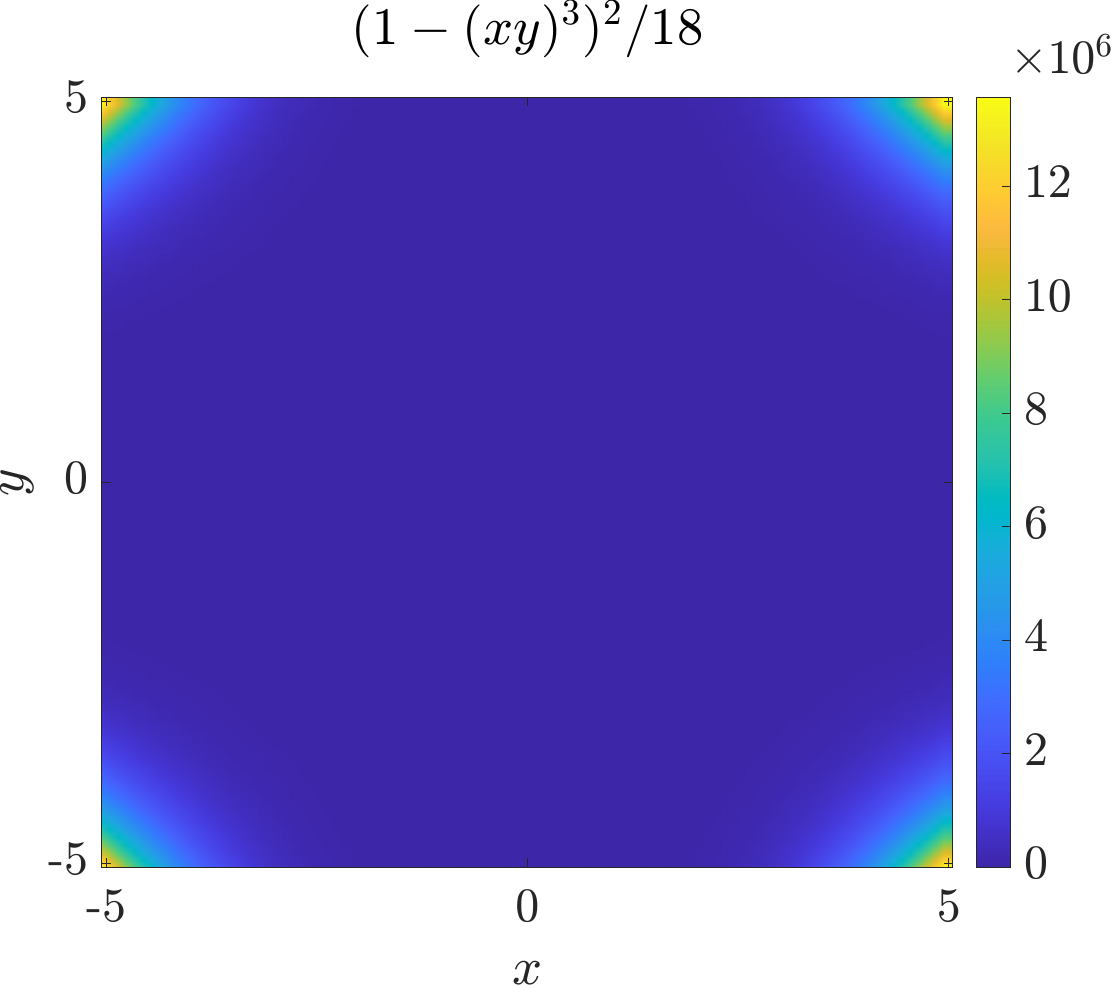}}
    \subfigure[sharpnesses of the good regularity function]{\includegraphics[width=0.23\textwidth]{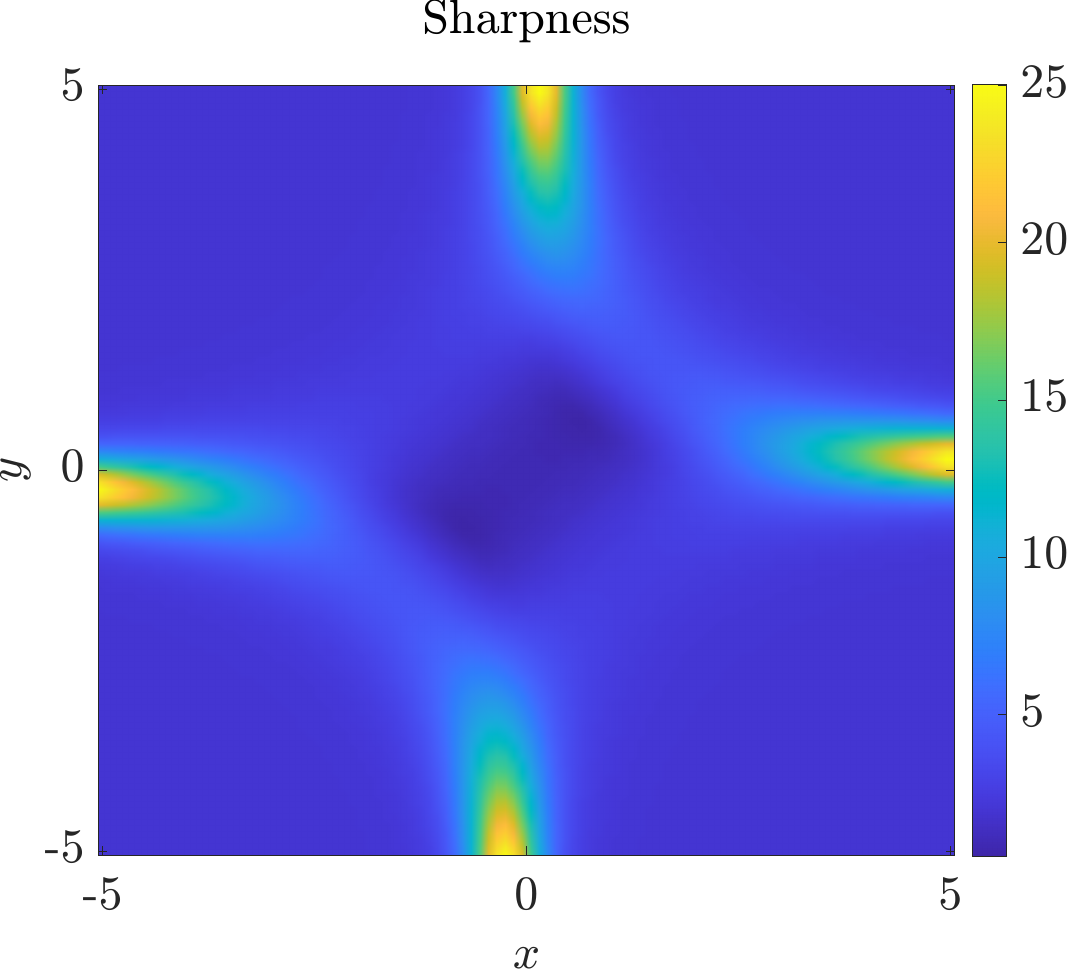}}
    \subfigure[sharpnesses of the bad regularity function]{\includegraphics[width=0.24\textwidth]{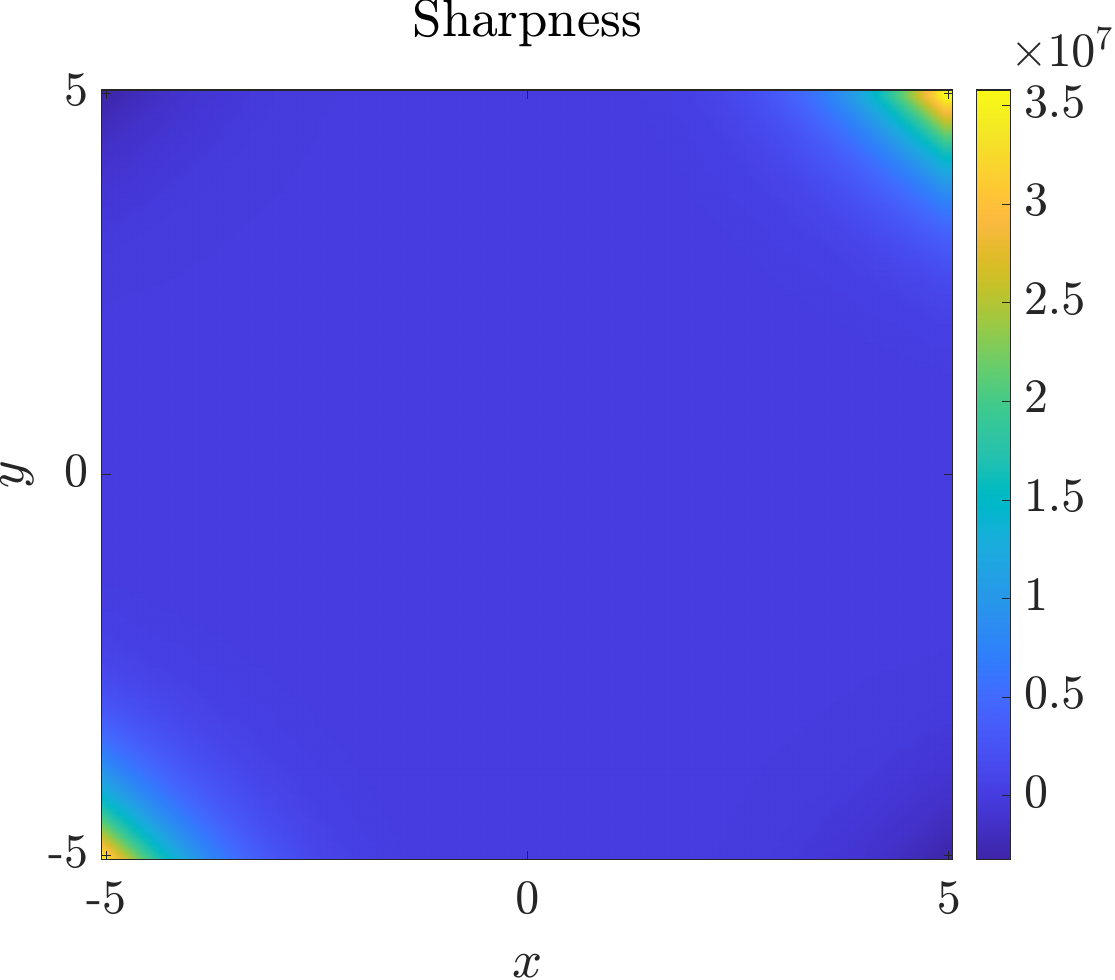}}
    \caption{The landscapes and sharpnesses of good and bad regularity functions; different colors indicate different values. (a)(b) are landscapes of good and bad regularity functions. For (a), the value of the good regularity function changes slowly, resulting in a small range on the color bar. In contrast, the value of a bad regularity function (b) changes significantly in the region away from the minima, while the changes near the minima can be ignored. As a result, the color bar exhibits a very large range, making it difficult to distinguish most of the regions on this heat map. (c)(d) are sharpnesses of the two functions. For (c), the sharpness of the good regularity function concentrates on the minima $xy=1$ while for other regions, the sharpness is very small. In contrast, for (d), the sharpness varies a lot outside the neighborhood of the minima.}
    \label{fig:heatmap_good_bad}
\end{figure}

Now we are ready to provide slightly more detailed, but still high-level answers to the two key questions:
\begin{itemize}
    \item \textbf{About `when'.} Within our designed family of functions, we rigorously demonstrate that large learning rate phenomena depend on the regularity of the objective function. Roughly speaking, when the objective has good regularity, i.e., small degrees of regularity (See Definition~\ref{def:degree_of_regularity} in Section \ref{sec:preliminary}), EoS and balancing are more prone to appear. However, when functions have bad regularity (large degrees of regularity), these two phenomena are more likely to vanish. 
    More precisely, whether there exists EoS and/or balancing is not a yes/no question. Instead, it is a quantitative one, focusing on how large is the set of initial conditions that exhibit such large learning rate phenomena.  
    As the degree of regularity increases, the set of the initial conditions that does not lead to these phenomena will expand, whose boundary moves in accordance with $2^{\frac{1}{b-1}}$, where $2b$ is the degree of regularity of the objective function.
    
    \item \textbf{About `why'.} Our theoretical analysis reveals that a crucial characteristic of these phenomena is the ability of large learning rates to steer GD toward flatter regions. In other words, the sharpness of the optimization landscape along GD iterations is influenced and controlled by how large the learning rate is.  In fact, upon the convergence of GD, which is nontrivial due to the large learning rate but proved as a groundwork for our theory, the limiting maximum Hessian eigenvalue is at most $\tilde{\mathcal{O}}(h)$ below $2/h$ for functions of good regularity, i.e., 
    the stability limit of sharpness $\leq 2/h$ is actually attained. But for bad regularity functions, this limit will be bounded by $1/h$ instead of the standard stability limit $2/h$. We name this nontrivial phenomenon as \emph{one-sided stability}. 
    Moreover, when the objective function has good regularity, GD is guaranteed to
    enter a region with sharpness less than the stability limit (the entry marks the onset of progressive sharpening), and then `crawl' up in sharpness till around $2/h$ (i.e., edge of stability).
\end{itemize}

Furthermore, we demonstrate the practical implications of our theory in deep learning through neural network models. Different choices of loss functions and/or activation functions will lead to different regularities of the training objective functions. This will result in very different training dynamics with or without EoS and balancing, which may also lead to different generalization performances.
Additionally, we consider a practical training technique---batch normalization \citep{ioffe2015batch} and show its capability of enhancing model regularity. When applied to bad regularity cases, batch normalization leads to the reemergence of the large learning rate phenomena. Collectively, these experiments consistently reinforce the significance of regularity in accordance with our theoretical framework.

We also remark on the nontriviality of our convergence analysis under large learning rates for a family of non-convex functions without globally Lipschitz continuous gradients. Especially, to the best of our knowledge, our work contains the first non-asymptotic result to quantify the rate of convergence of GD under large learning rates for nonconvex and non-Lipschitz-gradient functions (see Theorem~\ref{thm:bad_regularity_convergence}).

\vskip3pt
\noindent$\bigstar$ \textbf{Main contributions} of our paper are summarized in the following, except the technical contributions, which will be deferred to Section~\ref{sec:discussions}.
\begin{itemize}
    \item We demonstrate conditions for the occurrence of Edge of Stability and balancing within a class of functions.
    \item For bad regularity functions, we demonstrate the disappearance of EoS and a new, alternative phenomenon called \emph{one-sided stability}; we also quantify the disappearance of balancing. This new phenomenon is different from the results of stability theory. 
    \item We unify these phenomena and show they are different aspects of the same dynamical mechanism associated with the large learning rate.
    \item We provide the first rigorous bound of the global convergence rate of large-learning-rate GD for a family of non-convex functions without globally Lipschitz gradients.
\end{itemize}

This paper is organized as follows: Section~\ref{subsec:related_works} contains the related works, basic notations, and what is a large learning rate regime. Section~\ref{sec:preliminary} includes the notion degree of regularity, and the family of objective functions considered in Section~\ref{subsec:regularity_Functions}, as well as a discussion of these concepts for neural networks with and without batch normalization in Section~\ref{subsec:regularity_neural_network}. Section~\ref{sec:eos} and~\ref{sec:balancing} explain, respectively, the reasons for the occurrence of EoS and balancing, together with intuitions of GD's mechanisms influenced by both the regularity and the large learning rate in Section~\ref{sec:eos}. Section~\ref{sec:convergence} is about the asymptotic and non-asymptotic convergence results of GD under large learning rates for the family of non-convex functions. Section~\ref{sec:experiments} contains the validation of our theory in neural network applications. Section~\ref{sec:discussions} contains additional discussions such as the limitations and the innovations of our theory.

\section{Background}
\label{subsec:related_works}

We first briefly review related works and then introduce notations and some preliminaries on gradient descent.

\subsection{Related works on large learning rate phenomena in training/optimization}
$\bigstar$ {\it Balancing.} Prior to the theory of the balancing phenomenon in \cite{wang2022large}, there were already a considerable amount of interesting theoretical results studying the balanced manifold under gradient flow. For example, in deep fully connected feedforward networks, the difference between the squares of Frobenius norms of weight matrices at different layers is preserved if they are `trained' by gradient flow, which is the zero learning rate limit of deterministic gradient descent (i.e., no minibatch, no momentum) \citep{NEURIPS_Du2018algorithmic,arora2018optimization}. In order to gain further theoretical understanding, the community then simplified the model to the matrix factorization problem, i.e., $\min_{X,Y\in\mathbb{R}^{n\times m}}\|A-XY^\top\|^2$. \citet{NEURIPS_Du2018algorithmic} 
proved the convergence of GD under diminishing learning rate, which is an approximation of gradient flow,  and thus $\|X\|_{\rm F}^2-\|Y\|_{\rm F}^2$ is almost preserved throughout the GD iteration. 
\citet{ye2021global} later extended the convergence result to the constant small learning rate. 

In contrast to the regime of these infinitesimal or small learning rates, \citet{wang2022large} proved that large learning rate ($>2/L$ and can be approximately $4/L$, where $L$ is the local Lipschitz constant of the gradient) can disrupt such preservation and steer the GD trajectory towards a more balanced minimum. They coined this phenomenon as `\emph{balancing}'. Later on, \cite{chen2023edge} extended the scope of the large learning rate, and proved, by assuming all trajectories are positive, the convergence to an orbit of period 2 for scalar case, where almost perfect balancing can be achieved. They also explored periodic orbits in various settings, including a two-layer single-neuron network and the general matrix factorization problem.
Balancing will also be a central focus of this work; however, we will not view it merely as a large learning rate effect anymore; instead, we will elucidate that another factor is also important, namely the regularity of the objective function. This way, balancing can be understood for a larger class of objective functions.

\vskip2pt
\noindent$\bigstar$ {\it Edge of Stability.}  In a seminal paper where stability analysis of (S)GD was conducted, \citet{NEURIPS2018_6651526b} also observed that GD (SGD with full batch gradient) in certain experiments approximately achieves its largest stable sharpness, i.e., $2/h$. A comprehensive empirical characterization of this phenomenon, termed Edge of Stability (EoS), was first made by \citet{cohen2021gradient} with full batch gradient descent. Later, EoS was also observed in, for example, adaptive gradient methods with sufficiently large batch size \citep{cohen2022adaptive}, and reinforcement learning \citep{iordan2023investigating}.

The intricacy and importance of EoS led to, in addition to empirical results, a rapid accumulation of theoretical investigations, e.g., \citep{chen2023edge,macdonald2023progressive,ahn2022learning,damian2022self,zhu2023understanding,wang2022analyzing,kreisler2023gradient,arora2022understanding,ahn2022understanding,lyu2022understanding,wu2023implicit,song2023trajectory}. 
\citet{zhu2023understanding} for example theoretically prove the $\approx 2/h$ limiting sharpness based on the convergence of GD for a specific problem. More precisely, they considered an objective function, which in the settings of this paper corresponds to a relatively large regularity ($b=2$ and $\mathrm{dor}=4$; see~\eqref{eqn:functions} and Def.~\ref{def:degree_of_regularity}), and
initialized very close to the set of minimizers. 
Their setup successfully reproduced the later stage of EoS, i.e., the sharpness converging to $2/h$, but not necessarily the earlier stage, progressive sharpening. 
Note that such initialization is the `good' region of the bad (large) regularity function, which can still lead to  EoS (see more details in Section~\ref{sec:discussions}) and our theory does not contradict their result.
\citet{wang2022analyzing} adopted a two-layer linear network model, and used the norm of the output layers as an indicator for sharpness. They analyzed a four-phase\footnote{Their four phases correspond to more detailed divisions of the two stages of progressive sharpening and stabilization at $2/h$, and thus are different from our three stages of the EoS phenomenon (no de-sharpening, i.e., stage 1) and the two-phase convergence under large learning rates (no phase 1). See Section~\ref{sec:eos}.} 
behavior of the output norm, which corresponds to a detailed division of progressive sharpening and limiting stabilization of EoS, including the increase/decrease of sharpness above $2/h$. 
\citet{damian2022self} studied a similar four-phase behavior in EoS by leveraging third-order derivative information of the loss function.
\citet{ahn2022learning} analyzed a class of loss functions, corresponding to our setting with $\mathrm{dor}=1$, and provided a detailed estimate of the order of difference ($\mathcal{O}(h^{1/(\beta-1)})$ with $\beta> 1$) between limiting sharpness and $2/h$ for their loss functions. In this paper, we focus more on the reason why EoS occurs and design the objective functions such that this difference is upper bounded by $\mathcal{O}(h)$ when EoS appears, and $\mathcal{O}(h^{-1})$ otherwise. A follow-up work is \citet{song2023trajectory}, which studied two-layer linear network and nonlinear network using $\ell^2$ loss and activation with $\mathrm{dor}=0$ (output data is 0) from the bifurcation perspective, i.e., GD converging to a periodic orbit, and proved the limiting sharpness ($\approx 2/h$) as well as a later part of progressive sharpening. \citet{kreisler2023gradient} considered a local convergence of deep linear networks (scalar products) and proposed an approach to analyze GD via evolving gradient flow at each step. \citet{lyu2022understanding} studied scale-invariant loss ($\mathrm{dor}=0$), analyzed a local convergence of GD to the $\approx 2/h$ limit of its sharpness under normalization layers and weight decay, and showed that GD approximates a sharpness-reduction flow. 
While these results are undoubtedly remarkable, the effects of the large learning rate, especially in the early stage of iterations, and how it globally creates the three stages altogether, are not fully demonstrated. This will be one of the main focuses of this paper and a major factor contributing to the occurrence of EoS. 

\vskip2pt
\noindent$\bigstar$ {\it More results related to large learning rate.} Apart from balancing and EoS, there is another phenomenon known as catapult \citep{lewkowycz2020large} that will also be considered here. There, GD with $>2/L$ learning rate perceives an early phase catapulting of the loss. In addition to these three phenomena, the impact of a large learning rate manifests in numerous other studies as well. Examples of empirical works include:  
\citet{andriushchenko2023sgd} observed that a large learning rate will bias SGD towards sparse solutions; \citet{pmlr-v139-jastrzebski21a} showed that in the early phase, SGD with a large learning rate penalizes the trace of the Fisher Information Matrix, which helps generalization.

Large learning rates can lead to complicated training dynamics, very different from that of gradient flow which only characterizes the infinitesimal learning rate regime. The effects of small-to-intermediate learning rates can be partially characterized via the tool of modified equation \citep{li2019stochastic}. However, as the learning rate continues to grow, the modified equation will stop to provide a good approximation \citep{kong2020stochasticity}, and interesting implicit biases such as those discussed here can manifest. When the learning rate becomes even bigger, the optimization dynamics can begin to invalidate the naive understanding that GD iterations either converge to a (fixed) point or diverge, even if the learning rate is fixed as a constant. The iterations may converge to higher-order dynamical structures beyond a fixed point. For example, \citet{kong2020stochasticity} studied the case where the objective function is $L=\mathcal{O}(1/\epsilon)$-smooth but the learning rate is $o(1) \gg 2/L$, for which the convergence to a chaotic attractor is proved, via a route of period-doubling bifurcations, i.e., GD first convergent to 1-periodic orbit (i.e., fixed point), then 2-periodic orbit, 4-, 8-, and all the way to $\infty$-period; based on this chaotic convergence pattern they constructed a quantitative theory of local minima escape without stochastic gradients. \citet{chen2023edge} proved the convergence to a 2-periodic orbit with a larger learning rate for specific 1D and 2D functions under specific assumptions. \citet{song2023trajectory} also studied the convergence to a 2-periodic orbit.

Apart from constant large learning rate, another line of research studies the acceleration of convergence to the minimum by using cyclic learning rates, which combine small learning rate leading to slow convergence with very large learning rate leading to divergence \citep[e.g.,][]{young1953richardson,altschuler2018greed,oymak2021provable,grimmer2023provably,altschuler2023accelerationI,altschuler2023accelerationII}. 

\subsection{Related works on flatness and generalization} 
Training with large learning rate GD often empirically leads to better generalization. A popular belief is that this is a consequence of GD's preference for flat minima under large learning rates \citep[e.g.,][]{seong2018towards,lewkowycz2020large,yue2023salr,smith2018disciplined}. The fact that large learning rates lead to flat minima has been theoretically confirmed and quantified, for example, by \citet{NEURIPS2018_6651526b,wang2022large}.
Theoretical explorations of flat minimum leading to better generalization have also been fruitful, including those based on PAC-Bayes bounds \citep{langford2001not,dziugaite2017computing}, information theory \citep{hinton1993keeping,hochreiter1997flat,negrea2019information}, and Rademacher complexity \citep[e.g.,][]{pmlr-v162-nacson22a,gatmiry2023inductive}, 
although
alternative/more precise statements also exist \citep[e.g.,][]{dinh2017sharp,andriushchenko2023modern,wen2023sharpness}. Also worth mentioning is the recently proposed but very popular approach of sharpness-aware minimization \citep{foret2021sharpnessaware}, which explicitly seeks smaller sharpness for better generalization.

\subsection{Notations} 
\label{subsec:notation}

Gradient descent (GD) iteration for solving a general nonconvex optimization problem
$\min_u f(u)$
is given by
$
    u_{k+1}=u_k-h\nabla f(u_k),
$
where $h>0$ is the learning rate (a.k.a. step size).

For problems considered in this article (see \eqref{eqn:functions} for more details), $u$ contains two components, $x$ and $y$, and thus we will also use $(x_k,y_k)$ to denote the $k$th iterate of GD. Let $S_k$ denote the largest eigenvalue magnitude of Hessian $\nabla^2f$ evaluated at the $k$th iteration $(x_k,y_k)$. It serves as a quantification of the local geometry of the landscape and will be referred to as the sharpness (value). We use  $(x_\infty,y_\infty)$ and $S_\infty$ to denote the limiting point of GD iteration and the limiting sharpness respectively.

We denote $L$ to be the Lipschitz constant of the gradient in the bounded region where the GD trajectory lies. We use $L_{\rm glo}$ to denote the global Lipzchitz constant of the gradient. We use $L_0$ and $L_\infty$ to represent the local Lipschitz constant of the gradient in the neighborhood of the initial point $(x_0,y_0)$ and the limiting point $(x_\infty,y_\infty)$ respectively. 

 We follow the theoretical computer science convention and use $\mathcal{O}(\cdot)$ to indicate the order of a quantity and 
 $\tilde{\mathcal{O}}(\cdot)$ for its order omitting logarithmic dependence. 
 Although we consider large learning rate $h$ throughout this paper, note that the term `large' is a relative concept based on local geometry, and more precisely, $h$ depends on $x_0,y_0$ (see detailed bounds of $h$ in Section~\ref{sec:convergence}). Therefore the term $\mathcal{O}(h)$ is taken in the sense of $x_0^2+y_0^2\to\infty$. We use $\lesssim$ such that $g_1\lesssim g_2$ means $g_1\le g_2+c\frac{\log(x_0^2+y_0^2)}{x_0^2+y_0^2}$ for some universal constant $c>0$. 

Finally, given two matrices $A\text{ and }B$, $A\preceq B$ ($A\prec B$) means $B-A$ is positive semi-definite (positive definite). $\|\cdot\|_{\rm F}$ denotes matrix Frobenius norm.

\subsection{Learning rate in gradient descent}
\label{subsec:learning_rate}

What learning rate should be used in gradient descent? 2nd-order derivative information of $f$, if available, is always helpful. More precisely, if the gradient $\nabla f$ is globally $L_{\rm glo}$-Lipschitz, then $L_{\rm glo}$ plays an important role in determining/interpreting the learning rate to use. In fact, a well-known nonconvex optimization result states the following (see, e.g., more detailed discussion in \cite{wang2022large} Appendix F)

\begin{theorem}
\label{thm:2/h_converge_stationary_point}
If $f:\RR^N\to\RR$ is twice differentiable and $\nabla f$ is $L_{\rm glo}-$Lipschitz, i.e., $-L_{\rm glo}I\preceq\nabla^2f(u)\preceq L_{\rm glo}I$ for all $u\in\RR^N$,
then with $h<\frac{2}{L_{\rm glo}}$, GD converges to a stationary point, i.e., a point where $\nabla f$ vanishes.
\end{theorem}

This theorem
is based on the global Lipschitzness of the gradient. However, objective functions in machine learning often lack \emph{globally} Lipschitz continuous gradients. One strategy is to assume that the trajectory (i.e., $u_0,u_1,\cdots$) is located in a bounded region, in which the \emph{local} Lipschitz constant (defined as the maximum eigenvalue magnitude of the Hessian $\nabla^2 f$ over the region) is upper bounded by some constant $L$. Altogether, 
it is most common to consider the following learning rate regime:
\begin{align*}
   \text{Regular learning rate regime: } 0<h<\frac{2}{L},
\end{align*}
where the convergence of GD can be well understood.
There are finer divisions of this regime with various behavior of GD, e.g., the infinitesimal learning rate regime which approximates the gradient flow, and a relatively large regime $\frac{1}{L}\le h< \frac{2}{L}$, where GD may oscillate near the minima until convergence.

Nevertheless, Theorem~\ref{thm:2/h_converge_stationary_point} only demonstrates a sufficient condition for convergence under nonconvex settings. A learning rate larger than $2/L$ may also lead to the convergence of GD.
This case will be the main focus of this paper and will be referred to as:
\begin{align*}
  \text{Large learning rate regime: }  \frac{2}{L}\lesssim h\lesssim \frac{C}{L}\text{ for some }C>2.
\end{align*}

There is in general no guarantee of convergence in this regime. In certain situations, using a very large learning rate may still result in convergence, but not necessarily to a stationary point (e.g., \citep{kong2020stochasticity,chen2023edge}); however, this article considers cases where GD does converge to a minimizer (i.e., a stationary point), and we will show that
there is actually such a regime, for which  $C\approx 4$. This regime is where many of the interesting large learning rate phenomena occur.

\section{Degrees of Regularity}
\label{sec:preliminary}

Before we proceed with our main results, we first introduce a new family of functions and how we measure their regularities.

\subsection{A family of functions with different degrees of regularity}
\label{subsec:regularity_Functions}

 Apart from the large learning rate, the regularity of the objective function also impacts large learning rate phenomena. Proposing and demonstrating this is one of the main contributions of this article. To this end, we first define the following notion in order to quantify the effect of regularity later.

\begin{definition}
    [Degree of regularity]
    \label{def:degree_of_regularity}

    Given function $F(s):\mathbb{R}\to\mathbb{R}$, its degree of regularity is
    \begin{align}   \label{eqn:degree_of_regularity} 
    \mathrm{dor}(F)=\inf \big\{n: \left|F(s)\right|\le C_1 |s|^n, \text{ for } |s|\ge C_0 \text{ with some constant }C_0,C_1>0\big\}.
    \end{align}
\end{definition}


\begin{remark}
 One might wonder about the rationale behind our terminology. The intuition is that if a function $F$ has a high dor, its $(k+1)$th derivative can possibly be much larger than its $k$th derivative at large enough $s$ (e.g., think about a polynomial). This blowing-up phenomenon through differentiations can persist for $0\leq k \ll \mathrm{dor}(F)$, i.e., for many orders of derivatives. Therefore, the function has poor regularity. Similarly, low dor is a qualitative\footnote{Note: not quantitative, as our explanations were just intuitions. Rigorous minds can just consider `dor' as a name.} indicator of good regularity.
 \end{remark}


A preview of our results is that when $\mathrm{dor}(F)$ is small, it corresponds to good regularity functions in our function class, and we will show the large learning rate implicit biases of EoS and balancing are more likely to appear. In contrast, large $\mathrm{dor}(F)$ suggests bad regularity and may lead to the disappearance of these phenomena. See details in Section~\ref{sec:eos} and~\ref{sec:balancing}. The following remark clarifies the relation between dor, regularity, and large learning rate phenomena.

\begin{remark}
We clarify that the above definition may not be the most general condition since we don't yet have a theory that can analyze any function. Degree of regularity (dor) will only serve as an initial attempt to understand how our specific family of functions used in \eqref{eqn:functions} corresponds to implicit biases of large learning rate. It may appear to be more like a growth condition used in various fields, such as Banach algebra \citep{stampfli1968growth}, optimization \citep{drusvyatskiy2018error}, and partial differential equations \citep{marcellini2021growth}. It also resembles the idea of generalized smoothness condition in non-convex optimization; see for example \citet{li2023convex,zhang2019gradient}. However, this is not the first time when growth condition is related to the regularity of functions; see, for example, \citet{giaquinta1987growth,giaquinta1988quasiconvexity}\footnote{Our definition of degree of regularity is still different from the relation between growth condition and regularity in these papers. Nevertheless, this is our initial attempt and is sufficient for our family of functions.}. Between growth condition and regularity, 
we conjecture that regularity (in the sense of Sobolev) is a more intrinsic factor of the implicit biases considered here.

Note dor is a global property and it is our intention to have a global condition for regularity. This is because we conjecture that local behaviors of a function, such regularities near the minimizer or at infinity, are insufficient to guarantee large learning rate phenomena considered in this article. For example,
Figure~\ref{fig:regularity} shows a perturbed version of a good regularity function $f(x,y)$ with $a=1$ in~\eqref{eqn:functions}, denoted as $f_{\rm pert}(x,y)$. As is shown in the figure, when far away from or very close to the minima, it overlaps with the good regularity function, sharing exactly the same growth condition both globally towards infinity and locally at the minimizer. Nevertheless, the Sobolev norm increases after the perturbation, i.e.,  $\|f(x,y)\|_{W^{k,p}}<\|f_{\rm pert}(x,y)\|_{W^{k,p}}$ for Sobolev space $W^{k,p}$ with $k\in\mathbb{N}$ and $1\le p<\infty$. This means the regularity of the function becomes worse. As numerically evidenced, 
EoS can still disappear. 

In addition, the notion of EoS considered in this article refers to the entire process, including not only the final stage of sharpness stabilization around $2/h$, but also earlier preparational stages. We will use this terminology, unless EoS is otherwise specified to be the final stage itself. That is another reason why the global behavior of the objective function matters as GD can navigate across a large region throughout the full process.

\begin{figure}[ht]
\centering    \subfigure[$F(s)$ vs $F_{\rm pert}(s)$]{\includegraphics[width=0.26\textwidth]{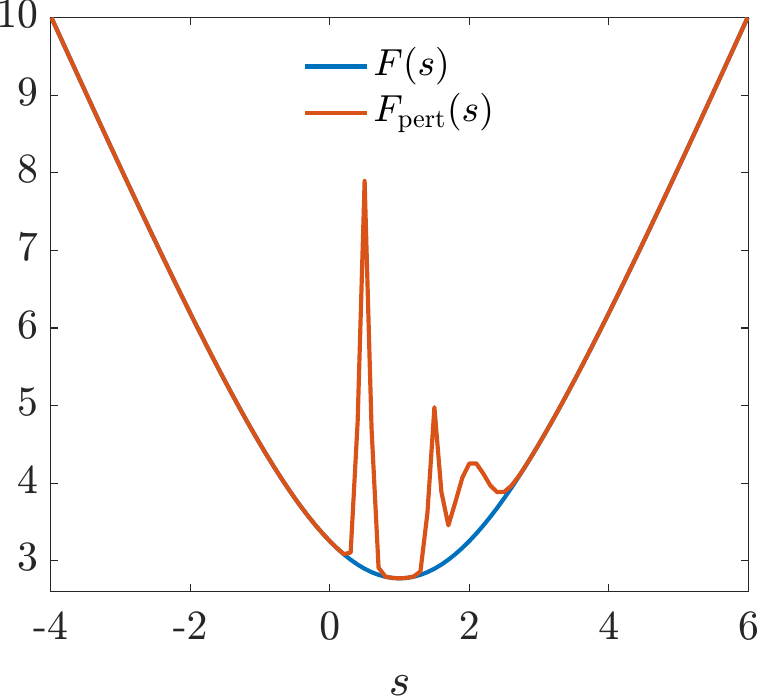}}
\subfigure[loss-$\{\min$ of loss$\}$]{\includegraphics[width=0.3\textwidth]{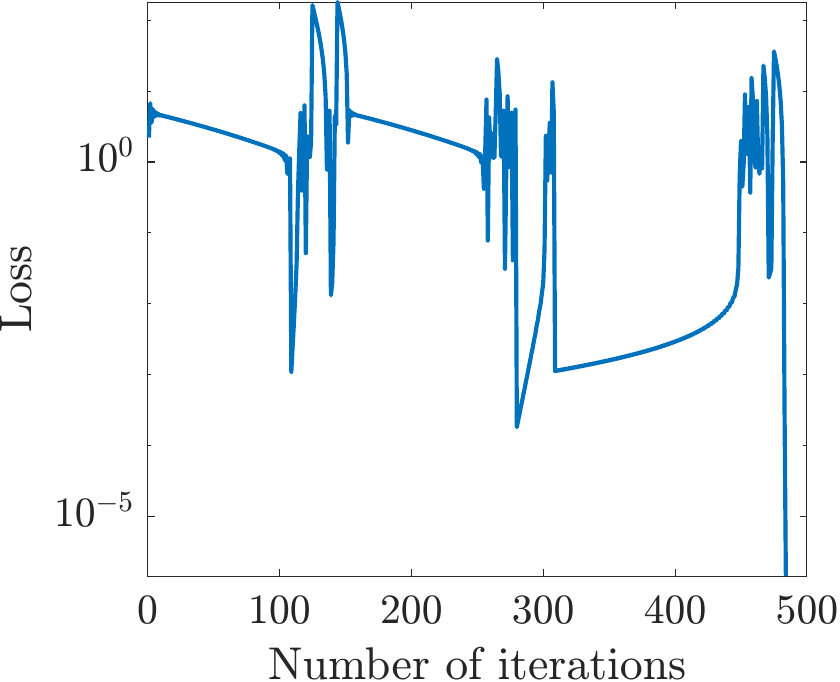}}
    \subfigure[no EoS]{\includegraphics[width=0.3\textwidth]{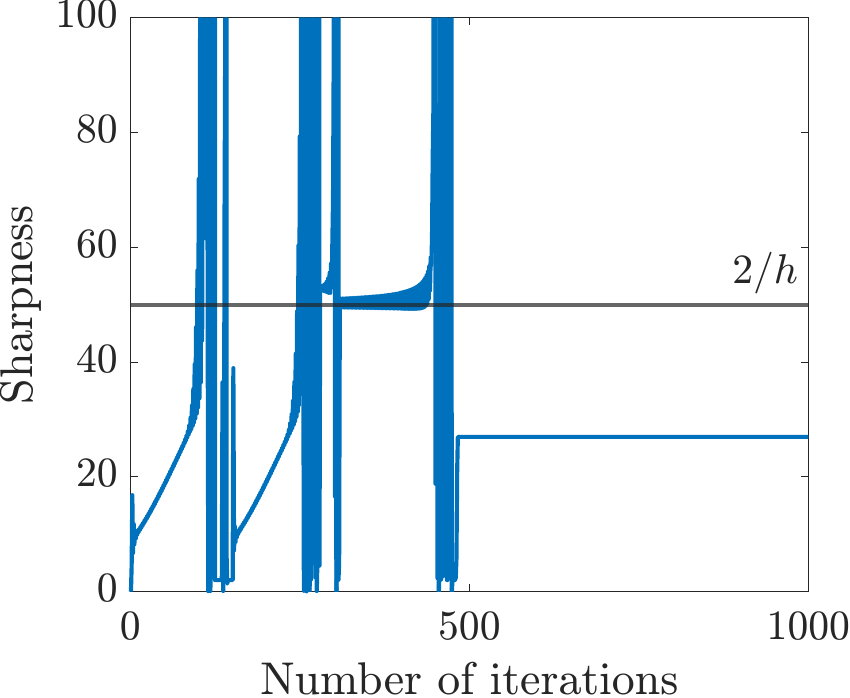}}
    \caption{Perturbed version of  $f(x,y)=2{(\log (\exp (x y-1)+1)+\log (\exp (1-x y)+1))}$, the good regularity function. Let $s=xy$. In (a), the blue curve is $F(s)=2{(\log (\exp (s-1)+1)+\log (\exp (1-s)+1))}$; the red curve is $F_{\rm pert}(s)=e^{-10 (s-2)^{2}}+2 e^{-100 (s-1.5)^{2}}+5 e^{-100 (s-0.5)^{2}}+F(s)$. In (b)(c), we apply GD on $f_{\rm pert}(x,y)=F_{\rm pert}(xy)$. The initial condition is $x_0=10,y_0=0.15$; the learning rate is $h=\frac{4}{x_0^2+y_0^2}$.}
    \label{fig:regularity}
\end{figure}

\end{remark}

In order to analyze the effect of regularity, we would like to design a family of $\mathcal{C}^2$ functions with various degrees of regularity. \cite{wang2022large} studied a matrix factorization problem with the objective $\|A-XY^\top \|_{\rm F}^2$. When $A$ is a scalar, it becomes $(A-x^\top y)^2$ where $x,y$ are vectors of the same dimension;  if we let $s=x^\top y$, then $f(x,y)=F(s)=(A-s)^2$ and $\mathrm{dor}(F)=2$. 
\cite{zhu2023understanding} 
analyzed a minimalist\footnote{But it doesn't mean EoS cannot appear for simpler functions.} example for EoS, which is $f(x,y)=F(xy)=(1-(xy)^2)^2$ with $\mathrm{dor}(F)=4$ (see Section~\ref{sec:discussions}). \cite{ahn2022learning} examine a class of Lipschitz functions on $xy$, i.e., $\mathrm{dor}=1$, including $F(xy)=\log(1+\exp(-xy))+\log(1+\exp(xy))$ and $\sqrt{1+(xy)^2}$. Inspired by these designs, we define a more general class of functions, with $\mathrm{dor}$ ranging from 0 to $+\infty$ as follows  
\begin{align}
\label{eqn:functions}
    f(x,y)=F(xy)=
    \begin{cases}
    C_a {(\log (e^{x y-1}+1)+\log (e^{1-x y}+1))^a}& \\
    \qquad\text{with constant } C_a=\frac{1}{{a 2^{a-2} \log ^{a-1}(2)}}
    &\text{ , for }0<a\le 1, a\in\mathbb{R}\\
    C_b(1-(xy)^{b})^2& \\
    \qquad\text{with constant } C_b=\frac{1}{2b^2} 
    &\text{ , for }b=2n-1, n\in\mathbb{Z}.
    \end{cases}
\end{align}
Note the 1st and 2nd lines of \eqref{eqn:functions} respectively correspond to $\mathrm{dor}\leq 1$ and $\mathrm{dor}>1$. At first glance, the splitting of the definition seems unnatural, but in fact note that the 1st line can be seen as a smoothened version of the 2nd line: when $b<1$, the gradient of $(1-(s)^b)^2$ at the minimum point $s=0$ is singular; by smoothing this function using the idea inspired by \cite{ahn2022learning} into that in the 1st line, its gradient becomes $\mathcal{C}^1$ for all the points. 

In addition, the design is such that the locations of all minimizers are fixed, specifically at $xy=1$, for all members of the function class.
This way, it is fairer and more consistent to compare GD's convergence for this class.
Note this choice of global minima introduces a more intricate landscape compared to the design in \citep{ahn2022learning} where the minima are $\{(x,y)\,|\,xy=0\}$. The reason for complicating the functions in order to move their minimizers lies in the variation of topological structure, where $xy=1$ exhibits two distinct branches while it degenerates into a single connected set in the case of $xy=0$. Moreover, our results for global minima $xy=1$ can be easily generalized to any non-zero value, not necessarily 1. An additional remark is that the global minima $xy=1$ are also different from $xy=\pm1$ used in~\cite{zhu2023understanding}, whose objective function shows similar landscape but additional symmetry, and that's why in this work we restrict $b$ to odd values, ruling out the even. However, we conjecture that our analysis has the potential to be extended to the case with minima $xy=\pm1$.

Finally, note the constant coefficients are chosen such that the largest eigenvalue of Hessian (i.e., sharpness) at any minimizer is always $x^2+y^2$, regardless of the value of $a$ or $b$, which contributes to a fairer comparison as well.

We summarize the properties of this family of functions in the following.

\begin{proposition}
\label{prop:functions}
This family of functions~\eqref{eqn:functions} have the following properties:
    \begin{itemize}
    \item $f\in\mathcal{C}^2(\mathbb{R}^2)$;
        \item All the minima of these functions are global minima, located at $xy=1$; 
        \item The sharpness at any minimizer $(x,y)$ is $x^2+y^2$, for all these functions;
        \item The degree of regularity satisfies $\mathrm{dor}(F)=\begin{cases}
                a,& \text{for } 0<a\le 1,\\
                2b,&\text{for }b=2n-1. 
            \end{cases}$
    \end{itemize}
\end{proposition}

\subsection{(Degree of) regularity of neural network training objective}
\label{subsec:regularity_neural_network}

\noindent$\bigstar$ \textbf{Toy model.}
Although still far from being realistic, the above functions~\eqref{eqn:functions} are designed to emulate the setup of a neural network, particularly combining both the loss function and the activation function. To supplement our theory, Section~\ref{sec:experiments} will provide empirical experiments on practical neural networks. But for now, let us first discuss some analytical connections. For example, consider the neural network model defined in the following problem
\begin{align}
    \label{eqn:neural_network_model}\min_{W_1,\cdots,W_\ell} f(W_1,\cdots,W_\ell)=\sum_i\mathcal{L}(g(W_1,\cdots,W_\ell;u_i^{\rm input}),u_i^{\rm output}) , 
\end{align}
where $g=W_\ell\sigma_{\ell-1}(W_{\ell-1}\sigma_{\ell-2}(\cdots\sigma_{1}(W_1 u_i^{\rm input})))$; $\mathcal{L}(\cdot,\cdot)$ is the loss function,  $\sigma_i(\cdot)$ is the activation function, $(u_i^{\rm input},u_i^{\rm output})$'s are data pairs, and $W_i$'s are weight matrices. 
To better understand the regularity in neural network model~\eqref{eqn:neural_network_model}, we consider the following toy model of a 3-layer neural network trained on one data point $\{u^{\rm input}=1,u^{\rm output}=1\}$: the first layer with $n$ linear neurons, the second layer with just 1 (nonlinear)  neuron, and last layer fixed (weight assumed WLOG to be 1). Then the training objective function is 
\begin{align}
   \label{eqn:nn_objective_toy_3_layer}
   f(W_1,W_2)=\mathcal{L}(\mathcal{\sigma}(W_2 W_1),1),\text{ where }W_1,W_2^\top\in\mathbb{R}^{n}. 
\end{align}
This function could again be rewritten as $F(W_1W_2)=f(W_1,W_2)$ and we have
\begin{align}
\label{eqn:dor_3layer_toy}
	\mathrm{dor}(F)=\mathrm{dor}(\mathcal{L}(\cdot,1))\mathrm{dor}(\sigma).
\end{align}

This means the regularity of this objective function depends on two parts: one is the neural network model $g$ which is further based on the regularity of the activation function, and the other is the loss function $\mathcal{L}$.

\vskip2pt
\noindent$\bigstar$ \textbf{Basic components of neural network.} Now let us exemplify some loss functions and activation functions with various regularities.

\begin{example}
\label{ex:regularity_loss_activation}
    Degrees of regularity of some loss functions
    \begin{itemize}
        \item $\ell^2$ loss: $\mathcal{L}(s,\cdot)=\frac{1}{2}(\cdot-s)^2$, $\mathrm{dor}(\mathcal{L})=2$;
        \item Huber loss: $\mathcal{L}(s,\cdot)=\begin{cases}
            \frac{1}{2}(\cdot-s)^2,& |\cdot-s|\le \delta\\
            \delta (|\cdot-s|-\frac{1}{2}\delta),& otherwise
        \end{cases}$, $\mathrm{dor}(\mathcal{L})=1$;
    \end{itemize}
    Degrees of regularity of some activation functions
    \begin{itemize}
    \item Hyperbolic tangent $\tanh$: $\sigma(s)=\frac{e^s-e^{-s}}{e^s+e^{-s}}$, $\mathrm{dor}(\sigma)=0$.
        \item (Leaky) ReLU function:
         $\sigma(s)=\begin{cases}
            \alpha s,& s<0\\
            s,& s\ge 0
        \end{cases}$ for some $0\leq\alpha\leq 1$, $\mathrm{dor}(\sigma)=1$.
        \item $\text{ReLU}\,^k$ function\footnote{Applications and analysis of this activation can be found in e.g. \citet{yang2023optimal,luo2020two,gao2023gradient}.}:
         $\sigma(s) = \max(0, s^k)$, $\mathrm{dor}(\sigma)=k$.
    \end{itemize}
\end{example}

According to the above discussion, the degree of regularity of the neural network objective varies as either the loss or the activation changes. Moreover, large learning rate phenomena may not occur for certain choices of loss and activation (i.e., when $\mathrm{dor}$ is large) if no extra techniques are applied to help training. Empirical results on more practical network models are consistent with our theory as shown in Section~\ref{sec:experiments}, where loss and activation functions with different regularities are tested.

Next, we will show that a certain training method---batch normalization---can eliminate the effect of choices of loss and activation that otherwise lead to bad regularity.

\vskip2pt
\noindent$\bigstar$ \textbf{Batch normalization.}
Batch normalization \citep{ioffe2015batch} is an efficient approach to enhance the robustness of the networks and improve both optimization and generalization. Existing works propose that batch normalization can smoothen the landscape, i.e., decrease the sharpness, with a larger learning rate \citep{lyu2022understanding,santurkar2018does,bjorck2018understanding,ghorbani2019investigation,karakida2019normalization}. We reinforce this perspective and suggest that batch normalization helps reduce the degree of regularity of the model. Consider the general objective function~\eqref{eqn:neural_network_model}. If we add batch normalization to each layer and implement (full batch) GD, i.e.
\begin{align}
\label{eqn:batch_normalization_neural_networks}
g_{j+1}=\sigma_j(\mathrm{BN}(W_jg_j)), g_1=(u_1^{\rm input},\cdots,u_N^{\rm input}),\text{ where }N\text{ is the number of data,}
\end{align}
\noindent then we can consider $\tilde{\sigma}(\cdot)=\sigma(\mathrm{BN}(\cdot))$ to be batch normalized activation function for some activation $\sigma$. 
Given $N$ inputs $\{v_i\}_{i=1}^N$ with $v_i=(v_i^{(1)},\cdots,v_i^{(n)})^\top\in\mathbb{R}^n$, 
we have
\begin{align*}
    \tilde{\sigma}(v_i^{(k)})=\sigma\left(\frac{v_i^{(k)}-\mathrm{Mean}(v^{(k)})}{\sqrt{\mathrm{Var}(v^{(k)})+\epsilon}}\right),
\end{align*}
where $\mathrm{Mean}(v^{(k)})=\frac{1}{N}\sum_{i=1}^N v_i^{(k)},\ \mathrm{Var}(v^{(k)})=\frac{1}{N}\sum_{i=1}^N (v_i^{(k)}-\mathrm{Mean}(v^{(k)}))^2,$ and $\epsilon>0$ is a very small hyperparameter. Then with high probability, $\tilde{\sigma}(v)$ is bounded.
By Definition~\ref{def:degree_of_regularity}, we have $\mathrm{dor}(\tilde{\sigma})\approx 0,$
which is independent of $\mathrm{dor}(\sigma)$. If we consider the above 3-layer toy example, then
\[
	\mathrm{dor}(F)=\mathrm{dor}(\mathcal{L}(\cdot,1))\mathrm{dor}(\tilde{\sigma})\approx 0.
\]
This shows that batch normalization can effectively drive the training objective function from bad (large) regularity to good (small) regularity. We also experimentally validate our theory on how batch normalization affects regularity and ramifications in Section~\ref{sec:experiments}, where large learning rate phenomena reappear after batch normalization is added to a model with bad regularity.

Despite all the discussions above, we emphasize again that, the theories constructed in this article are still not enough for analyzing general neural networks, although they can already work for a larger class of objective functions than those considered in the literature.

Now we are ready to quantify large learning rate phenomena under different regularities. See Section \ref{sec:eos} and \ref{sec:balancing}, and note results there will be based on the convergence theorems in Section~\ref{sec:convergence}.

\section{Edge of stability}
\label{sec:eos}

Section~\ref{subsec:related_works} mentioned that EoS is a large learning rate phenomenon, and we now provide more detailed explanations. Its original description \citep{cohen2021gradient} contained two stages, namely progressive sharpening, and limiting sharpness stabilization around $2/h$. \citet{ahn2022learning} later observed a third stage before progressive sharpening (which will be referred to as pre-EoS) where the sharpness will first decrease before increasing, and more empirical investigations of sharpness reduction in early training were then conducted by \citet{kalra2023phase}. A description of the full process is the following:
\begin{itemize}
    \item 
\textbf{Pre-EoS (de-sharpening).} This stage characterizes the situation where at the very beginning of the iterations, the sharpness decreases sharply before the occurrence of the well-known EoS (see Figure~\ref{fig:intuition_phenomenoa_appears}(c) and \citet{ahn2022learning}). It does not necessarily come with all the EoS phenomenon (see Figure~\ref{fig:intuition_phenomenoa_appears}(c) and (d)) and only appears when the initial sharpness is very large. Nevertheless, it helps demonstrate the behavior of GD under large learning rates.

\item 
\textbf{Progressive sharpening.} This stage is governed by increasing sharpness. Due to pre-EoS, GD is guaranteed to start in a relatively flat region (small sharpness) when this stage begins, even if it was initialized in a sharp region. 
The minimizer that GD will eventually converge to has larger sharpness than the majority of sharpness values along GD trajectory in this stage, and as GD progresses, the sharpness value `crawls' up. Such behaviors stem from the good regularity of the objective function and do not necessarily appear in functions with bad regularity.

\item
\textbf{Limiting sharpness near $2/h$.} Stability theory of the GD dynamics (see Appendix~\ref{app:stability_analysis}) guarantees that the limiting sharpness has to be not exceeding $2/h$, but not necessarily close to $2/h$. It is the good regularity of objective function that will drive the final sharpness towards $2/h$.
\end{itemize}

In this section, we establish a rigorous theoretical framework that elucidates the above three stages, and quantify our claim that EoS is more likely to occur when the objective function exhibits good regularity.
In the context of the progressive sharpening in EoS, we will primarily focus on the concept of `sharpening' rather than `progressive' due to the subjective nature of this word: this stage can actually exhibit diverse sharpening speeds (see Section \ref{sec:experiments}).

In order to better understand the insight of EoS, we commence by exploring the two-phase convergence of GD under large learning rates, as initially proved in \cite{wang2022large} for matrix factorization problem: 
\begin{itemize}
    \item {\it Phase 1}, GD tries to escape from the attraction of sharp minima and enters a flat region,  where the local Lipschitz constant of the gradient is upper bounded by $\approx 2/h$.
    
    \item {\it Phase 2}, GD converges inside the flat region.
\end{itemize}

  This two-phase 
  convergence pattern is also provably true for our family of functions under large learning rates. 
  Detailed convergence analyses are deferred to  Section~\ref{sec:convergence}. 

\vskip2pt
\noindent $\bigstar$ \textbf{Intuition for de-sharpening and progressive sharpening.} 
To illustrate the intuition, let $s=xy$ and consider $F(s)$ in~\eqref{eqn:functions} (note $x,y$ follow GD trajectories but $s$ does not). Consider the sharpness of $F(s)$ (i.e., $F''(s)$) in two example cases, one with good regularity and the other with bad regularity.  As shown in Figure~\ref{fig:sharpness_goodregu_badregu}, the sharpness of good regularity function concentrates near the minimum ($s=xy=1$) and vanishes when far away from it (Figure~\ref{subfig:good_regu}), while the situation with bad regularity function is the opposite (Figure~\ref{subfig:bad_regu}). Therefore, in \emph{Phase 1} as GD escapes to a flatter region, GD travels away from the minimum for the good regularity function, while for the bad regularity function, GD moves closer to it. This corresponds to the \textbf{de-sharpening} stage that can occur in both cases. In \emph{Phase 2}, GD starts to converge but results in different behaviors. In the good regularity case (Figure~\ref{subfig:good_regu}), GD faces a longer path, requiring its sharpness to `climb up the hill' to converge. This process allows sufficient time for sharpness to increase (\textbf{progressive sharpening}) and eventually reach its peak, namely \textbf{limiting sharpness near $2/h$}. In contrast, in the bad regularity case, GD is already closer to the minimum at the end of \emph{Phase 1} when compared to its initial state. 
Consequently, the sharpness will not undergo significant changes before GD converges. As a result, the typical phenomena observed in good regularity scenarios do not occur.

Note, however, that the actual landscape in which GD navigates is $f(x,y)=F(xy)$ instead of $F(s)$, and the behavior of GD is thus much more complicated. For example, 
since $s=xy$,
each point in Figure~\ref{fig:sharpness_goodregu_badregu} actually represents a continuous set of points. The function $f$ has a rescaling symmetry of $(x,y)$,
known as homogenity \citep{wang2022large}, meaning that $(x,y)\mapsto (cx,y/c)$ will not change $f$'s value for any $c\neq 0$. This transformation can drastically change the eigenvalues of Hessian $\nabla^2f$, especially when $|c|$ is large. Therefore, the fact that $f(x,y)=F(xy)$ creates a combination of flat and sharp regions. 
In particular, the global minima of $F(xy)$ form an unbounded set, as opposed to one single point in $F(s)$. Even if we just consider these global minima alone, their sharpnesses, which are given by $x^2+y^2$, already assume any values in $[2,\infty)$, creating sharp and flat regions per se. The complexity of the actual landscape is the groundwork of diverse phenomena like EoS and makes the following results for understanding them rather nontrivial.

\begin{figure}[ht]
    \centering    \subfigure[good regularity]{{\includegraphics[width=0.38\textwidth]{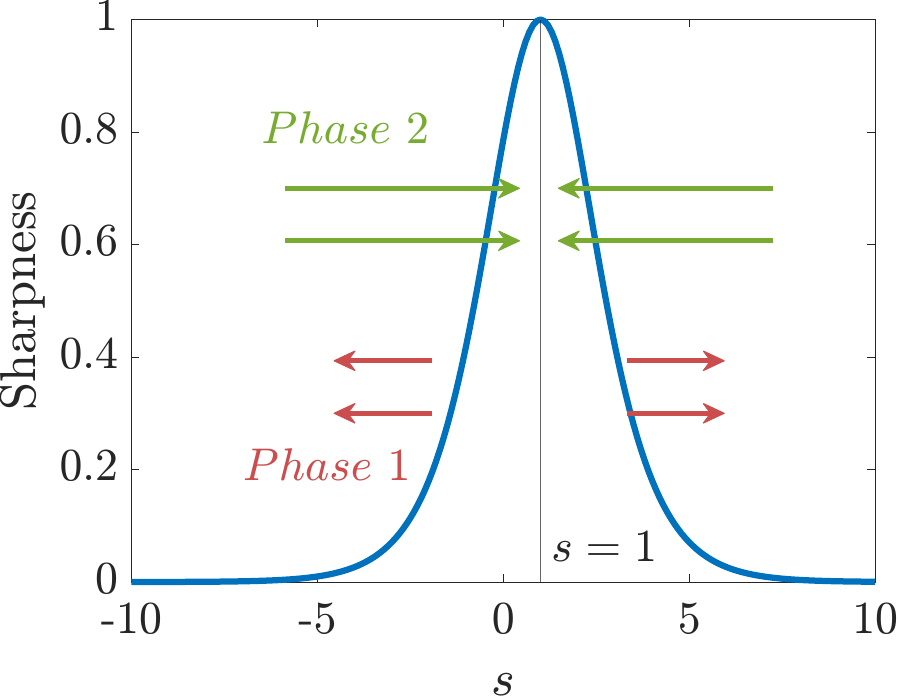} }\label{subfig:good_regu}}%
    \qquad   \subfigure[bad regularity]{{\includegraphics[width=0.38\textwidth]{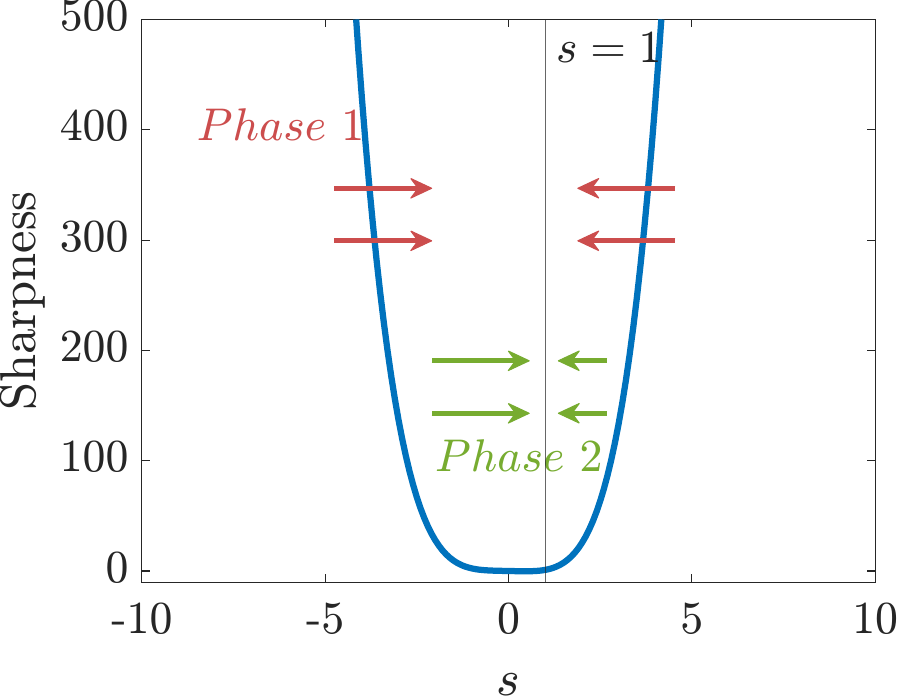} }\label{subfig:bad_regu}}%
    \caption{
    The sharpness of different $F(s)$ with good (small) and bad (large) regularities. The blue curves are the sharpness, i.e., $F''(s)$; the red arrows denote the direction of GD in {\it Phase 1}; the green arrows denote the direction of GD in {\it Phase 2}.}
    \label{fig:sharpness_goodregu_badregu}%
\end{figure}

For functions with good regularity (i.e., small degree), the occurrence of EoS is linked to a large range of initial conditions as is shown in the following theorem.

\begin{theorem}[EoS]
\label{thm:eos_good_regularity}
   Consider $0<a\le 1$ in~\eqref{eqn:functions}. Assume the initial condition satisfies $(x_0,y_0)\in \{(x,y):1<xy<M_1(a), x^2+y^2\gtrsim 4c_1^{-4/3}\}\backslash \mathcal{B}_a\text{, for some }c_1=c_1(a)>0\text{ and }M_1(a)>3,$
   where $\mathcal{B}_a$ is a Lebesgue measure-0 set. Let the learning rate be $h=\frac{C}{x_0^2+y_0^2}\text{ for }2<M_2(a)\le C\le 4,\text{ where }M_2(a)\lesssim2.6.$ See definitions of $M_1,M_2,c_1$ in Theorem~\ref{thm:good_regularity_convergence}.
   The sharpness at $k$th iteration $S_k$ 
   satisfies the following properties: 
    \begin{itemize}
        \item Entering flat region (end of pre-EoS, and preparation for progressive sharpening): There exists $N\in\mathbb{N}$, s.t.,
        \begin{align*}
        S_N\lesssim \frac{1}{4}(6-C)S_\infty,
        \end{align*} 
        and note $\frac{1}{4}(6-C)S_\infty < S_\infty$ for all $C$ defined above.
        For example, if $C=4$, then $
         S_N\lesssim \frac{1}{2}S_\infty.$
    \item Limiting sharpness: GD converges to a global minimum  $(x_\infty,y_\infty)$, and its sharpness $S_\infty$ satisfies 
    \begin{align*}
        \frac{2}{h}-\tilde{\mathcal{O}}(h)\le S_\infty=x_\infty^2+y_\infty^2\le\frac{2}{h}.
\end{align*}
        
    \end{itemize}

\end{theorem}

    \begin{remark}
    \label{rmk:aofdlsihfgopuihqoiu21}
    The above theorem embodies a detailed description of the whole EoS process. More precisely, 
there is \textbf{progressive sharpening}: since there is at least one point along the trajectory with small sharpness (smaller than limiting sharpness), the sharpness will eventually increase to the limiting sharpness. Note despite the growth of sharpness, GD still converges inside the flat region (i.e., in \emph{Phase 2}) since the limiting sharpness is upper bounded by $2/h$. In the end, the \textbf{sharpness stabilizes near $2/h$} within a distance of $\tilde{\mathcal{O}}(h)$.
The \textbf{pre-EoS (de-sharpening)} can also occur if the initial condition $(x_0,y_0)$ is close to the minima. In this case, the initial sharpness $S_0$ is approximately $x_0^2+y_0^2$ and we have
\begin{align*}
    S_N\lesssim \frac{1}{4}(6-C)S_\infty\approx \frac{1}{4}(6-C)\frac{2}{h}= \left(\frac{3}{C}-\frac{1}{2}\right)(x_0^2+y_0^2)\approx \left(\frac{3}{C}-\frac{1}{2}\right) S_0 < S_0,\text{ for }2<M_2\le C\le 4
\end{align*}
which means the sharpness will first decrease before GD enters the progressive sharpening stage. This also implies that the limiting sharpness is smaller than the initial sharpness, which corresponds to the balancing phenomenon (see Section~\ref{sec:balancing}).
    \end{remark}
    
\begin{figure}[ht]
    \centering
    \includegraphics[width=\textwidth]{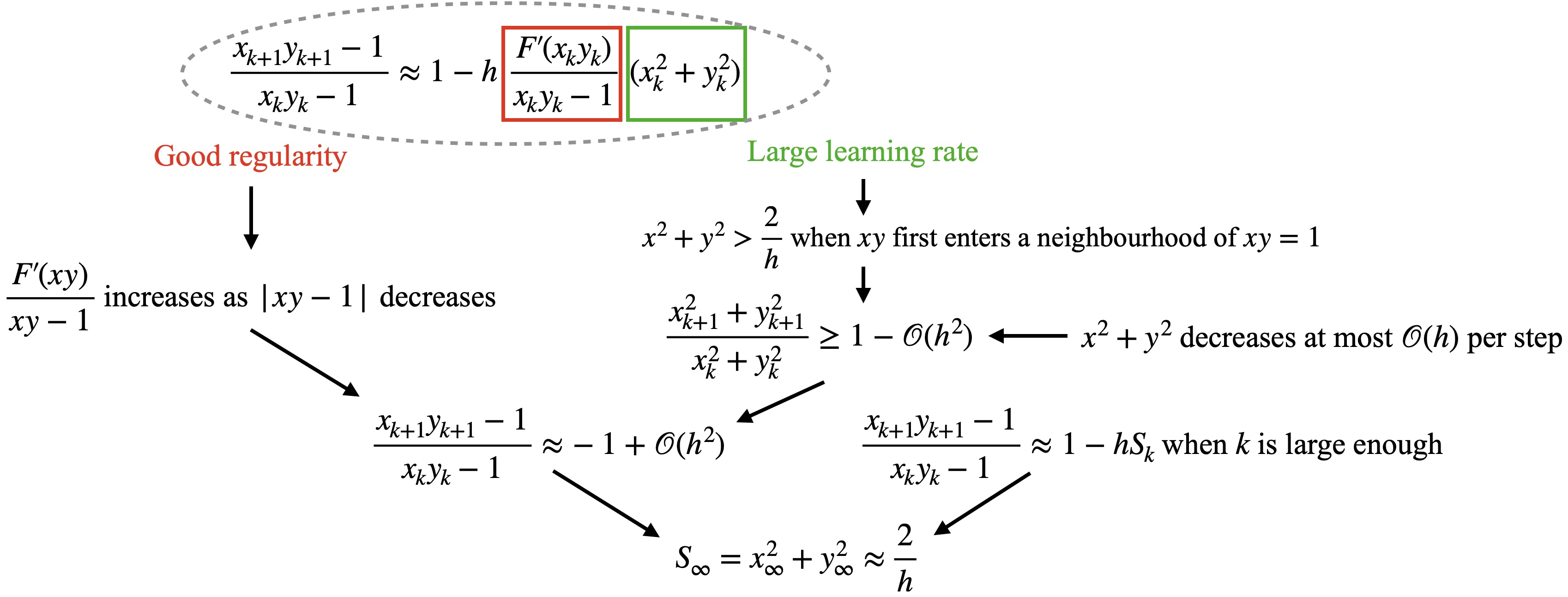}
    \caption{Proof idea of $S_\infty\approx\frac{2}{h}$ for good regularity functions. The formula inside the dashed ellipse is what we mainly consider. The term inside the red box is primarily affected by good regularity (also in red), while the one inside the green box is influenced by large learning rate (in green).}
    \label{fig:proof_sketch_good_regu}
\end{figure}
Figure~\ref{fig:proof_sketch_good_regu} shows a diagram of the proof idea for the limiting sharpness $\approx 2/h$. The complete proof is in Appendix~\ref{app:proof_good_regularity}. Roughly speaking, good regularity and large learning rate control the two terms $\frac{F'(xy)}{xy-1}$ and $x^2+y^2$ respectively, which result in a slow and oscillating convergence pattern of $xy-1$. This convergence rate can be approximated by a function of sharpness $(1-hS_k)$ at each step and thus eventually leads to the $\approx 2/h$ limiting sharpness.

Next, we will show that when the regularity is bad (large degree), 
there exists a large set of initial conditions that do not lead to EoS even though we use a large learning rate. Moreover, such a set expands as regularity gets worse (i.e., the degree of regularity becomes larger). 
\begin{theorem}[no EoS; one-sided stability]
\label{thm:non_eos_bad_regularity}
Consider ${b}=2n+1\text{ for }n\in\mathbb{Z}$ in~\eqref{eqn:functions}. 
Assume the initial condition satisfies $(x_0,y_0)\in \{(x,y):xy>2^{\frac{1}{b-1}}, x^2+y^2\ge 4 \}\backslash \mathcal{B}_b$, where $\mathcal{B}_b$ is a Lebesgue measure-0 set. Let the learning rate be $\frac{2}{(x_0^2+y_0^2+4)(x_0 y_0)^{2b-2}} \le h\le \frac{M_3(b,x_0,y_0)}{(x_0^2+y_0^2+4)(x_0 y_0)^{2b-2}}$, where $3<M_3(b,x_0,y_0)\le 4$, and the precise definition of $M_3$ is given in  Theorem~\ref{thm:bad_regularity_convergence}.
Then GD converges to a global minimum  $(x_\infty,y_\infty)$, and its sharpness $S_\infty$ satisfies
\begin{align*}
    S_\infty=x_\infty^2+y_\infty^2\le \frac{1}{h}.
\end{align*}
\end{theorem}
This theorem shows that for functions with bad regularities, the limiting sharpness is bounded by $1/h$, which is well below the stability limit of $2/h$. Therefore, not only is there no EoS, but also a different phenomenon occurs, which we refer to as \textbf{one-sided stability}. Note the lower bound of  $x_0y_0$, i.e., $2^{\frac{1}{b-1}}$, decreases as $b$ increases. This shows the inflation of the initial condition set that leads to the non-EoS phenomenon as the degree of regularity increases. See more explanation in `Example implications of the general theory' at the end of this section.

Here is the intuition of why bad regularity gives 
$1/h$ which is very different from the $2/h$ stability limit. We denote $\tilde{L}$ to be the Lipschitz constant of the gradient in the bounded region of GD trajectory in {\it Phase 2}, the final converging phase, and $L$ to be the Lipschitz constant of the gradient in the bounded region of whole GD trajectory. The key idea is that the sharpness of bad regularity functions varies far more than that of good regularity functions; this leads to 
$h\approx 2/\tilde{L}$ for good regularity but $h<1/\tilde{L}$ for bad regularity. 
Before delving into further details, let us first clarify the difference between $h<1/\tilde{L}$ and $1/\tilde{L}<h<2/\tilde{L}$. As is shown in Figure~\ref{fig:different_GD_1L_2L}, when $h<1/\tilde{L}$, GD evolves along one side of the minimum point; when $1/\tilde{L}<h<2/\tilde{L}$, GD jumps between the two sides.
\begin{figure}[ht]
    \centering
    \includegraphics[width=0.6\textwidth]{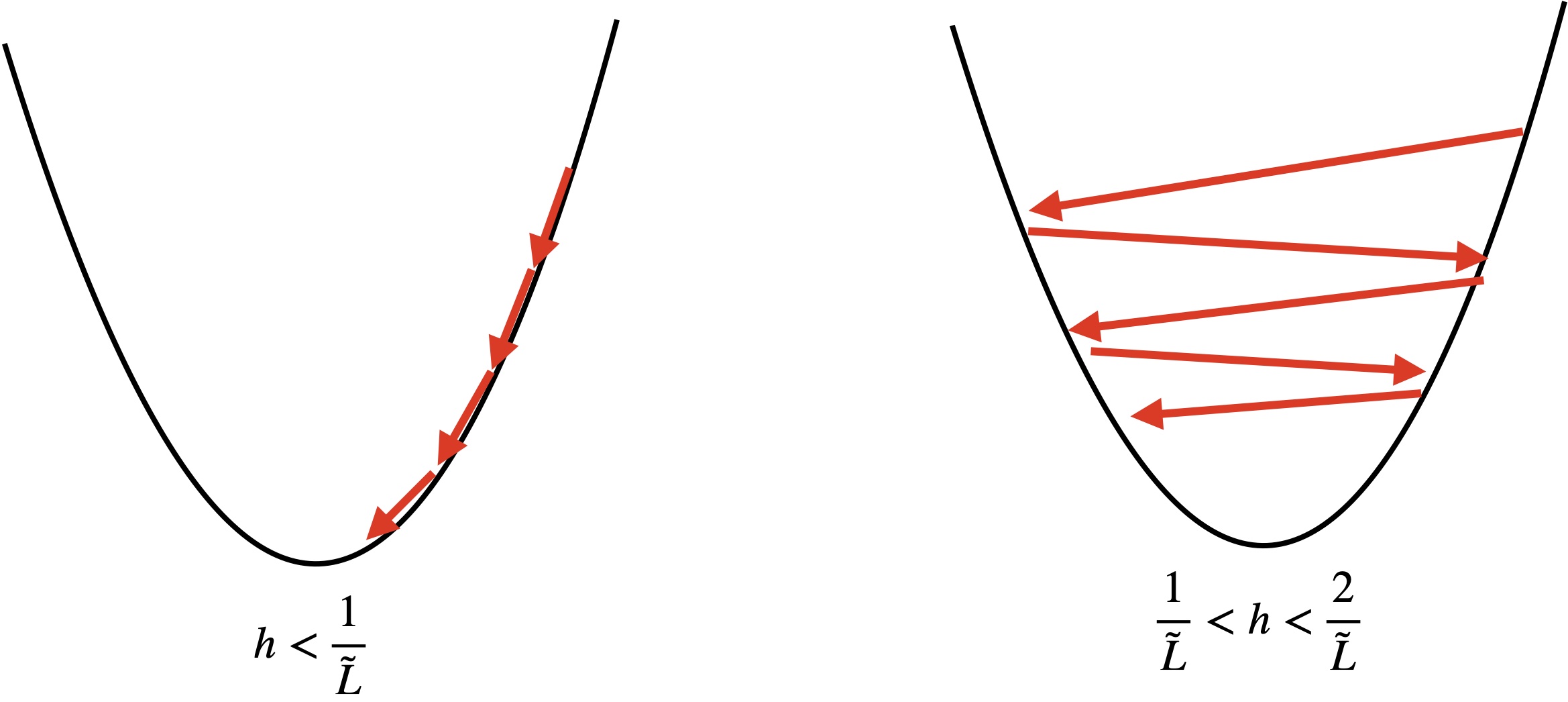}
    \caption{GD trajectory under different learning rate $h$. The black curve represents the objective function; the red arrows represent the evolution of GD iterations.}
    \label{fig:different_GD_1L_2L}
\end{figure}
Suppose we start with a large learning rate $h$ in a sharp region. The landscape along the GD trajectory will quickly become flat enough, i.e., the sharpness in {\it Phase 2} 
is small enough compared to those of the early iterations so that roughly $\tilde{L}<L/C$ for $C$ defined in Theorem~\ref{thm:non_eos_bad_regularity}. Then, the large learning rate $h$, which is $>2/L$, will eventually be $<1/\tilde{L}$, and result in GD decreasing from only one side of the continuous set of minima thereafter, which is where the name `one-sided stability' comes from. This $h<1/\tilde{L}$ gives rise to the limiting sharpness bound $1/h$.

\begin{figure}[ht]
    \centering
    \includegraphics[width=0.8\textwidth]{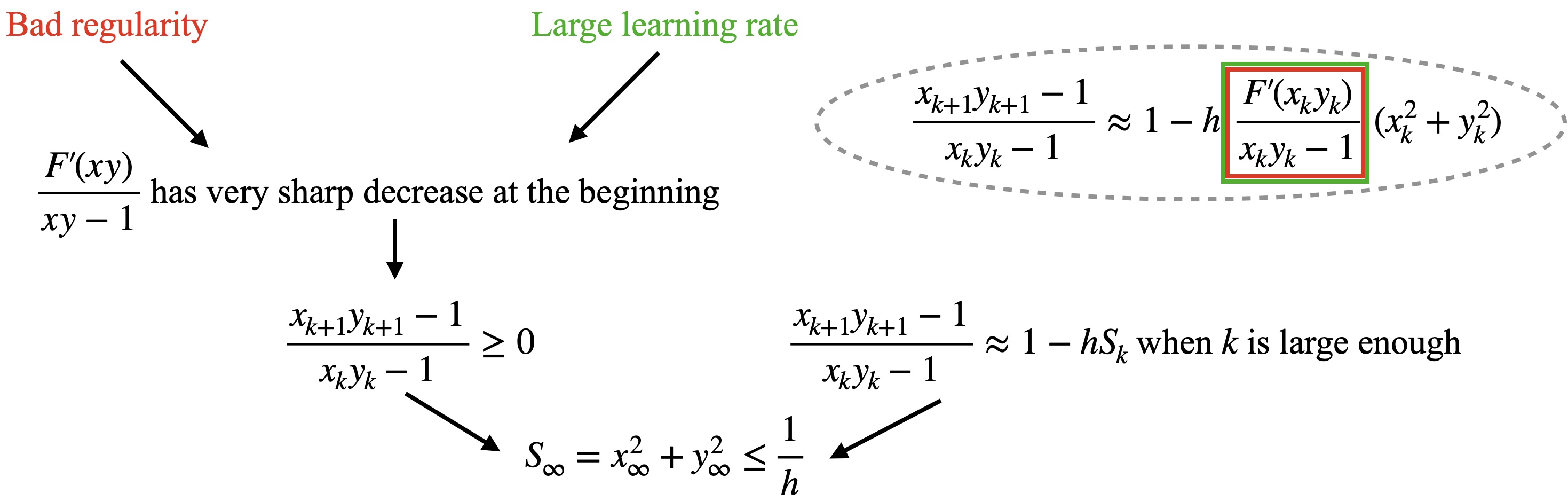}
    \caption{Proof idea of $S_\infty\le \frac{1}{h}$ for bad regularity functions. The formula inside the dashed ellipse is identical to that in Figure~\ref{fig:proof_sketch_good_regu}. Different from good regularity cases (Fig.\ref{fig:proof_sketch_good_regu}), here bad regularity (red color) and large learning rate (green color) mainly affect the same term in the dash-circled expression.
    }
    \label{fig:proof_sketch_bad_regu}
\end{figure}
More precisely, Figure~\ref{fig:proof_sketch_bad_regu} sketches the proof idea of limiting sharpness $\le \frac{1}{h}$. A complete proof is in Appendix~\ref{app:proof_bad_regularity}. For bad regularity functions, the sharpness varies a lot even if the function is just perturbed a little bit (see Figure~\ref{fig:sharpness_goodregu_badregu}). Together with large learning rate, the term $\frac{F'(xy)}{xy-1}$, which can be seen as an approximation of $F''(xy)$ near $xy=1$, exhibits a sharp decrease at the beginning of the GD iterations; after that, $\frac{F'(xy)}{xy-1}$ will be small enough throughout the rest of the trajectory such that $xy-1$ has monotone decrease, which leads to the $\frac{1}{h}$ bound of limiting sharpness.

\begin{figure}[ht]
    \centering
    \includegraphics[width=0.9\textwidth]{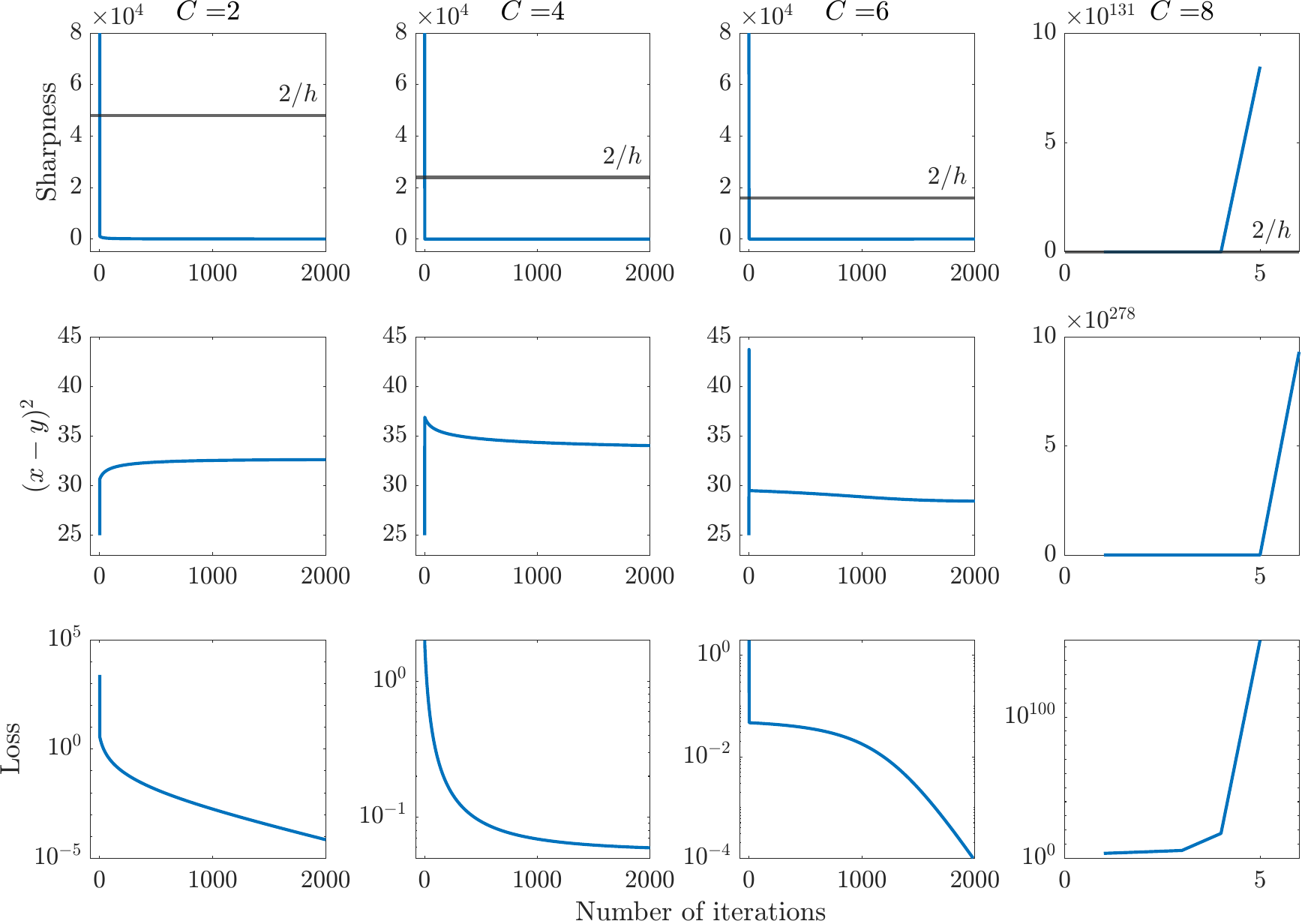}
    \caption{No EoS or balancing for bad regularity function~\eqref{eqn:functions} with $b=3$ under different learning rates. All the figures share the same initial condition $x_0=6,y_0=1$; the learning rate is chosen to be $h=\frac{C}{(x_0^2+y_0^2)(x_0y_0)^{2b-2}}$ with $b=3$, where $C=2,4,6,8$, until divergence.}
    \label{fig:bad_regularity_no_eos_not_lr_dependent}
\end{figure}

We now demonstrate that our choice of the large learning rate is not `cherry-picking', meaning the disappearance of the EoS phenomenon is not a consequence of specific choices of the learning rate, whether it is larger or smaller. Figure~\ref{fig:bad_regularity_no_eos_not_lr_dependent} exhibits the changes of sharpness (and also $(x_k-y_k)^2$) along GD trajectories under different learning rates chosen equidistantly until divergence. None of the plots shows EoS (and balancing) behavior since the sharpness stays far below $2/h$ at the limit.

Note that Theorem~\ref{thm:non_eos_bad_regularity} does not contain the $b=1$ case due to technical reasons. The convergence pattern of $b=1$ is different from the one described in the above theorem, which can be shown in, for example, the initial condition set of Theorem~\ref{thm:non_eos_bad_regularity}: if $b=1$, the lower bound $2^{\frac{1}{b-1}}$ of $xy$ is $+\infty$. Nevertheless, our main focus is to show the effect of regularity on large learning rate phenomena. The results under $b\ge 3$ already serve this purpose on the bad regularity side. For completeness, we illustrate the $b=1$ case by experiments (see Figure~\ref{fig:b=1_no_eos_experiments}) and will leave the theoretical result for future exploration.

\begin{figure}[ht]
    \centering
    \includegraphics[width=0.7\textwidth]{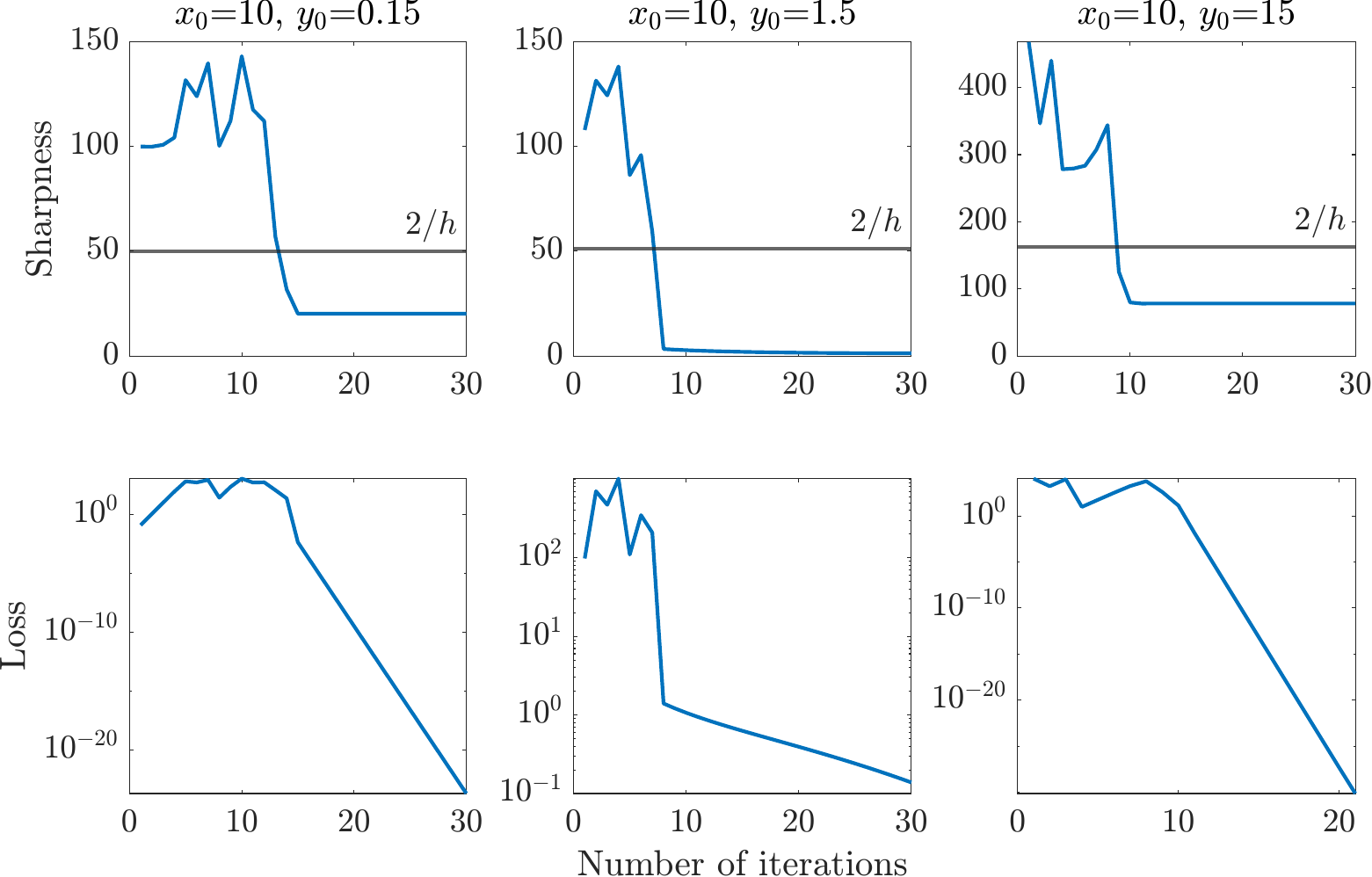}
    \caption{No EoS when $b=1$. For the initial condition, we fix $x_0=10$, and choose $y_0$ to be 0.15, 1.5, and 15; the learning rates are all taken to be $h=4/(x_0^2+y_0^2)$. It turns out that none of the three cases have the EoS phenomenon. The convergence pattern of both the sharpness and the loss is very irregular, especially in the early stage of the convergence.}
    \label{fig:b=1_no_eos_experiments}
\end{figure}

\vskip2pt
\noindent$\bigstar$ \textbf{Example implications of the general theory.}
Consider the family of functions in \eqref{eqn:functions}, i.e., with both good and bad regularity.
\begin{itemize}
    \item \emph{Bad regularity eliminates the EoS phenomenon.} If we fix the initial condition set for all functions, i.e., take the intersection of all initial condition sets \begin{align}
    \label{eqn:initial_condition_regularity_eliminates_phenomena}
    (x_0,y_0)\in \{\sqrt{2}<x_0y_0\le 3, x_0^2+y_0^2\gtrsim \max_a 4c_1(a)^{-4/3}\}\backslash \mathcal{B},\text{ where }\mathcal{B}=\bigcup_a \mathcal{B}_a\cup\bigcup_b\mathcal{B}_b,    
    \end{align}
 \begin{center}
     we have: {$\text{when } \mathrm{dor}\le 1,  \ S_\infty\approx\frac{2}{h};
       \text{when } \mathrm{dor}>2,  \ S_\infty\le \frac{1}{h}.$}
 \end{center} 
This means as the regularity increases, the same initial condition will lose the ability to lead to EoS, even if a large learning rate is used. Moreover, worse regularity (larger degree) results in a larger set of initial conditions that do not lead to EoS. When $\mathrm{dor}>2$, the initial condition set that does not show EoS expands as $\mathrm{dor}$ increases, i.e., 
\begin{align}
\label{eqn:initial_condition_expand_bad_regularity}
(x_0,y_0)\in \{(x,y):xy>4^{\frac{1}{\mathrm{dor}(F)-2}}, x^2+y^2\ge 4 \}\backslash \mathcal{B}\text{, where }\mathcal{B}=\bigcup_b\mathcal{B}_b\text{.}
\end{align}
\item \emph{Larger learning rate leads to smaller limiting sharpness, i.e., flatter minima.} Previous results, as are based on the stability theory, state that \emph{if} GD converges, then the limiting sharpness is upper bounded by $\mathcal{O}(1/h)$; this is a more precise version of the common belief that larger $h$ leads to a flatter minimum. The results above (Theorem~\ref{thm:eos_good_regularity} and~\ref{thm:non_eos_bad_regularity}), however, no longer require the assumption of convergence (i.e., the `\emph{if}' clause), have it proved and yield the same $\mathcal{O}(1/h)$ bound, for both good and bad regularity cases. 
\end{itemize}
 Although we do not characterize all the initial conditions in the space, our results are still quantitative and suggestive of the effect of increasing regularity as reducing large learning rate phenomena.

\subsection{Critical effect of large learning rate on EoS in both the good and bad regularity cases}
In this section, we will first demonstrate that the learning rates used in Theorem~\ref{thm:eos_good_regularity} and~\ref{thm:non_eos_bad_regularity} are indeed in the large learning rate regime, i.e., $h\gtrsim\frac{2}{L}$, where $L$ is the local Lipschitz constant of the gradient. Then, we will theoretically elucidate how large learning rates change the local Lipschitz constant as GD evolves, in both the good and bad regularity cases. 

For most of the initial conditions, our choices of learning rate $h$ depend on $L_0$, the local Lipschitz constant near the initial condition (see Section~\ref{subsec:notation}). To see this, for good regularity functions (Theorem~\ref{thm:eos_good_regularity}), if GD is initialized near the minima, then according to Proposition~\ref{prop:functions}, the sharpness $S_0\approx x_0^2+y_0^2$, which is also $\approx L_0$. Therefore, $h=\frac{C}{x_0^2+y_0^2}\approx \frac{C}{L_0}$.  Such $L_0$ controls the sharpness along the whole trajectory (see Figure~\ref{fig:intuition_phenomenoa_appears}(c) for example), namely, $L_0\approx L$. Consequently, $h\approx \frac{C}{L}$ and additionally $\frac{2}{L}\lesssim h\lesssim \frac{4}{L}$. For bad regularity functions (Theorem~\ref{thm:non_eos_bad_regularity}), by Lemma~\ref{lem:bad_regu_lr_2/L_4/L}, the lower bound of $h\gtrsim\frac{2}{L_0}$ and the upper bound of $h\gtrsim\frac{4}{L_0}$. Similarly, we have $h\gtrsim\frac{2}{L}$ (see for example Figure~\ref{fig:intuition_phenomenoa_disappears}(c)).

Apart from the above situations, there exist exceptions of initial conditions, where their corresponding $h$ does not depend on $L_0$. For good regularity functions (Theorem~\ref{thm:eos_good_regularity}), if the initialization is away from the minima, $L_0$ is very small and in this case, and $h$ depends on $L_\infty$, which is the local Lipschitz constant near the limiting point (see Section~\ref{subsec:notation}). Then $L$ is controlled by $L_\infty$ (see for example Figure~\ref{fig:intuition_phenomenoa_appears}(d)), i.e., $L\approx L_\infty$. By Theorem~\ref{thm:eos_good_regularity}, $L\approx L_\infty\approx S_\infty\approx \frac{2}{h}$ and thus $h\approx \frac{2}{L}$.

From a theoretical perspective, the mechanism of GD under large learning rates can be described in the following: when $h>\frac{2}{L}$, according to stability analysis (see Theorem~\ref{thm:stbility_analysis_2/L_bound} in Appendix~\ref{app:stability_analysis}),
GD cannot converge in the region where the minimum local Lipschitz constant $\approx L$.  
Instead, GD tends to search for a region with smaller local Lipschitz constant $\tilde{L}<L$ (which corresponds to \emph{Phase 1} discussed above), s.t. $\frac{2}{L}<h<\frac{2}{\tilde{L}}$, i.e., $h$ is inside the regular learning rate regime for the new flatter region (which is in \emph{Phase 2}). $h\approx \frac{2}{L}$ is on the boundary of the two regions, and therefore there is no `searching' for the flatter region.

\section{Balancing}
\label{sec:balancing}

The existence of the balancing phenomenon is originally proved in matrix factorization problem with objective function $\frac{1}{2}\|A-XY^\top\|_{\rm F}^2$ \citep{wang2022large}, for which GD optimization can still converge when learning rate $h$ exceeds $2/L$, and its limiting point $(X_\infty,Y_\infty)$ satisfies
\begin{align*}
    \left|\|X_\infty\|_{\rm F}-\|Y_\infty\|_{\rm F}\right|^2< \frac{2}{c_1 h}-c_2,\text{ for some }c_1,c_2>0.
\end{align*}
The larger the learning rate is, the smaller the gap between the magnitudes of $X$ and $Y$ will be, i.e., $X$ and $Y$ will become more balanced when compared to their initials.

One feature of balancing is that it can be used as an explicit characterization of sharpness. To see this, let us first recall that at any minimizer (satisfying $xy=1$), the sharpness of any function in our function class \eqref{eqn:functions} is $x^2+y^2$, by Proposition~\ref{prop:functions}. Then at that minimizer, we have
\begin{align*}
    (|x|-|y|)^2=(x-y)^2=x^2+y^2-2xy=x^2+y^2-2,
\end{align*}
i.e., $(x-y)^2$ is equivalent to the sharpness $x^2+y^2$ up to a constant shift. 
Therefore, if GD starts near a minimum, we can compare the limiting difference $(x_\infty-y_\infty)^2$ with its initial $ (x_0-y_0)^2$ to gain an understanding of whether GD converges to a sharper or a flatter minimum. 

In this section, we generalize the balancing phenomenon scrutinized in \citep{wang2022large} to the larger class of functions given by \eqref{eqn:functions} (when $\mathrm{dor}\le 2$), and prove that bad regularity can eliminate balancing.

 The following theorem shows the balancing phenomenon for functions with good regularity.
\begin{theorem}[Balancing]
\label{thm:balancing_good_regu}

  Consider $0<a\le 1$ in~\eqref{eqn:functions}. Assume the initial condition satisfies $(x_0,y_0)\in \{(x,y):1<xy<M_1(a), x^2+y^2\gtrsim 4c_1^{-4/3}\}\backslash \mathcal{B}_a\text{, for some }c_1=c_1(a)>0\text{ and }M_1(a)>3,$
   where $\mathcal{B}_a$ is a Lebesgue measure-0 set. Let the learning rate be $h=\frac{C}{x_0^2+y_0^2}\text{ for }2<M_2(a)\le C\le 4,\text{ where }M_2(a)\lesssim2.6,$ where the precise definitions of $M_1,M_2,c_1$ are in Theorem~\ref{thm:good_regularity_convergence}.
   Then GD converges to a global minimum  $(x_\infty,y_\infty)$, and we have
\begin{align*}
    (x_\infty-y_\infty)^2\le \frac{2}{C}(x_0-y_0)^2+2(x_0y_0-1).
\end{align*}
\end{theorem}
The above theorem shows that
given a large $h$ corresponding to $C>2$,
if $x_0y_0$ is close to 1, $(x_\infty-y_\infty)^2 < (x_0-y_0)^2$ and the balancing phenomenon occurs. Additionally, a larger $h$ leads to a smaller upper bound of $(x_\infty-y_\infty)^2$, and therefore is more balanced at the limit. Especially, if $C=4$, we have 
    \begin{align}
    \label{eqn:balancing_inequality_good_regu_a_C=4}
    (x_\infty-y_\infty)^2\le \frac{1}{2}(x_0-y_0)^2+2(x_0y_0-1)\lesssim \frac{1}{2}(x_0-y_0)^2.
\end{align}
The case with $b=1$ was already studied in the literature, and it also exhibits the balancing phenomenon as reviewed below:

\begin{theorem}[Original balancing; Corollary of Theorem 3.2 in~\cite{wang2022large}]
\label{thm:balancing_b=1}
Consider $b=1$ in~\eqref{eqn:functions}. Assume the initial condition satisfies $(x_0,y_0)\in\big\{(x,y):  x^2+y^2>8\big\}\backslash\mathcal{B},$ where $\mathcal{B}$ is some  Lebesgue measure-0 set. Let the learning rate be $h=\frac{C}{x_0^2+y_0^2+4},\text{ where }2\le C\le 4.$
Then GD converges to a global minimum  $(x_\infty,y_\infty)$, and we have
\[
(x_\infty-y_\infty)^2 \le \frac{2}{C}(x_0-y_0)^2+\frac{4}{C}(x_0y_0+2)-2.
\]
\end{theorem}
In the above theorem, the range of $h$ is approximately the same as Theorem~\ref{thm:balancing_good_regu} when $x_0^2+y_0^2$ is large, and consistent with Theorem~\ref{thm:no_balancing_bad_regu}. Especially, if $C=4$, we have 
    \begin{align*}
    (x_\infty-y_\infty)^2\le \frac{1}{2}(x_0-y_0)^2+x_0y_0,
\end{align*}
which has a similar upper bound as~\eqref{eqn:balancing_inequality_good_regu_a_C=4}  when $x_0y_0$ is small.

More important than proving balancing for more functions, the following theorem demonstrates that the balancing is lost when the regularity of the objective becomes bad.
\begin{theorem}[no Balancing]
\label{thm:no_balancing_bad_regu}

Consider ${b}=2n+1\text{ for }n\in\mathbb{Z}$ in~\eqref{eqn:functions}. Assume the initial condition satisfies $(x_0,y_0)\in \{(x,y):xy>2^{\frac{1}{b-1}}, x^2+y^2\ge 4 \}\backslash \mathcal{B}_b\text{, where }\mathcal{B}_b$ is a Lebesgue measure-0 set. Let the learning rate be $h= \frac{C}{(x_0^2+y_0^2+4)(x_0 y_0)^{2b-2}}\text{ for } 2\le C\le M_3(b,x_0,y_0)\text{, where }3<M_3(b,x_0,y_0)\le 4,$ and the precise definition of $M_3$ is given in Theorem~\ref{thm:bad_regularity_convergence}. 
Then GD converges to a global minimum  $(x_\infty,y_\infty)$, and we have
\begin{align*}
    (x_\infty-y_\infty)^2\ge (x_0-y_0)^2+\min\left\{2(x_0y_0-1)-\frac{2C}{b}x_0y_0,\Big(2-\frac{8b}{4b-C}\Big)(x_0y_0-1)\right\}
\end{align*}

\end{theorem}
The above theorem states that although a larger learning rate (i.e., larger $C$) can still reduce the lower bound of $(x_\infty-y_\infty)^2$,
this limiting difference may not be smaller than the initial difference $(x_0-y_0)^2$, especially when $b$ is large, 
i.e., there is no balancing in this case.
Note this does not contradict the stability theory (see Appendix~\ref{app:stability_analysis}), despite the fact that GD does not converge to a flatter minimum (see more discussions below). In addition, similar to the discussion in Section~\ref{sec:eos} on  EoS, a larger learning rate does not help recover the balancing in bad regularity cases either (see Figure~\ref{fig:bad_regularity_no_eos_not_lr_dependent}).
A more detailed version of Theorem~\ref{thm:no_balancing_bad_regu} can be found in Theorem~\ref{thm:no_balancing_bad_regu_Formal_version}.

\vskip2pt
\noindent$\bigstar$ \textbf{Example implications of the general theory.}
Consider the family of functions in \eqref{eqn:functions}, i.e., with both good and bad regularities.
\begin{itemize}
    \item \emph{Bad regularities eliminate the balancing phenomenon.} Consider the same set of initial conditions~\eqref{eqn:initial_condition_regularity_eliminates_phenomena} as in Section~\ref{sec:eos}.
     \begin{align*}
      \text{When }& \mathrm{dor}(F)\le 2,\text{ we have } (x_\infty-y_\infty)^2 \le \frac{2}{C}(x_0-y_0)^2 + \tilde{c},\\
      &\text{ where }\ 2.6\lesssim C=C(h)
      \le 4,\  \tilde{c}\text{ is a constant independent of }h,x_0,y_0,a.\\
      \text{When }& \mathrm{dor}(F)>2,\text{ we have } (x_\infty-y_\infty)^2\ge (x_0-y_0)^2-\hat{c}\frac{C}{b}\to (x_0-y_0)^2\text{ as }b\to\infty,\\
      &\text{where } 2\le C=C(h)\le M_3(b,x_0,y_0)\le 4\text,\ \hat{c}\text{ is a constant independent of }h,x_0,y_0,b.
   \end{align*}
   This shows that if the initial condition is near a sharp minimum, i.e., $(x_0-y_0)^2$ is large, then for functions with good (small) regularity, $(x_\infty-y_\infty)^2\lesssim(x_0-y_0)^2$, while for functions with bad (large) regularity, $(x_\infty-y_\infty)^2\gtrsim(x_0-y_0)^2$. 

   Similarly to Section~\ref{sec:eos}, a worse regularity (larger degree) also results in a larger set of initial conditions that do not lead to balancing.
\item \emph{When balancing appears, larger learning rate leads to more balancedness.} Note the above $C\approx h (x_0^2+y_0^2)$\footnote{Indeed, $C=h (x_0^2+y_0^2)$ for $\mathrm{dor}\le 1$; $C=h (x_0^2+y_0^2+4)$ for $\mathrm{dor}=2$}. Then, for $\mathrm{dor}\le 2$, we equivalently have 
\begin{align*}
    (x_\infty-y_\infty)^2 \lesssim \frac{2(x_0-y_0)^2/(x_0^2+y_0^2)}{h } + \bar{c},\quad\text{ where }\bar{c}\text{ is independent of }h,x_0,y_0,a.
\end{align*}
Namely, for any fixed initial condition, large $h$ reduces the gap between $x_\infty$ and $y_\infty$, i.e., brings more balancedness.
    \item \emph{Large learning rate does not necessarily help GD converge to a flatter minimum.} For bad regularity cases, $(x_\infty-y_\infty)^2$ does not decrease compared to its initial $(x_0-y_0)^2$. This implies that if starting near a sharp minimum (but not too close; see Section~\ref{sec:discussions}), GD will just converge to it instead of searching for flatter ones. Note this does not contradict our claim that large learning rate drives GD towards flatter regions. For functions with bad regularity, the minimizer large-learning-rate GD converges to is, roughly speaking, already flatter than its neighborhood (see Fig~\ref{fig:sharpness_goodregu_badregu}). As the degree of regularity increases, the difference in sharpness between this sharp minimum and its neighborhood also increases, i.e., the minimum becomes much flatter than its neighborhood.
    Thus converging to this sharp minimum is already a reflection of GD escaping to a flatter region.

\end{itemize}

\begin{remark}[EoS vs Balancing]
    Note that EoS and balancing are different phenomena despite both being large learning rate behaviors of GD. More precisely, balancing mainly stems from the \textbf{de-sharpening (pre-EoS)} stage in EoS but is almost unrelated to the \textbf{progressive sharpening} and \textbf{limiting stabilization of sharpness near $2/h$}. In contrast, the \textbf{pre-EoS} is not a prerequisite for EoS to manifest; instead, EoS relies on the latter two stages.
\end{remark}

\section{Convergence}
\label{sec:convergence}

The conditions required by theorems in both Section~\ref{sec:eos} and~\ref{sec:balancing} are to ensure the convergence of GD even with a large learning rate. The proofs of those theorems rely on a detailed understanding of this convergence. It is essential to emphasize that achieving convergence under large learning rates is a non-trivial task. Our results differ from existing milestones that demonstrate certain implicit biases based on unverified assumptions of convergence. Next, we will provide precise definitions of the settings used for the convergence results.

For the convenience of stating our convergence results, let us first begin with some additional \textbf{notations} and \textbf{definitions}:

Consider GD update for the objective function $f(x,y):=F(xy)$
\begin{align}
\label{eqn:gd}
\begin{pmatrix}x_{k+1}\\y_{k+1}\end{pmatrix}=\begin{pmatrix}x_{k}\\y_{k}\end{pmatrix}-h \ell_k \begin{pmatrix}
        0 & 1\\ 1& 0  \end{pmatrix}\begin{pmatrix}x_{k}\\y_{k}\end{pmatrix}=\begin{pmatrix}
        1 & -h\ell_k\\ -h\ell_k& 1  \end{pmatrix}\begin{pmatrix}x_{k}\\y_{k}\end{pmatrix},
\end{align}
where $\ell_k=F'(x_ky_k)$.

By GD update~\eqref{eqn:gd}, we have
\begin{align*}
    x_{k+1}y_{k+1}-1=(x_k y_k-1)\left(1-h\frac{\ell_k}{x_ky_k-1}(x_k^2+y_k^2-h\ell_k x_k y_k)\right)=r_k(x_k y_k-1).
\end{align*}
where $r_k=1-h\frac{\ell_k}{x_ky_k-1}(x_k^2+y_k^2-h\ell_k x_k y_k)$. Let $\delta:=xy-1$.  
We define 
\begin{align*}
q(\delta):=\frac{F'(xy)}{xy-1}=\frac{F'(\delta+1)}{\delta},\text{ and then }q_k:=q(x_ky_k-1)=\frac{\ell_k}{x_ky_k-1}.
\end{align*}
Moreover, for $0<a\le 1$, 
\begin{align}
\label{eqn:q(delta)_bound_c_1}
   1-c_1\delta^2\le  q(\delta)\le 1 ,\text{ where }c_1=-\frac{-3+3 a-2 \log(2)}{24 \log(2)}.
\end{align}

Next, we prove the convergence of GD under large learning rates for the functions~\eqref{eqn:functions} in the following theorems. Note GD still converges when $h$ is smaller, but that is already guaranteed by classical optimization theory and not our focus here. See more discussions in the last paragraph of this section.
\begin{theorem}
\label{thm:good_regularity_convergence}
Consider $0<a\le 1$ in~\eqref{eqn:functions}. Assume the initial condition 
    $$(x_0,y_0)\in \{(x,y):1<xy\lesssim q^{-1}\left( \frac{q(1)}{2} \right), x^2+y^2\gtrsim 4c_1^{-4/3}\}\backslash \mathcal{B}_a$$
for some Lebesgue measure-0 set $\mathcal{B}_a$. Let the learning rate be
  $$ \frac{2}{x_0^2+y_0^2}<\frac{\frac{2}{q(1)}}{x_0^2+y_0^2}+\mathcal{O}\left(\frac{1}{(x_0^2+y_0^2)^2}\right)\le h\le\frac{4}{x_0^2+y_0^2}.$$
Then GD converges to a global minimum.
\end{theorem}

\begin{theorem}[Theorem 3.1 in \cite{wang2022large}]
\label{thm:b=1_convergence}
    Consider $b=1$ in~\eqref{eqn:functions}. Assume the initial condition 
    $$(x_0,y_0)\in \{(x,y): x^2+y^2\ge 8\}\backslash \mathcal{B}$$
    for some Lebesgue measure-0 set $\mathcal{B}$.
Let the learning rate be
   $$\frac{2}{x_0^2+y_0^2+4}\le h\le\frac{4}{x_0^2+y_0^2+4}.$$
Then GD converges to a global minimum.
\end{theorem}

\begin{theorem}\label{thm:bad_regularity_convergence}
Consider ${b}=2n+1\text{ for }n\in\mathbb{Z}$ in~\eqref{eqn:functions}. Assume the initial condition 
    $$(x_0,y_0)\in \{(x,y):xy>2^{\frac{1}{b-1}}, x^2+y^2\ge 4\}\backslash \mathcal{B}_b,\text{ for some Lebesgue measure-0 set }\mathcal{B}_b.$$
Let the learning rate be
   $$\frac{2}{(x_0^2+y_0^2+4)(x_0 y_0)^{2b-2}} \le h\le \frac{M_3}{(x_0^2+y_0^2+4)(x_0 y_0)^{2b-2}},$$
where
   $M_3=\min\Big\{4,\arg\max_C\Big\{\Big(1+\big(\frac{C}{(x_0^2+y_0^2+4)(x_0 y_0)^{2b-2}}\big)^2\ell_0^2\Big)x_0 y_0-\frac{C}{(x_0^2+y_0^2+4)(x_0 y_0)^{2b-2}}\ell_0 (x_0^2+y_0^2)\ge\epsilon\Big\}\Big\}>3,$
for some small $\epsilon>0$.
Then GD converges to a global minimum. Moreover, we have the rate of convergence 
\begin{align*}
    |x_ky_k-1|\le |x_1y_1-1| S^{k-1},\quad \text{for }k\ge 1,
\end{align*}
where
$S=\begin{cases}
        1-h(2-h\ell_0),&\text{ if }r_0>0\\
        1-h(2-h)q_1x_1y_1,& \text{otherwise}
    \end{cases},\text{ and }0<S<1.$

\end{theorem}
Note the convergence rate quantified by Theorem~\ref{thm:bad_regularity_convergence} is a global one. This is the first characterization of the convergence speed of GD under large learning rates for non-convex functions~\eqref{eqn:functions} ($b\ge 3$) that do not even have globally Lipschitz gradients
(note the quantification of convergence speed of GD for nonconvex functions without Lipschitz gradient is already highly nontrivial, even for regular learning rates; examples of great work in that direction include \cite{zhang2019gradient, ward2023convergence, li2023convex}).

The main idea of the proof is to construct a quantitative version of the two-phase 
convergence pattern of large learning rate GD as discussed in Section~\ref{sec:eos}. That is, GD first tries to escape from the sharp region to a flat region (\emph{Phase 1}) and then converges inside the flat one (\emph{Phase 2}). For more detailed discussions of this two-phase convergence and its interpretation under large learning rates, please see Section~\ref{sec:eos}.

The convergence of GD also holds in the regular learning rate regime (we omit it here since it is not our main focus). The analysis is more standard although still non-trivial for non-convex functions. The convergence pattern for regular learning rate contains only \emph{Phase 2} and GD will converge to a nearby minimum.

\begin{table}[ht]
\caption{Degrees of regularity (dor) for different compositions of loss and activation functions.}
\label{tab:dor_ml}
\begin{center}
\begin{tabular}{l|lll}

 & tanh & ReLU & $\mathrm{ReLU}^3$
\\ \hline 
huber &0        &1 &3 \\
$\ell^2$ & 0    & 2 & 6
\end{tabular}
\end{center}
\end{table}

\begin{figure}[ht]
    \centering    \subfigure[huber+tanh $\to$ EoS]{{\includegraphics[width=0.3\textwidth]{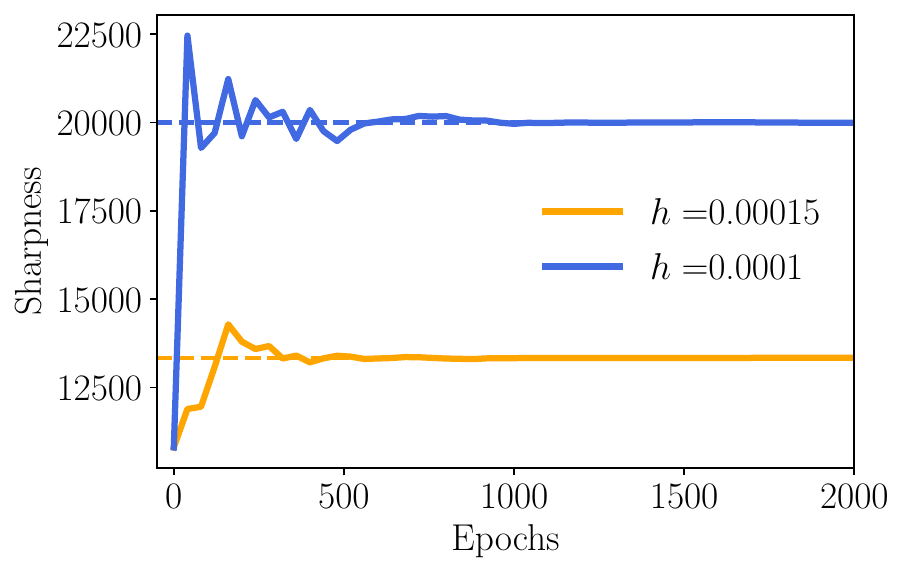} }\label{subfig:sharpness_huber_tanh}}%
    \subfigure[huber+ReLU $\to$ EoS]{{\includegraphics[width=0.3\textwidth]{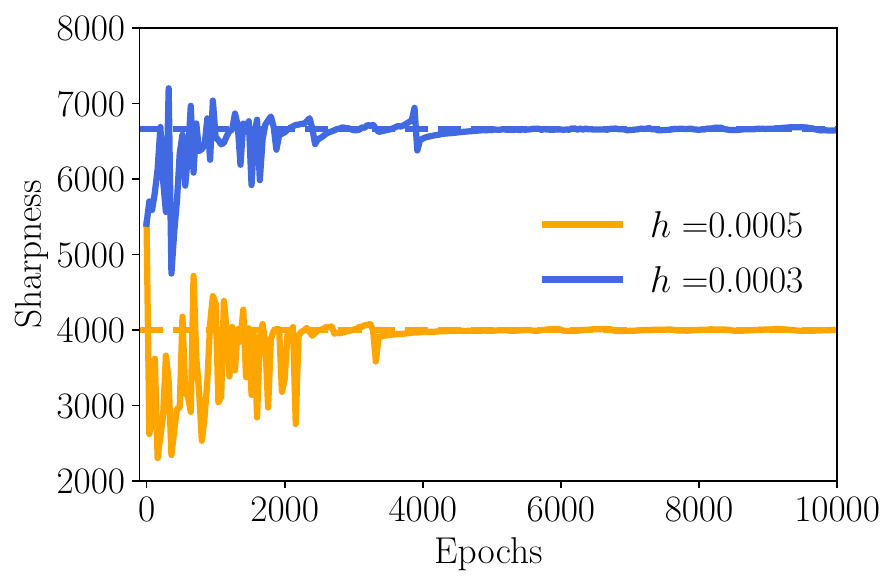} }\label{subfig:sharpness_huber_relu}}
    \subfigure[huber+$\mathrm{ReLU}^3$ $\to$ no EoS]{{\includegraphics[width=0.3\textwidth]{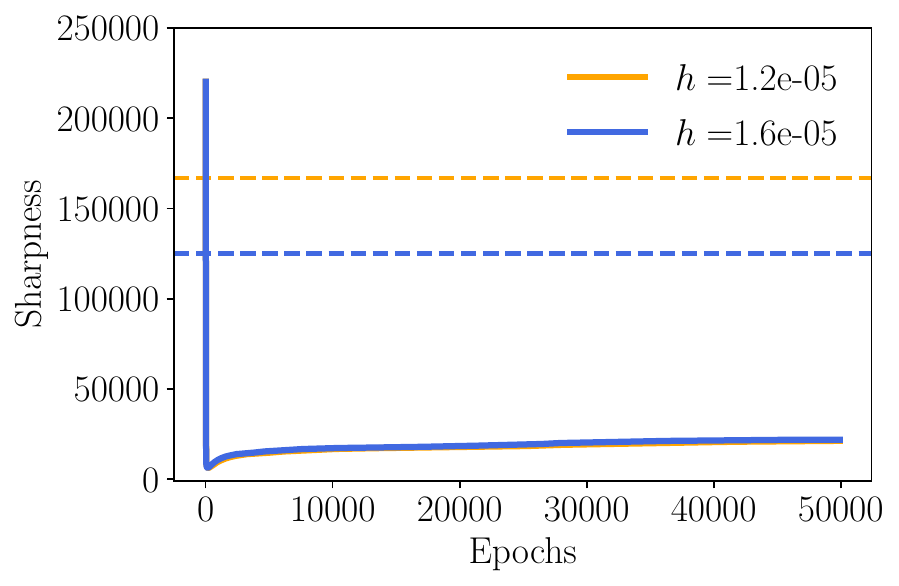} }\label{subfig:sharpness_huber_cubic_relu}}
    
    \subfigure[$\ell^2$+tanh $\to$ EoS]{{\includegraphics[width=0.3\textwidth]{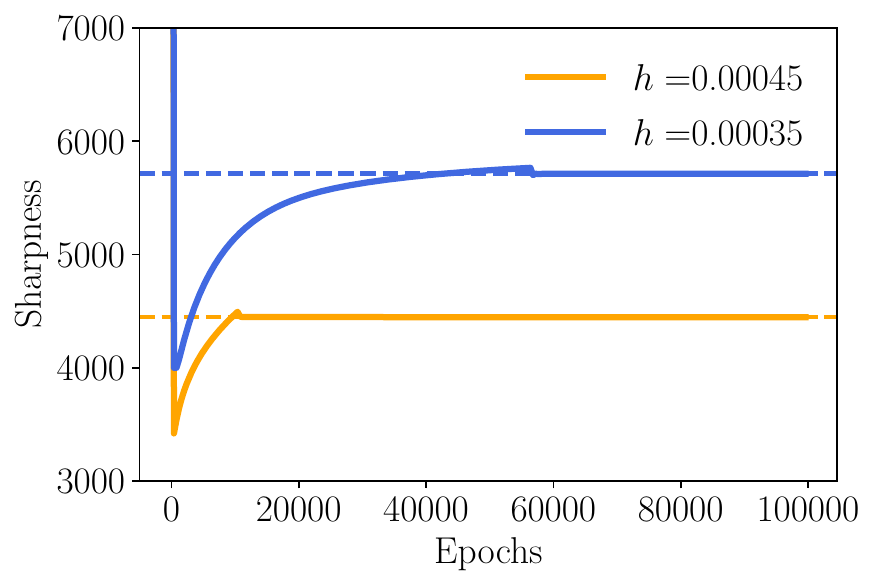} }\label{subfig:sharpness_mse_tanh}}
    \subfigure[$\ell^2$+ReLU $\to$ no EoS]{{\includegraphics[width=0.3\textwidth]{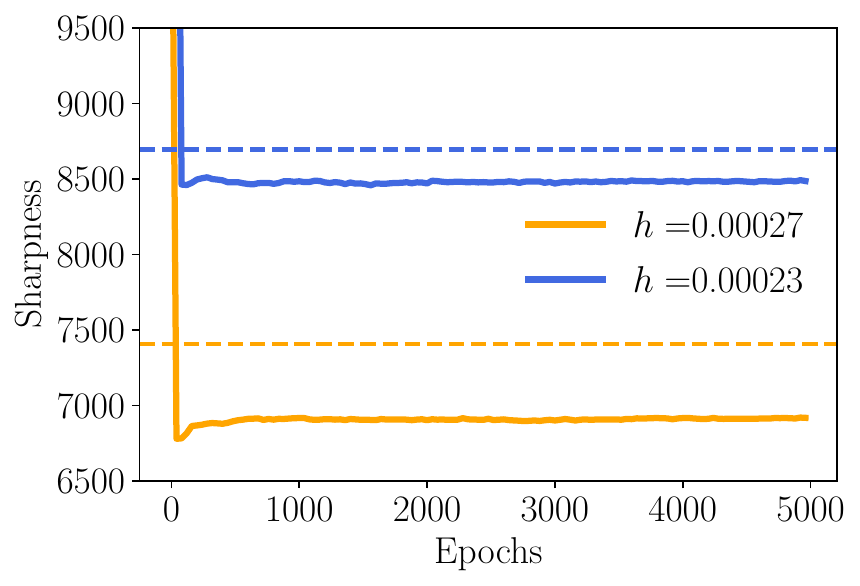} }\label{subfig:sharpness_mse_relu}}
    \subfigure[$\ell^2$+$\mathrm{ReLU}^3$ $\to$ no EoS]{{\includegraphics[width=0.3\textwidth]{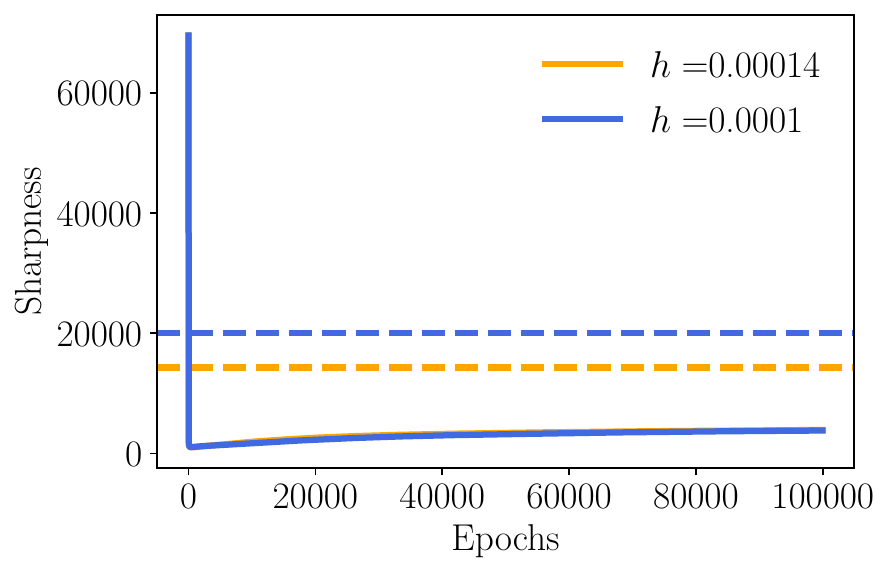} }\label{subfig:sharpness_mse_cubic_relu}}
    \caption{EoS and non-EoS phenomena. The solid lines are the evolution of sharpness in different situations; the dashed lines are $2/h$ for different learning rates $h$. 
    }
    \label{fig:sharpness_neural_networks}%
\end{figure}
\begin{figure}[ht]
    \centering    \subfigure[huber+tanh $\to$ balancing]{{\includegraphics[width=0.3\textwidth]{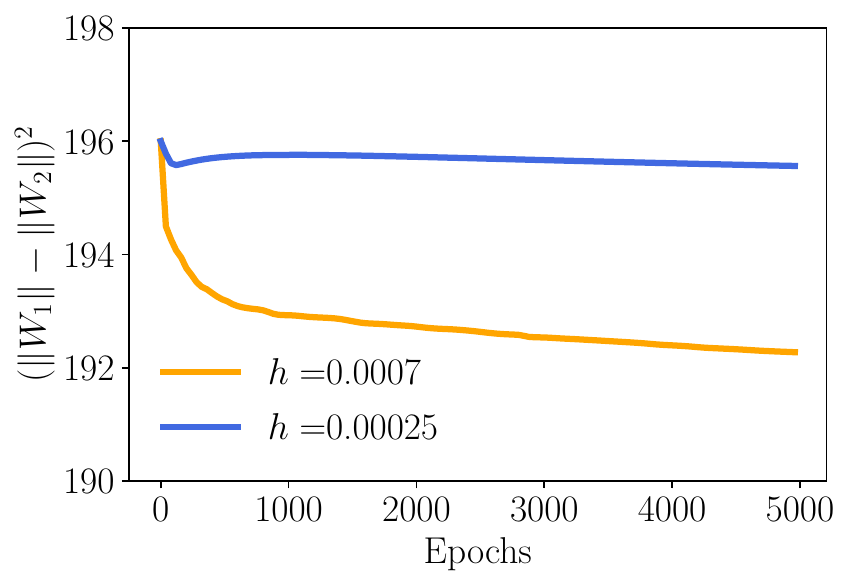} }\label{subfig:balancing_huber_tanh}}%
    \subfigure[huber+ReLU $\to$ balancing]{{\includegraphics[width=0.3\textwidth]{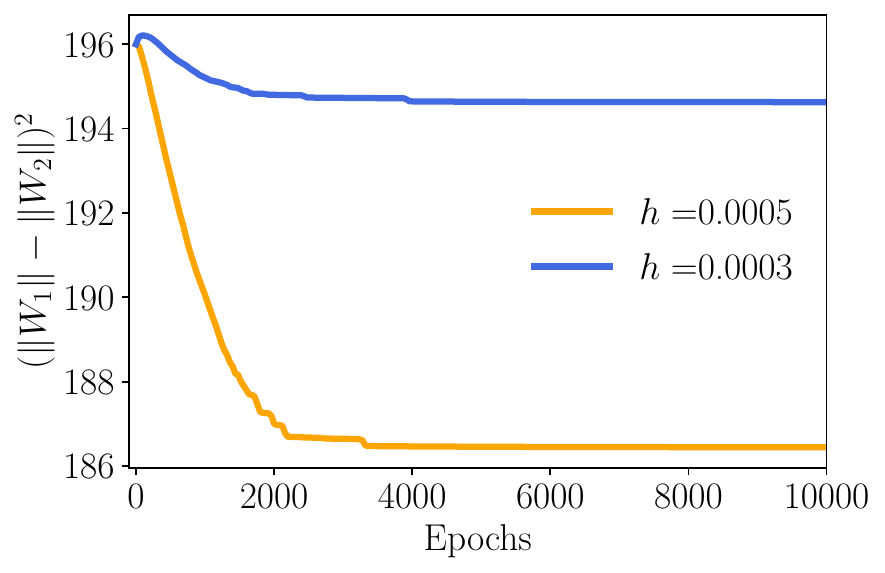} }\label{subfig:balancing_huber_relu}}
    \subfigure[huber+$\mathrm{ReLU}^3$ $\to$ no balancing]{{\includegraphics[width=0.3\textwidth]{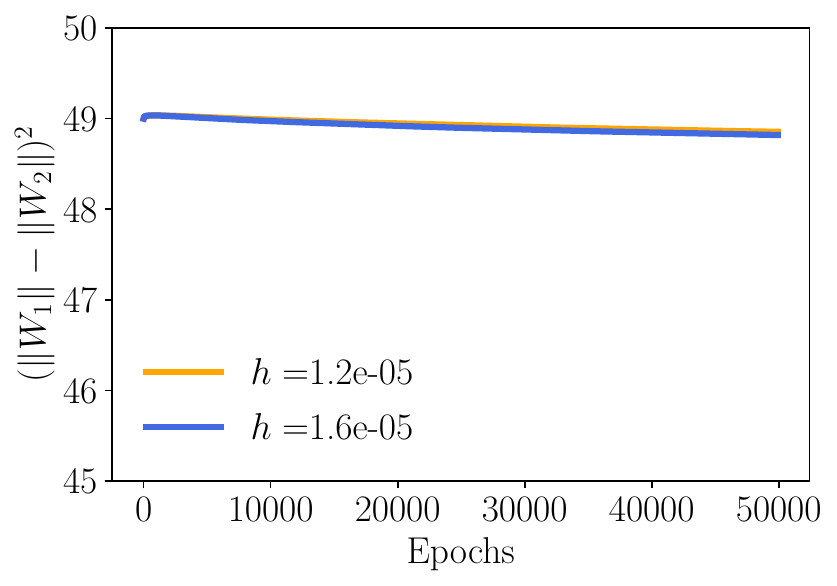} }\label{subfig:balancing_huber_cubic_relu}}
    
    \subfigure[$\ell^2$+tanh $\to$ balancing]{{\includegraphics[width=0.3\textwidth]{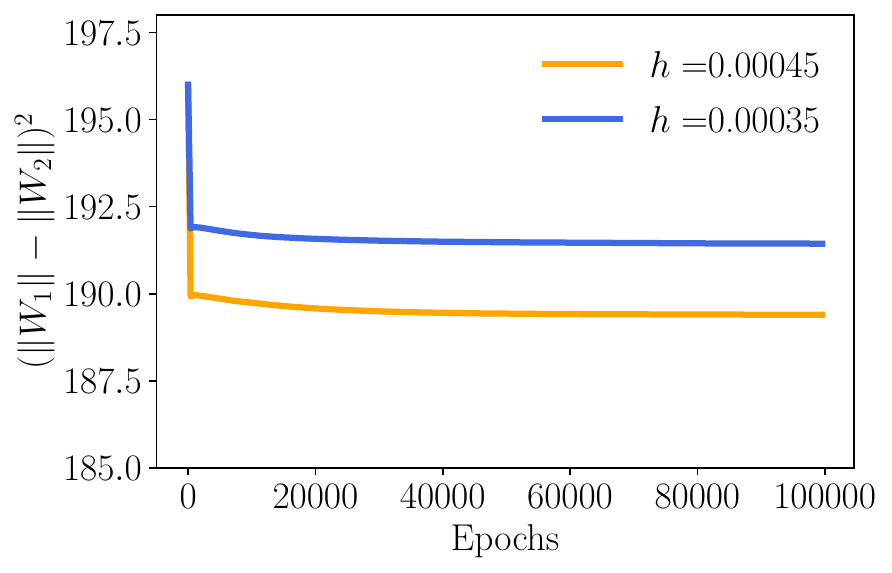} }\label{subfig:balancing_mse_tanh}}
    \subfigure[$\ell^2$+ReLU $\to$ balancing]{{\includegraphics[width=0.3\textwidth]{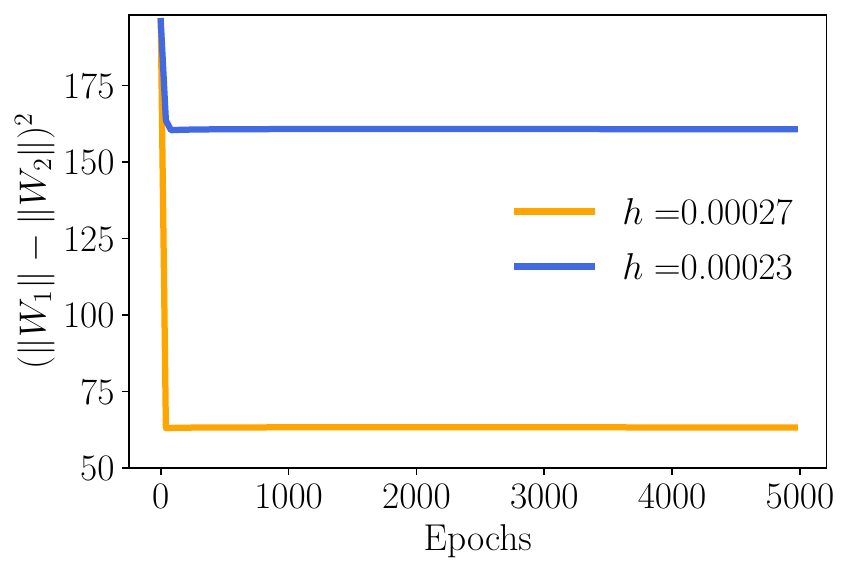} }\label{subfig:balancing_mse_relu}}
    \subfigure[$\ell^2$+$\mathrm{ReLU}^3$ $\to$ no balancing]{{\includegraphics[width=0.3\textwidth]{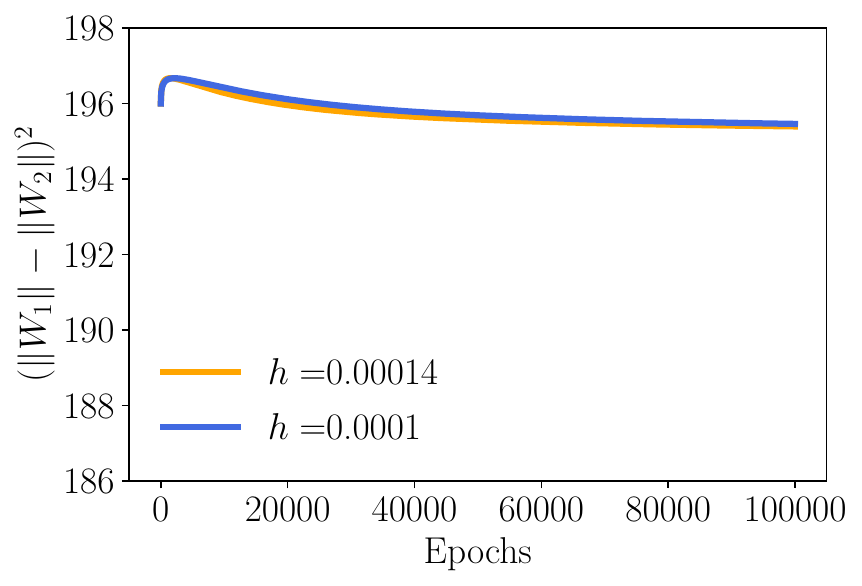} }\label{subfig:balancing_mse_cubic_relu}}
    \caption{Balancing and non-balancing. For (a) huber+tanh, we use larger learning rates than that in Figure~\ref{fig:sharpness_neural_networks} to further highlight the balancing phenomenon; all the other cases use the same learning rates as in Figure~\ref{fig:sharpness_neural_networks}.
    }
    \label{fig:balancing_neural_networks}%
\end{figure}

\section{Experiments}
\label{sec:experiments}

This section is dedicated to demonstrating that our results still hold in more complicated setups corresponding to deep learning problems, and regularity is an essential factor determining large learning rate implicit biases. We also verify that these biases only occur when learning rates are sufficiently \emph{large}. Moreover, batch normalization, a popular technique for improving training and generalization, indeed works and can turn the neural network training objective that has bad regularity into one with good regularity, creating pleasant large learning rate implicit biases. 
Detailed setup and more experiments can be found in Appendix~\ref{app:experiments}.

\vskip2pt
\noindent$\bigstar$ \textbf{Good regularity vs bad regularity.} In this section we focus on 6 different objective functions, whose degrees of regularity are computed under \eqref{eqn:dor_3layer_toy} for objective function \eqref{eqn:nn_objective_toy_3_layer} with various loss (huber and $\ell^2$)and activation (tanh, ReLU, and $\mathrm{ReLU}^3$) functions (see more details in Section~\ref{subsec:regularity_neural_network}) in Tab.~\ref{tab:dor_ml}. Although the neural network models used in the experiments are mostly 2-layer which is different from~\eqref{eqn:nn_objective_toy_3_layer} (except for huber+$\mathrm{ReLU}^3$, where we use the same structure as~\eqref{eqn:nn_objective_toy_3_layer} to highlight the effect of regularity; see details in Appendix~\ref{app:experiments}), Tab.~\ref{tab:dor_ml} still helps illustrate the insight of regularity affecting on large learning rate phenomena in different models. Figure~\ref{fig:sharpness_neural_networks} and~\ref{fig:balancing_neural_networks} shows the evolution of sharpness and $(\|W_1\|_{\rm F}-\|W_2\|_{\rm F})^2$ along GD iterations respectively under the six cases of Tab.~\ref{tab:dor_ml}. When $\mathrm{dor}>1$, there is no EoS (see Figure~\ref{subfig:sharpness_mse_relu}~\ref{subfig:sharpness_huber_cubic_relu}~\ref{subfig:sharpness_mse_cubic_relu}); when $\mathrm{dor}>2$, there is no balancing (see Figure~\ref{subfig:balancing_huber_cubic_relu}~\ref{subfig:balancing_mse_cubic_relu}). This validates our conjecture that in neural network models, bad regularity can also eliminate large learning rate phenomena.

\vskip2pt
\noindent$\bigstar$ \textbf{Restoration of EoS: batch normalization.} In Figure~\ref{fig:bn_bad_regularity_ml}, we add batch normalization to the three bad regularity cases that do not lead to EoS (see Figure~\ref{subfig:sharpness_mse_relu}~\ref{subfig:sharpness_huber_cubic_relu}~\ref{subfig:sharpness_mse_cubic_relu}). It is shown in the figures that EoS reappears for each case, which shows that batch normalization turns bad regularity into good one and consequently lead to large learning rate phenomena.

\vskip2pt
\noindent$\bigstar$ \textbf{Large learning rate vs small learning rate.} Figure~\ref{fig:ml_small_lr_vs_large_lr} shows two examples where we use both small and large learning rates for optimization. It turns out that a small learning rate can indeed vanish the large learning rate phenomena.

\vskip2pt
\noindent$\bigstar$ \textbf{Independence of datasets: CIFAR-10 and MNIST.} All the results above are performed on CIFAR-10. In Figure~\ref{fig:mnist_eos_balancing}, we also show results on MNIST, which are consistent with those on CIFAR-10. See more results in Appendix~\ref{app:experiments}.

\begin{figure}[ht]
    \centering    
    \subfigure[$\ell^2$+ReLU]{\includegraphics[width=0.3\textwidth]{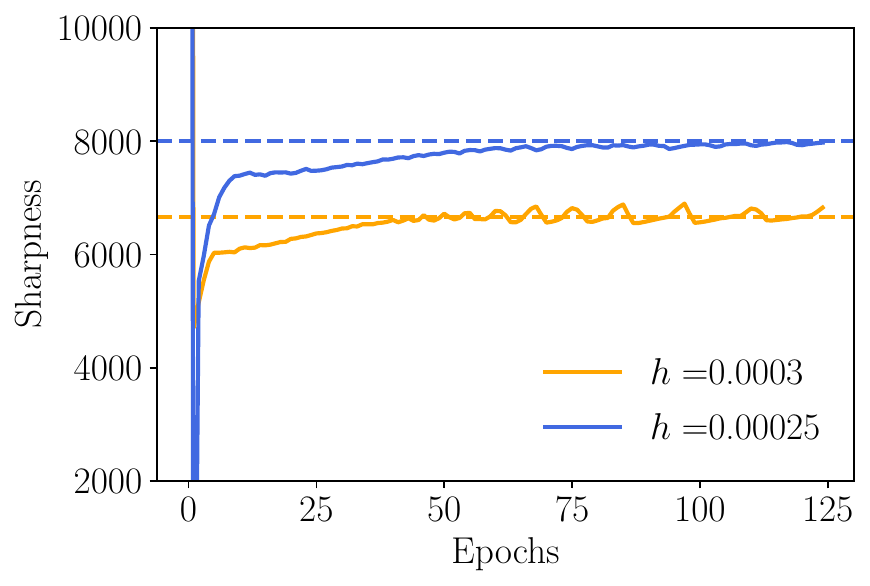}}
    \subfigure[huber+$\mathrm{ReLU}^3$]{\includegraphics[width=0.3\textwidth]{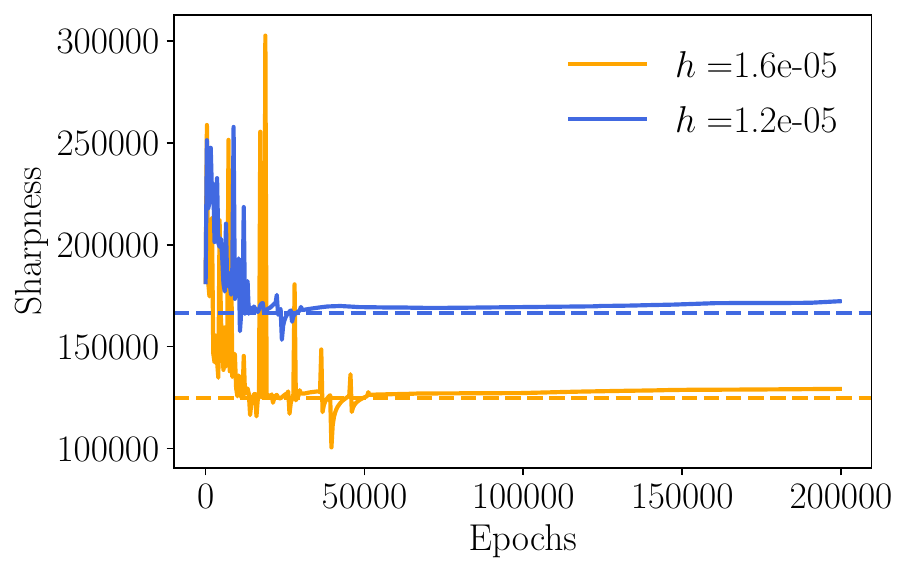}}
    \subfigure[$\ell^2$+$\mathrm{ReLU}^3$]{\includegraphics[width=0.3\textwidth]{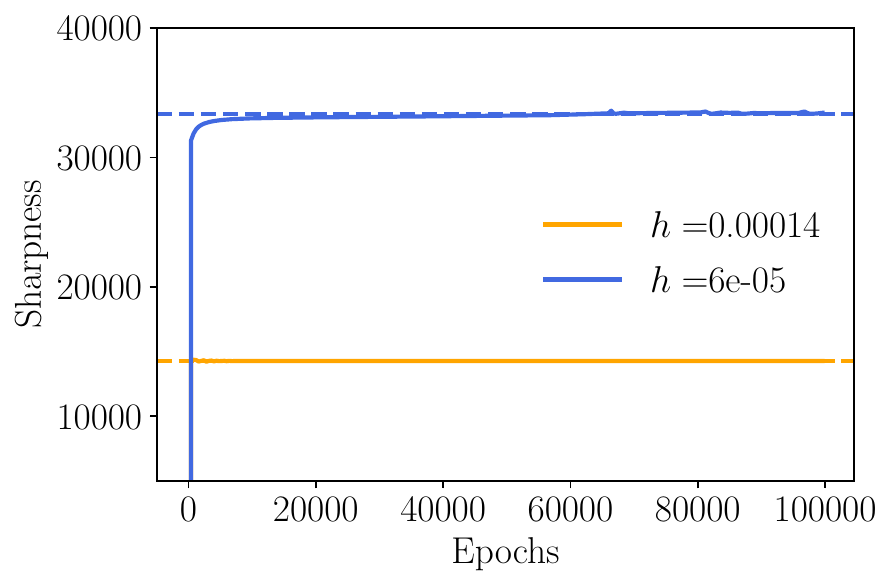}}
    
    \caption{Restoration of EoS for bad regularity cases under batch normalization.}
    \label{fig:bn_bad_regularity_ml}
\end{figure}

\begin{figure}[ht]
    \centering
    \subfigure[huber+ReLU EoS]{\includegraphics[width=0.24\textwidth]{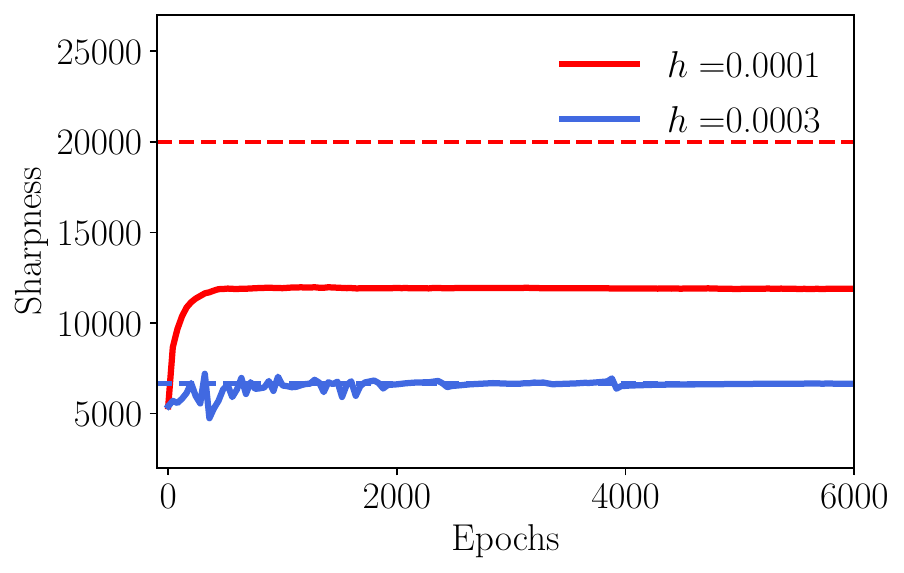}}
    \subfigure[huber+ReLU balancing]{\includegraphics[width=0.24\textwidth]{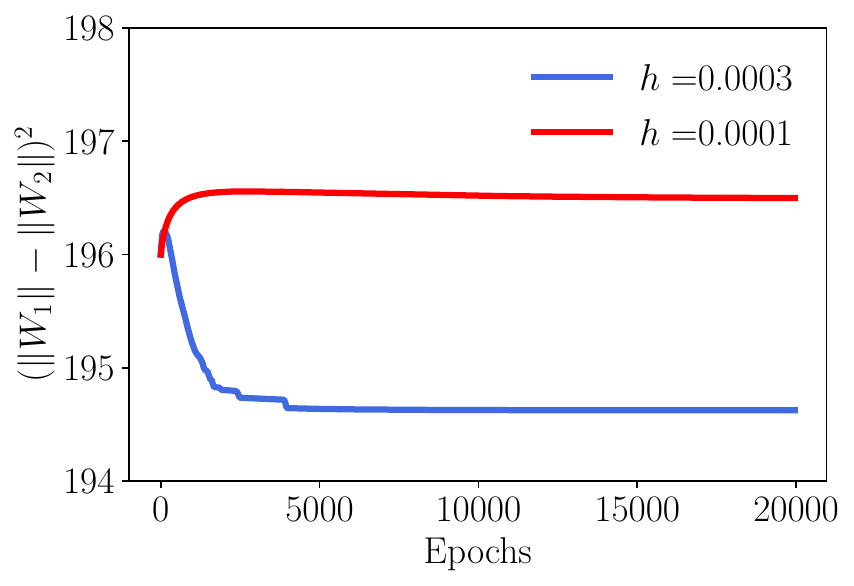}}
    \subfigure[$\ell^2$+tanh EoS]{\includegraphics[width=0.24\textwidth]{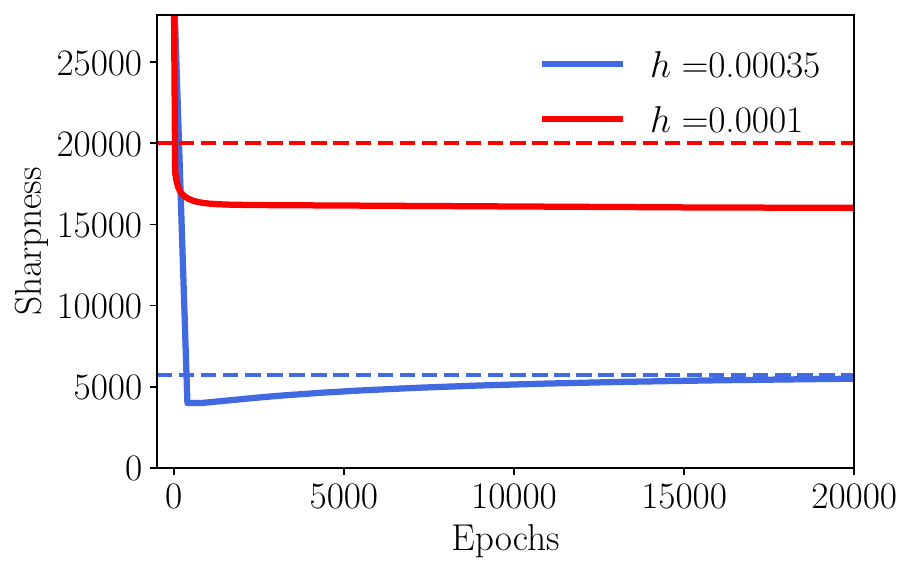}}
    \subfigure[$\ell^2$+tanh balancing]{\includegraphics[width=0.24\textwidth]{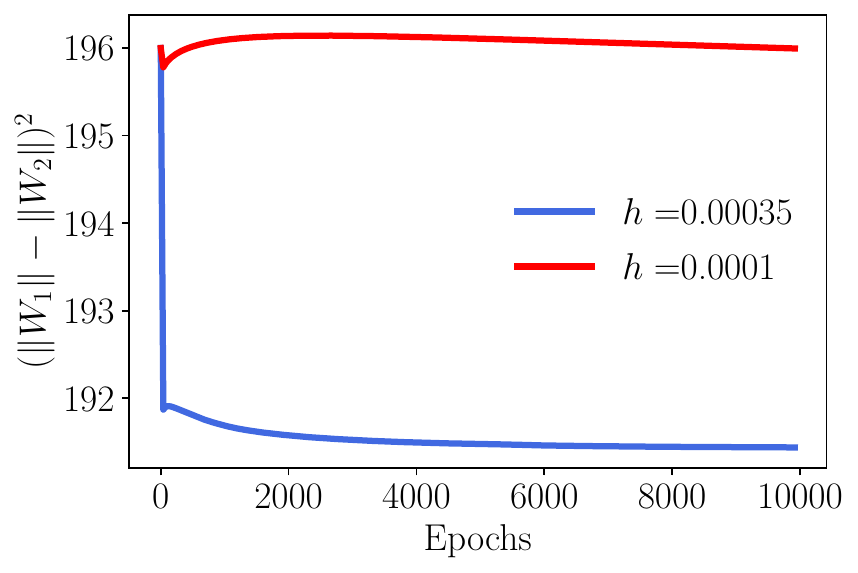}}
    \caption{Small learning rate (red) vs large learning rate (blue).}
    \label{fig:ml_small_lr_vs_large_lr}
\end{figure}

\begin{figure}[ht]
    \centering
    \subfigure[huber+tanh $\to$ EoS]{\includegraphics[width=0.24\textwidth]{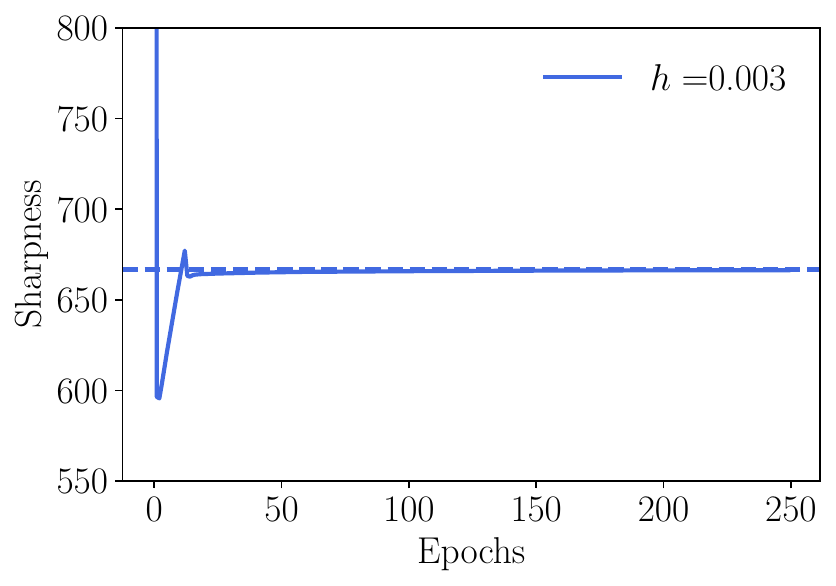}}
    \subfigure[huber+tanh $\to$ balancing]{\includegraphics[width=0.24\textwidth]{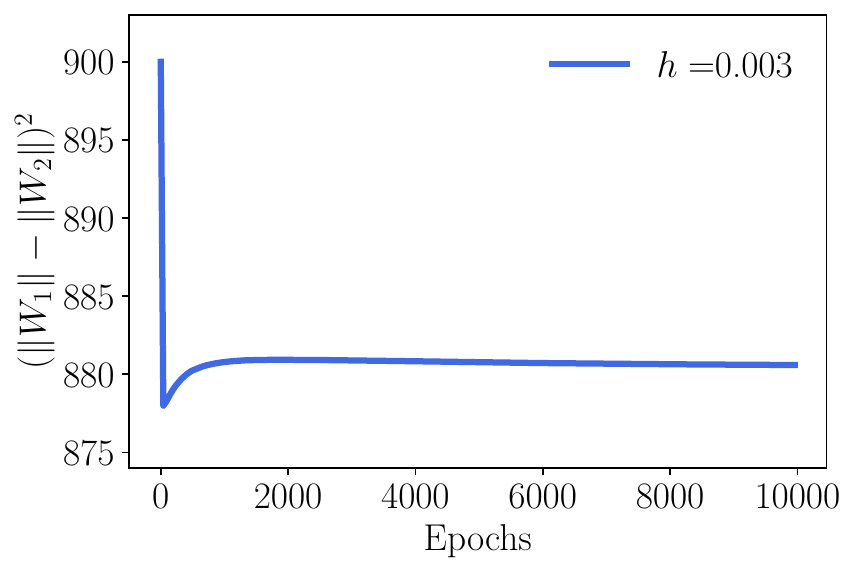}}
    \subfigure[$\ell^2$+ReLU $\to$ no EoS]{\includegraphics[width=0.24\textwidth]{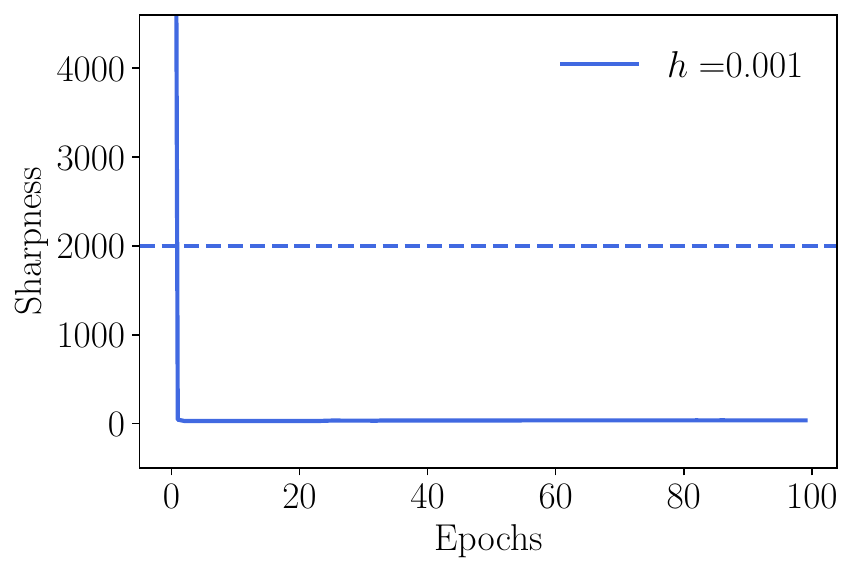}}
    \subfigure[$\ell^2$+ReLU $\to$ no balancing]{\includegraphics[width=0.24\textwidth]{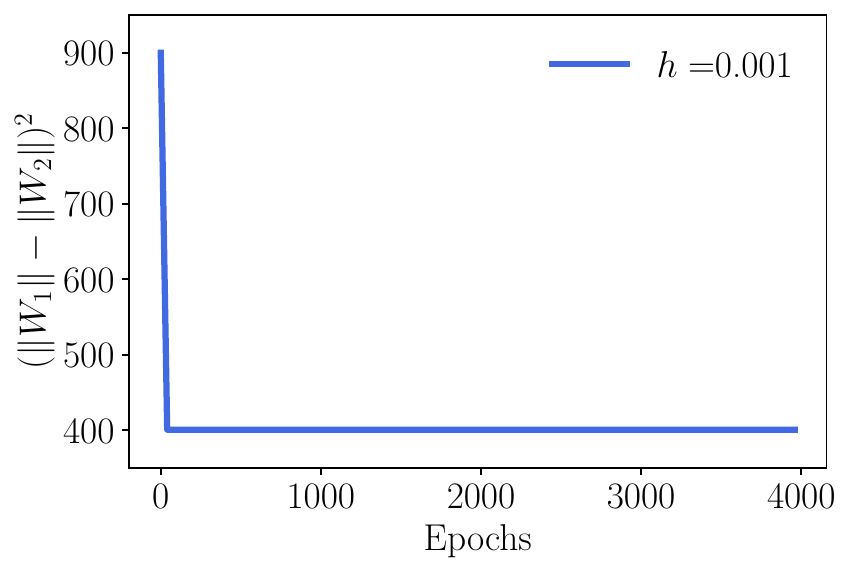}}
    \caption{Large learning rate phenomena on MNIST.}
    \label{fig:mnist_eos_balancing}
\end{figure}

\begin{figure}[h]
    \centering
    \includegraphics[width=0.8\textwidth]{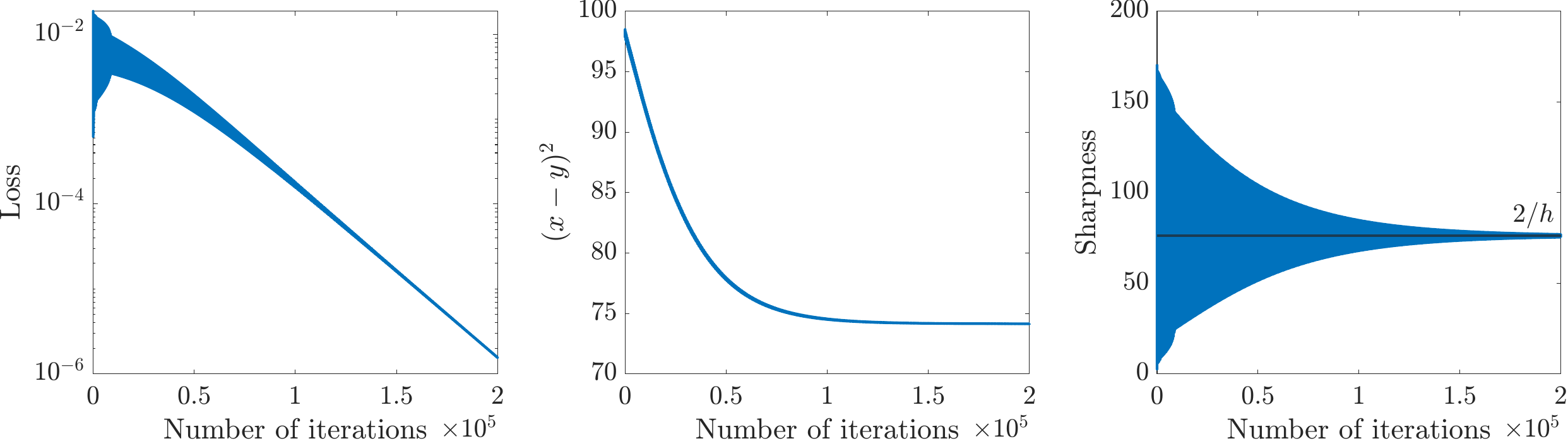}
    \caption{Appearance of large learning rate phenomena in `good' region of functions with bad regularity. The objective function is~\eqref{eqn:functions} with $b=3$; the initial condition is $x=10,y=0.11$; the learning rate is  $h=\frac{4}{(x_0^2+y_0^2+4)(x_0 y_0)^{2b-2}}$.}
    \label{fig:bad_regularity_but_large_lr_phenomena}
\end{figure}
\begin{figure}[h]
    \centering
    \includegraphics[width=0.8\textwidth]{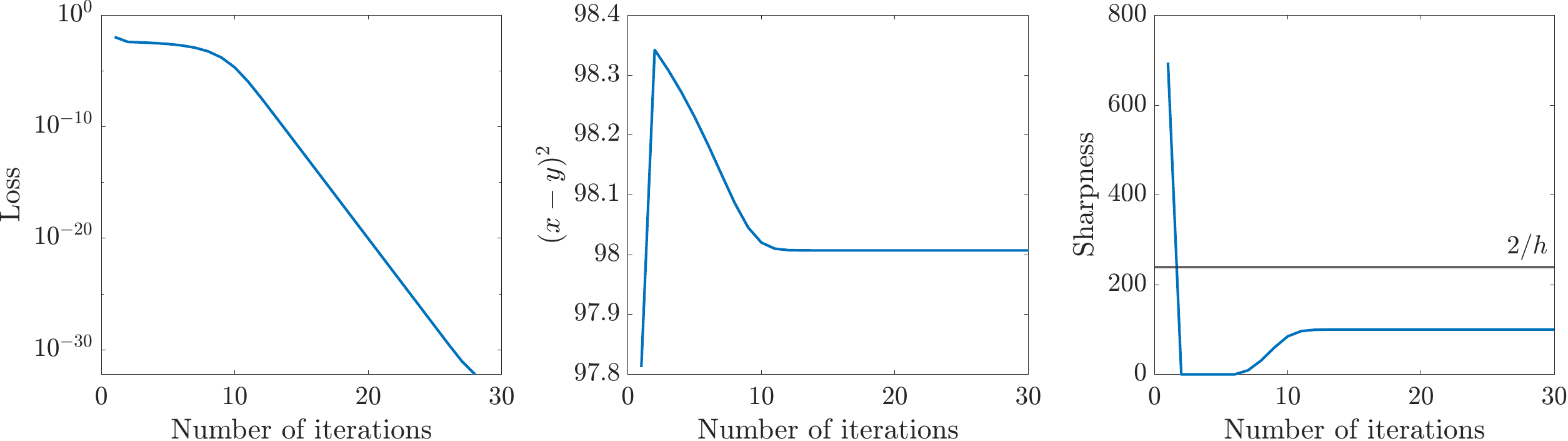}
    \caption{Disappearance of large learning rate phenomena in `good' region of functions with bad regularity. This Figure uses the same initialization and rule of choosing learning rate with $b=9$.}
    \label{fig:bad_regularity_but_large_lr_phenomena_no_eos}
\end{figure}

\section{Discussions}
\label{sec:discussions}

\noindent$\bigstar$ \textbf{Large learning rate phenomena under bad regularity.} Our findings indicate that, for a substantial set of initial conditions and within our family of functions, the occurrence of large learning rate phenomena becomes less likely 
as the regularity of objective function worsens (i.e., a larger degree of regularity). In the meantime, a small portion of `good' regions of initial conditions may still exist, for which the large learning rate phenomena (partially) persist even for functions with bad regularities. For example, when the initial condition is very close to the minima, Figure~\ref{fig:bad_regularity_but_large_lr_phenomena} shows that balancing and the limiting behavior (sharpness near $2/h$) of EoS still occur, although there is no progressive sharpening stage. A similar limiting behavior of EoS is theoretically described in~\cite{zhu2023understanding} with $b=2$. Figure~\ref{fig:bad_regularity_but_large_lr_phenomena_no_eos} shows that when we increase the degree of regularity, large learning rate phenomena will still disappear even if we use the same initialization as Figure~\ref{fig:bad_regularity_but_large_lr_phenomena}.

\vskip2pt
\noindent$\bigstar$ \textbf{Beyond our function class.}
Within our function class~\eqref{eqn:functions}, regularity is a sufficient and necessary condition for the implicit biases of large learning rate discussed here; however, beyond these functions, more elements can contribute to these phenomena.
One of them is symmetry. For example, in~\eqref{eqn:functions}, if we consider the asymmetrical version $\log(\exp(s-1)+1)$, the minimizer is at $-\infty$ with 0 Hessian. This means there is no `climb up the hill' for sharpness towards $2/h$ and thus no EoS. Additional factors may also affect these phenomena, such as the involvement of bias in the neural network, the usage of multiple data points, and more complicated network architectures. We will leave these for future explorations.

\vskip2pt
\noindent$\bigstar$ \textbf{The effect of depth.} Our framework currently lacks the analysis of the effect of depth, due to the fact that it is not a multivariate theory. Extending the current techniques and/or results to multivariate cases is highly nontrivial. 
Besides, there might be additional consequences or phenomena of increasing depth, for example, more complicated balancing between subsets of layers compared to the two-layer case, where balancing can only occur between the two layers. We will also need to leave this for future exploration.  

\vskip2pt
\noindent$\bigstar$ \textbf{Comparison between our analysis techniques and existing ones. } 
\citet{damian2022self} and~\citet{wang2022analyzing} analyzed the generic objective functions and therefore required extra assumptions involving information along the GD trajectory. These assumptions cannot be easily verified theoretically. In contrast, our work follows \citet{wang2022large}, \citet{zhu2023understanding}, \citet{ahn2022learning}, and \citet{song2023trajectory} by constructing explicit example functions (see detailed discussions in Section~\ref{subsec:regularity_Functions}). Specifically, 
inspired by~\citet{wang2022large}, we deal with more complicated objective functions that are structurally close to how neural network composes activation and loss functions. Our theory is enabled by a new analysis technique, namely scrutinizing the co-evolution of two auxiliary quantities\footnote{Note here both $x^2+y^2$ and $xy$ are generic expressions that work for various functions, as the former is the sharpness at minimizer, and the later is the minimizer; all our objective functions are designed such that these two expressions remain the same so that a comparison across different functions is fair.} $x^2+y^2$ and $xy$ through the entire GD process, which helps us provide a first characterization of the transition between the first two stages of EoS.

\section*{Acknowledgment}
YW and MT are partially supported by NSF DMS-1847802, NSF ECCS-1942523, Cullen-Peck Scholarship, and GT-Emory AI. Humanity Award. The authors thank Sinho Chewi, Yuejie Chi, Jeremy Cohen, Rong Ge, Francesco Orabona, Suvrit Sra for inspiring discussions, and John Zhang for proofreading the paper.

\bibliography{ref}

\clearpage
\appendix

\section*{Appendix}

\section{Necessary condition of convergence from stability theory}
\label{app:stability_analysis}

For an objective function $f(u)$, consider the GD update in terms of a iterative map $\psi$
\begin{align*}
    u_{k+1}=\psi (u_k):=u_k-h\nabla f(u_k).
\end{align*}
We further consider a stationary point $u^*$ of the objective function $f$, i.e., $\nabla f(u^*)=0$. Then this point $u^*$ is a fixed point of the map $\psi$ since
\begin{align*}
    u^*=u^*-h\nabla f(u^*)=\psi(u^*).
\end{align*}
If all the magnitudes of the eigenvalues of Jacobian matrix $\nabla\psi(u^*)$ are less than 1, $u^*$ is a stable fixed point (see more explanation in Section 5 of~\citet{wang2022large}). Consequently, we have the following theorem
\begin{theorem}[Necessary condition, see for example~\citet{NEURIPS2018_6651526b}]
\label{thm:stbility_analysis_2/L_bound}
    Let $u^*$ be a local minimum point of $f(u)$ and consider GD updates. If $-I\prec I-h\nabla^2f(u^*)\prec I$, i.e., $h<\frac{2}{L^*}$, where $\nabla^2f(u^*)\preceq L^*I$, we have that $u^*$ is a stable fixed point of GD map.
\end{theorem} 
Note the above theorem is a necessary condition of the convergence of GD to a minimizer. This is due to the fact that if $h>\frac{2}{L^*}$, there exists at least one eigendirection s.t. the magnitude of its eigenvalue is greater than 1. Namely, there will be an unstable direction of the map that prevents GD from converging towards the point $u^*$.

\section{Preparation for proofs}
\label{app:preparation_for_proofs}
Before the proofs, we first take a closer look at the GD iteration for the function $f(x,y)=F(xy)$~\eqref{eqn:functions}
\begin{align*}
\begin{pmatrix}x_{k+1}\\y_{k+1}\end{pmatrix}=\begin{pmatrix}x_{k}\\y_{k}\end{pmatrix}-h \ell_k \begin{pmatrix}
        0 & 1\\ 1& 0  \end{pmatrix}\begin{pmatrix}x_{k}\\y_{k}\end{pmatrix}=\begin{pmatrix}
        1 & -h\ell_k\\ -h\ell_k& 1  \end{pmatrix}\begin{pmatrix}x_{k}\\y_{k}\end{pmatrix},
\end{align*}
where $\ell_k=F'(x_ky_k)$. Let $u_k=\begin{pmatrix}
    x_k\\y_k
\end{pmatrix}$. Then
\begin{align}
\label{eqn:u^top_u_one_step}
    u_{k+1}^\top u_{k+1}=(1+h^2\ell_{k}^2)u_{k}^\top u_{k}-4h\ell_{k} x_{k} y_{k},
\end{align}
\begin{align}
\label{eqn:u^top_u_two_steps}
    u_{k+2}^\top u_{k+2}&=(1+h^2\ell_{k+1}^2)u_{k+1}^\top u_{k+1}-4h\ell_{k+1} x_{k+1} y_{k+1}\notag\\
    &=(1+h^2\ell_{k+1}^2)((1+h^2\ell_{k}^2)u_{k}^\top u_{k}-4h\ell_{k} x_{k} y_{k})-4h\ell_{k+1} x_{k+1} y_{k+1}.
\end{align}

\begin{lemma}
\label{lem:u_k+1_upper_bounded_by_u_k}
    Under the samewowo assumption as Theorem~\ref{thm:good_regularity_convergence}, we have
\begin{align*}
    u_{k+1}^\top u_{k+1}\lesssim u_k^\top u_k-4h\ell_k (x_k y_k-\ell_k)
\end{align*}
\end{lemma}
\begin{proof}
    By Lemma~\ref{lem:u^2_bounded_by_4/h+h}, we have $$u_k^\top u_k\le\frac{4}{h}+\mathcal{O}(h).$$ Then by~\eqref{eqn:u^top_u_one_step}, we have
    \begin{align*}
    u_{k+1}^\top u_{k+1}\le u_k^\top u_k-4h\ell_k (x_k y_k-\ell_k)+\mathcal{O}(h^3\ell_k^2)
\end{align*}
\end{proof}

For the update of $x_ky_k$, we have
\begin{align}
\label{eqn:xy_update}
    x_{k+1}y_{k+1}=(1+h^2\ell_k^2)x_k y_k-h\ell_k u_k^\top u_k,
\end{align}
\begin{align}
\label{eqn:xy-1_update}
    x_{k+1}y_{k+1}-1=(x_k y_k-1)\left(1-h\frac{\ell_k}{x_ky_k-1}(u_k^\top u_k-h\ell_k x_k y_k)\right).
\end{align}

Let $\delta:=xy-1$ and $\delta_k:=x_ky_k-1$. Then we define the following functions:
$$\ell(\delta)=F'(\delta+1),\quad\text{ and then }\ell_k=\ell(\delta_k)=F'(x_ky_k);$$
$$q(\delta)=\frac{\ell(\delta)}{\delta},\quad\text{ and then }q(\delta_k)=\frac{\ell_k}{x_ky_k-1};$$
$$r(u_k,\delta)=1-hq(\delta)(u_k^\top u_k-h\ell(\delta) (\delta+1)),\quad\text{ and then }r_k=r(u_k,\delta_k)=1-h\frac{\ell_k}{x_ky_k-1}(u_k^\top u_k-h\ell_k x_k y_k).$$ 

Let $$C_k=\frac{1-r_k}{q(\delta_k)},\quad\text{ and then }r_k=1-C_kq(\delta_k).$$

From Lemma~\ref{lem:u_k+1_upper_bounded_by_u_k}, we also define
$$L(\delta)=\ell(\delta)(\delta+1-\ell(\delta)),\text{ and then }L(\delta_k)=\ell_k (x_k y_k-\ell_k).$$
All the above functions also depend on $a$ or $b$. If not specified, all the properties for these functions used in the proofs are valid for all $0<a\le1$ or $b=2n+1$ with $n\in\mathbb{N}$.


\section{Proofs of results for functions~\eqref{eqn:functions} with $0<a\le 1$}
\label{app:proof_good_regularity}

Consider function $$f(x,y)=\frac{(\log (\exp (x y-1)+1)+\log (\exp (1-x y)+1))^a}{a 2^{a-2} \log ^{a-1}(2)}\text{ with }0<a\le 1.$$

The following propositions show the properties of the functions defined above under this objective. The proof is based on simple analysis and Taylor expansion and is thus omitted. If not specified, these properties are independent of $a$.
\begin{proposition}
\label{prop:ell}
The function $\ell(\delta)=\frac{2^{2-a} \left(e^{\delta }-1\right) \log ^{1-a}(2) \left(\log \left(e^{-\delta }+1\right)+\log \left(e^{\delta }+1\right)\right)^{a-1}}{\left(e^{\delta }+1\right) }$ has the following properties:
\begin{itemize}
    \item $\ell(\delta)=-\ell(-\delta)$
    \item $\ell(\delta)>0$ for $\delta>0$.
    \item $|\ell(\delta)|\le 3$ for all $\delta>0$.
\end{itemize}

\end{proposition}

\begin{proposition}
\label{prop:L(delta)}
    The function $L(\delta)=\ell(\delta)(\delta+1-\ell(\delta))$ has the following properties:
    \begin{itemize}
        \item $L(\delta)$ monotonically increases for $\delta\ge 0$, and $L(\delta)\ge L(0)$ for $\delta\ge 0$.
        \item $L(\delta)+L(-\delta)\ge 0$ for all $\delta$.
        \item $L(\delta_1)+L(\delta)\ge 0$ for $\delta_1\ge 1$ and all $\delta$.
        \item $L(\delta)+L(r\delta)\ge 0.8(1+r)\delta$ for $\delta>0$ and $-1<r<0$.
    \end{itemize}
\end{proposition}

\begin{proposition}
\label{prop:q}
The function $q(\delta)=\frac{\ell(\delta)}{\delta}$ has the following properties:
\begin{itemize}
    \item $q(\delta)$ is symmetric with respect to $\delta=0$, i.e., $q(\delta)=q(-\delta)$. 
    \item For $\delta\ge0$, $q(\delta)$ monotonically decreases as $\delta$ increases, and then $q(\delta)\le q(0)=1$. 
    \item  From Taylor expansion, we have $$q(\delta)\ge 1-c_1 \delta^2,\text{ where }c_1=-\frac{-3+3 a-2 \log(2)}{24 \log(2)}>0.$$ Also, when $\delta\le 1$, 
    $$q(\delta)\le 1-c_2 \delta^2,\text{ where }c_2=-\frac{-3+3 a-2 \log(2)}{30 \log(2)}.$$  
    \item For $2\le C\le 4$, $$\delta=q^{-1}(\frac{1}{C})\le (1+a)C.$$
\end{itemize}
   
\end{proposition}
From the above propositions, we will use $\delta_k=|x_ky_k-1|$ instead of $x_ky_k-1$ for the rest of the proofs.

\subsection{Proof of convergence}
\begin{proof}[Proof of Theorem~\ref{thm:good_regularity_convergence}]

By Proposition~\ref{prop:q} and Proposition~\ref{prop:ell},
$$r_k=1-hq(\delta_k)(u_k^\top u_k-h\ell_k x_k y_k)<1.$$ As is discussed in Lemma~\ref{lem:does_not_converge_outside_2/h}, the initial condition set removes a null set of converging initial conditions within finite steps, i.e., $r_k\ne 0$. By our choice of initial condition and learning rate $h$,
\begin{align*}
    r_0&\le 1-\big(\frac{2/q(1)}{u_0^\top u_0}+\mathcal{O}(h^2) \big)q\left(q^{-1}\big(\frac{q(1)}{2}\big)(u_0^\top u_0-\mathcal{O}(h))\right)\\&=1-(1-\mathcal{O}(h^2)+\mathcal{O}(h)-\mathcal{O}(h^3))=-\mathcal{O}(h)<0
\end{align*}
By Lemma~\ref{lem:r_k_negative_for_all}, for all $n\ge k$, if $x_ny_n>1$, then $x_{n+1}y_{n+1}<1$ and vice versa. Then we can consider the $k$th iteration when $x_ky_k>1$ and we have $x_{k+2n}>1$ for $n=1,\cdots,$.

By Lemma~\ref{lem:xy-1_decrease_final_phase}, $|x_ky_k-1|$ is guaranteed to decrease monotonically in $|x_ky_k-1|\le R_2(a)$, with $|r_k|<1$. Therefore, GD will converge to $xy=1$ (otherwise, if $x_ky_k$ converges to $c\ne 1$, with $|r_k|<1$, $|x_ky_k-1|$ will keep decreasing, contradiction).
\end{proof}

\subsection{Proofs of EoS}
\begin{proof}[Proof of Theorem~\ref{thm:eos_good_regularity} Part II: limiting sharpness]

By Theorem~\ref{thm:good_regularity_convergence}, GD converges to a global minimum. Throughout this proof, we will use the big $\mathcal{O}$ notation for complexity and distance. The proof can be made more rigorous by considering some specific constant scaling of these orders.

By the lower bound of $h$, $u_0^\top u_0\gtrsim \frac{2/q(1)+\mathcal{O}(h)}{h}$, which implies $|r_0(1)|\ge 1+\mathcal{O}(h)$. Then we consider the decrease of $u_k^\top u_k$ until GD enters the region $|r_k|<1$. More precisely, we consider three regions: 1) when $|\delta_k|$ is small enough s.t. $u_k^\top u_k$ does not decrease every two steps. By the following bound, 
\begin{align*}
    u_{k+2}^\top u_{k+2}&=(1+h^2\ell_{k+1}^2)u_{k+1}^\top u_{k+1}-4h\ell_{k+1} x_{k+1} y_{k+1}\\
    &=(1+h^2\ell_{k+1}^2)((1+h^2\ell_{k}^2)u_{k}^\top u_{k}-4h\ell_{k} x_{k} y_{k})-4h\ell_{k+1} x_{k+1} y_{k+1}\\
    &\ge u_k^\top u_k
-4h[\ell_k x_k y_k-\frac{1}{2}\ell_{k}^2+\ell_{k+1} x_{k+1}y_{k+1}-\frac{1}{2}\ell_{k+1}^2]-4h^3\ell_{k+1}^2\ell_k x_k y_k+h^4\ell_{k+1}^2\ell_k^2u_k^\top u_k\\
&\ge u_k^\top u_k-4h[(1+r_k)(x_ky_k-1)+\frac{1}{2}(1+r_k^2)(x_ky_k-1)^2]-4h^3\ell_{k+1}^2\ell_k x_k y_k+h^4\ell_{k+1}^2\ell_k^2u_k^\top u_k,
\end{align*}
we have that such $|\delta_k|\le\mathcal{O}(h)$. 2) Starting from $|\delta_k|=\mathcal{O}(h)$, consider $|\delta_k|$ increases to some region that is $\mathcal{O}(h)$ away from 1, i.e., the rate of increase is at least $1+\mathcal{O}(h)$. In this region, it takes GD at most $\mathcal{O}(-\log h/h)$ steps and $u_k^\top u_k$ may decrease every two steps, where the order of decrease is $\mathcal{O}(h)$. 3) Starting from the end of 2), if the rate of increase is at least $1+\mathcal{O}(h)$, then the complexity of entering the region $|r_k|<1$ is at most $\mathcal{O}(1/h)$ which follows the same derivation as 2). Otherwise, i.e., the rate of increase is less than $1+\mathcal{O}(h)$. Since the decrease of $u_k^\top u_k$ is $\mathcal{O}(h)$ and $q(\delta_k)$ keep decreasing before $|r_k|<1$, it takes GD at most $\mathcal{O}(1/h)$ steps such that $|r_k|$ is $\mathcal{O}(h)$ less than $1+\mathcal{O}(h)$, i.e., $|r_k|<1$.
Therefore, the overall decrease of $u_k^\top u_k$ is at most $\tilde{\mathcal{O}}(1)$ and we still have $u_k^\top u_k\ge \frac{2/q(1)-\tilde{\mathcal{O}(h)}}{h}$ at the end of all the above processes. Also, according to Lemma~\ref{lem:u^2_decrease_every_two_step}, $u_k^\top u_k$ will decrease to at least $\frac{2/q(1)+\mathcal{O}(h)}{h}$. Otherwise, $|r_k|>1$ in $|x_ky_k-1|\ge R_1(a)$ and therefore $u_k^\top u_k$ will keep decreasing.

Before further discussing the complexity of GD, we still need an upper bound of $x_ky_k$. Let $u_k^\top u_k=\frac{C_k}{h}$. Then since $r_k<0$ for all $k$, we can consider the maximum $\delta$ s.t. $r_0\approx0$, i.e.,
$\delta=q^{-1}(\frac{1}{C_k})\le (1+a)C_k$ by Proposition~\ref{prop:q}. Therefore, $\delta_k\lesssim (1+a)C_k\le 5$. Here we just relax $C_k$ to be upper bounded by 2.5 based on ${2/q(1)+\mathcal{O}(h)}$.

If $|x_ky_k-1|>R_2(a)$, suppose  $|x_ky_k-1|>1.5$ then $|r_k|< r_k(1.5)<0.8+\mathcal{O}(h^2)<0.85$. Then the complexity of entering $\delta_k\le 1.5$ is $\frac{\log(1.5/5)}{\log(0.85)}=\mathcal{O}(1)$. Therefore, the step of GD entering this region is less than the order $\mathcal{O}(1/\log(r_k(1.5)))=\tilde{\mathcal{O}}(1)$. 

 Then consider the complexity of GD entering $\{r<1,|xy-1|<1-\mathcal{O}(h^2)\}$ (By Lemma~\ref{lem:guaranteed_decrease_xy-1_in_R2}, $R_2(a)\ge 1-\mathcal{O}(h^2)$). We can consider $|r_k|$ via the map $g(\delta)=\delta(C_k q(\delta)-1)$, where we ignore the $\mathcal{O}(h^2)$ terms and fix $C_k=C_N$ for some $N$. We would like to analyze the complexity of the map converging to $\mathcal{O}(h)$ error of the fixed point $\delta^*$ near $\delta=1$, i.e., $(C_k q(\delta)-1)=1$. Since $(C_k q(1-\mathcal{O}(h^2))-1)=1-\mathcal{O}(h)$, 
\begin{align*}
   1=C_k q(\delta^*)-1 &\ge C_k(1-c_1(\delta^*)^2)-1\\
   &=C_k(1-c_1)-1+(1-(\delta^*)^2)C_kc_1\\
   &=1-\mathcal{O}(h)+(1-(\delta^*)^2)C_kc_1
\end{align*}
Therefore $\delta^*\ge 1-\mathcal{O}(h/c_1)$.

By checking the derivative of the map, we have $g'(\delta)=C_k q(\delta)-1+C_k\delta q'(\delta)$ and $q'(\delta)<-c_1$ for $0.8<\delta<1.5$. Since $C_k>2$, when $\delta>\delta^*$, $C_k q(\delta)-1<1$ and thus $0<g'(\delta)\le 1-C_k\delta^*c_1.$ Then the map will decrease from above to enter the region within $\mathcal{O}(h)$ distance of $\delta^*$ with the complexity of $\mathcal{O}(-\log(h)/c_1)$. For the $\mathcal{O}(h^2)$ terms, the error between this map and the true $\delta_k$ update is of the order $\mathcal{O}(\exp(h^2\cdot \log(h)/c_1)-1)<\mathcal{O}(h)$ and thus can be omitted. Then GD starts to decrease in the monotone decreasing region of $|x_ky_k-1|$ with $u_k^\top u_k\ge \frac{2}{h}+\mathcal{O}(1)$.

Next we consider the $N$th iteration with $u_N^\top u_N=\frac{2}{h}+\bar{C}_N$, for some $\bar{C}_N=\mathcal{O}(1)$. For $k\ge N$, $u_k^\top u_k>\frac{2}{h}$, and $x_ky_k>1$, we have that $\bar{C}_k-\bar{C}_{k+2}=\mathcal{O}(h)$. We would like to prove that eventually $|r_k|\ge  1-\mathcal{O}(h^2)$ when $u_k^\top u_k\ge \frac{2}{h}$, i.e., we would like to analyze the complexity of $|r_k|$ increasing to $|r_k|=1-\mathcal{O}(h^2)$.

First note
$$|r_k|\ge q(\delta_k)h(u_k^\top u_k-h\ell_kx_ky_k)-1= 1-2c_1\delta_k^2+2\bar{C}_Nh\pm\mathcal{O}(h^2).$$ Then the complexity is upper bounded by the complexity of $\delta_k$ decreasing to the value s.t. $-2c_1\delta_k^2+2\bar{C}_Nh=0$, i.e., $\delta_k=\mathcal{O}(\sqrt{h/c_1})$. Let $c_1= h^p$. By the assumption of $h$, we have $p\le \frac{3}{4}$. Therefore $\sqrt{h/c_1}=h^\frac{1-p}{2}$. Then we analyze 
\begin{align*}
    \delta_{n+N}=\delta_N-2h^p\delta_N^3+2\bar{C}_Nh\delta_N-2h^p \delta_N^3(1-2h^p\delta_N^2+2\bar{C}_Nh)^2+\cdots=\mathcal{O}(h^\frac{1-p}{2})
\end{align*}
We further remove all the $\bar{C}_N$ terms since the sum of them is  $\mathcal{O}(h)<\mathcal{O}(h^\frac{1-p}{2})$. We first consider
\begin{align*}
    \delta_N-2h^p\delta_N^3-2h^p \delta_N^3(1-2h^p\delta_N^2)^2+\cdots=h^{p_1}
\end{align*}
Ignoring the $\mathcal{O}(h^2)$ part, we have that $|r_k|$ is increasing as $|\delta_k|$ decreases. Then
\begin{align*}
    h^{p_1}=LHS\lesssim \delta_N-h^{p}h^{3p_1} -h^{p}h^{3p_1}+\cdots=\delta_N-n_1 h^{3p_1+p}\le 1-n_1 h^{3p_1+p}
\end{align*}
Then
\begin{align*}
    n_1\le h^{-(3p_1+p)}
\end{align*}
Iteratively, we consider $\delta_N= h^{p_1}$ and $p_2=\frac{4}{3}p_1$. Then
\begin{align*}
    h^{p_2}=LHS\lesssim h^{p_1}-h^{p}h^{3p_2} -h^{p}h^{3p_2}+\cdots=\delta_N-n_2 h^{3p_2+p}\le h^{p_1}-n_2 h^{3p_2+p}
\end{align*}
\begin{align*}
    n_2\le h^{-(3p_2+p-p_1)}=h^{-(3p_1+p)}
\end{align*}
We can next solve $i$ s.t.
$$\left(\frac{4}{3}\right)^{i-1}p_1=\frac{1-p}{2}\Rightarrow i=1+\frac{\log\frac{1-p}{2p_1}}{\log\frac{4}{3}}$$
From the above, we also need to consider the $\bar{C}_N$ part and thus require
$$3p_1+p<\min\{\frac{1+3p}{2},1\}$$
Then $$\frac{1-p}{2p_1}>\frac{3}{2}\text{ and }i \text{ is at most }\mathcal{O}(\log h)$$

Then the total complexity of achieving $h^\frac{1-p}{2}$ is $\mathcal{O}(h^{-(3p_1+p)})<\mathcal{O}(h^{-1})$. Thus we will eventually have $|r_k|=  1-\mathcal{O}(h^2)$ when $u_k^\top u_k\ge \frac{2}{h}$.
For the $k+2$th iteration, the increase of $1-c\delta_k^2$ is  $\mathcal{O}((1-|r_k|)\delta_k^2+(1-|r_{k+1}|)\delta_k^2)=\mathcal{O}((1-|r_k|)\delta_k^2)$, where one step change of $|r_k|$ is at most $\mathcal{O}(h^2)$ and thus can be omitted, and the decrease of $h(u_k^\top u_k-h\ell_kx_ky_k)$ is $\mathcal{O}(h^2\delta_k(1-|r_k|))$ by Lemma~\ref{lem:u^2_decrease_every_two_step}. Thus when $\delta_k\ge \mathcal{O}(h^2)$, $\mathcal{O}((1-|r_k|)\delta_k^2)\ge \mathcal{O}(h^2\delta_k(1-|r_k|))$, i.e., $|r_k|\ge 1-\mathcal{O}(h^2)$; when $\delta_k<\mathcal{O}(h^2)$, from the expression of $|r_k|$, we have $|r_k|\ge 1-\mathcal{O}(h^2)$ for the $k$th iteration s.t. $u_k^\top u_k\ge \frac{2}{h}-\mathcal{O}(h)$. Therefore, $|r_k|\ge 1-\mathcal{O}(h^2)$ for all the $k$th iteration s.t. $u_k^\top u_k\ge \frac{2}{h}-\mathcal{O}(h)$.

Next, consider the first step when $\frac{2}{h}\ge u_N^\top u_N\ge\frac{2}{h}-\mathcal{O}(h)$
and $x_Ny_N>1$ since the decrease of $u_k^\top u_k$ at each step is at most $\mathcal{O}(h)$. Then based on the expression of $r_N$, we have $|r_N|\le 1-\mathcal{O}(h^2)$ and additionally by Lemma~\ref{lem:u^2_decrease_every_two_step}, $u_k^\top u_k$ decreases every two steps and thus we have $|r_k|=1-\mathcal{O}(h^2)$ for all $k\ge N$.
    By series expansion, $\ell_{N+k} x_{N+k}y_{N+k}= \mathcal{O}(|x_{N+k}y_{N+k}-1|) $. Then
  \begin{align*}
 {u_{N+k+2}}^\top u_{N+k+2}&\ge {u_{N+k}}^\top u_{N+k}-4h [\ell_{N+k} x_{N+k}y_{N+k}+\ell_{N+k+1} x_{N+k+1}y_{N+k+1}]-\mathcal{O}(h^3)\\
      &\ge {u_{N+k}}^\top u_{N+k}-\mathcal{O}\left(h [\ell_{N+k} x_{N+k}y_{N+k}+\ell_{N+k+1} x_{N+k+1}y_{N+k+1}]\right)\\
      &\ge {u_{N+k}}^\top u_{N+k}-\mathcal{O}\left(h [(1-|r_{N+k}|)\delta_{N+k}]\right)\\
      &\ge {u_{N}}^\top u_{N}-\mathcal{O}\left(h \sum_{i=0}^k[(1-|r_{N+i}|)\delta_{N+i}]\right)\\
      &\ge {u_{N}}^\top u_{N}-\mathcal{O}\left(h \sum_{i=0}^k[(1-|r_{N+i}|)\prod_{j=0}^i r_{N+j}\delta_N]\right)\\
      & \ge {u_{N}}^\top u_{N}-\mathcal{O}\left(h^3 \delta_N \sum_{i=0}^k(1-\mathcal{O}(h^2))^i\right)
    \end{align*}
    Take $k\to\infty$ for both side and we have $u_\infty^\top u_\infty\ge u_N^\top u_{N}-\mathcal{O}(h)$.
\end{proof}

\begin{proof}[Proof of Theorem~\ref{thm:eos_good_regularity} Part I: entering flat region]
    Consider the Hessian $\nabla^2 f$. The trace is 

    $$\Tr(\nabla^2f)=(x^2+y^2)G_1(xy-1),$$ where
    \begin{align*}
        G_1(\delta)&=\frac{2^{2-a} \log ^{1-a}(2) \left(\log \left(e^{-\delta }+1\right)+\log \left(e^{\delta }+1\right)\right)^{a-2} \left((a-1) \left(e^{\delta }-1\right)^2+2 e^{\delta } \log \left(e^{-\delta }+1\right)+2 e^{\delta } \log \left(e^{\delta }+1\right)\right)}{\left(e^{\delta }+1\right)^2}\\
        &=\frac{2^{3-a} e^{\delta } \log ^{1-a}(2) \left(\log \left(e^{-\delta }+1\right)+\log \left(e^{\delta }+1\right)\right)^{a-1}}{\left(e^{\delta }+1\right)^2}\\
        &+ \frac{2^{2-a} (a-1) \left(e^{\delta }-1\right)^2 \log ^{1-a}(2) \left(\log \left(e^{-\delta }+1\right)+\log \left(e^{\delta }+1\right)\right)^{a-2}}{\left(e^{\delta }+1\right)^2}\\
        &= q(\delta)\frac{2\delta}{e^\delta-e^{-\delta}}+ q(\delta)\frac{(a-1) \left(e^{\delta }-1\right) \delta }{\left(e^{\delta }+1\right) \left(\log \left(e^{-\delta }+1\right)+\log \left(e^{\delta }+1\right)\right)}\\
        &\le q(\delta)\frac{2\delta}{e^\delta-e^{-\delta}}+q(\delta)(a-1)\left(1-\frac{2\delta}{e^\delta-e^{-\delta}}\right)\\
        &= \left((2-a) \frac{2\delta}{e^\delta-e^{-\delta}}-(1-a)\right)q(\delta).
    \end{align*} 
    Let $$q_1(\delta,a)=(2-a) \frac{2\delta}{e^\delta-e^{-\delta}}-(1-a).$$ Then both $q_1(\delta,a)$ and $q(\delta)$ decreases as $\delta$ increases for $\delta\ge 0$. Also, fix $\delta$, $q_1(\delta,a)$ and $q(\delta,a)$ increases as $a$ increases.

    The determinant is
    \begin{align*}
        \det(\nabla^2f)=G_2(xy-1)
    \end{align*}
    where
    \begin{align*}
        G_2(\delta)&=-\frac{4^{2-a} \left(e^{\delta }-1\right) \log ^{2-2 a}(2) \left(\log \left(e^{-\delta }+1\right)+\log \left(e^{\delta }+1\right)\right)^{2 a-3}}{\left(e^{\delta }+1\right)^3} \\
        &\times\left(2 (a-1) (\delta +1) \left(e^{\delta }-1\right)^2+\left(4 e^{\delta } (\delta +1)+e^{2 \delta }-1\right) \left(\log \left(e^{-\delta }+1\right)+\log \left(e^{\delta }+1\right)\right)\right).
    \end{align*}
    Then
     \begin{align*}
         G_2(\delta)\sim \mathcal{O}(\delta^{2a-2})\text{ as }\delta\to\infty, \text{ and }G_2(0)=0.
     \end{align*}
      Thus we have that $G_2(\delta)$ is bounded since $G_2(\delta)\in\mathcal{C}^1$. Therefore, by Taylor expansion, the largest eigenvalue of Hessian (sharpness) is upper bounded by $$G_1(\delta)(x^2+y^2)+\mathcal{O}\left(\frac{|G_2(xy-1)|}{x^2+y^2}\right), $$ 
    where $|G_2(xy-1)|$ is bounded due to the boundedness of $x_0y_0$ in the initial condition set and the convergence of the GD trajectory (see more details in the proof of limiting sharpness).

    Let $h=\frac{C}{x_0^2+y_0^2}$. If $|r_0|<1$, then 
    \begin{align*}
        q(\delta_0)\lesssim \frac{2}{C}.
    \end{align*}
    Otherwise, similar to the proof of part II, $u_k^\top u_k$ decreases at most $\tilde{\mathcal{O}}(1)$ before GD enters the region where $\delta_k$ starts to decrease for the first time, i.e., $|r_k|<1$. 
    Therefore, for such $\delta=\delta_k$, we also have $$1-c_1\delta^2\le q(\delta)\lesssim \frac{2}{C},$$ where $\tilde{\mathcal{O}}(h)$ term is omitted. Then $$\delta\gtrsim \frac{\sqrt{C-2}}{\sqrt{Cc_1}},$$ and  $$G_1(\delta)\lesssim \frac{2}{C}q_1(\frac{\sqrt{C-2}}{\sqrt{Cc_1}},a)\le \frac{2}{C}q_1(\frac{\sqrt{C-2}}{\sqrt{C/12}},1),$$
    where the last inequality follows easily from checking the derivative of $q_1(\frac{\sqrt{C-2}}{\sqrt{Cc_1(a)}},a)$ with respect to $a$. Also, $S_\infty\approx \frac{2}{h}=\frac{2}{C}u_0^\top u_0$. Based on the above discussion, $u_k^\top u_k\le u_0^\top u_0-\mathcal{O}(h^{1-m})$. Therefore, $$S_k\lesssim G_1(\delta_k)u_k^\top u_k\lesssim q_1(\frac{\sqrt{C-2}}{\sqrt{C/12}},1)S_\infty\le \frac{1}{4}(6-C)S_\infty,$$
    where the last inequality follows from linearization and shift.
\end{proof}

\subsection{Proof of balancing}

\begin{proof}[Proof of Theorem~\ref{thm:balancing_good_regu}]
By Theorem~\ref{thm:eos_good_regularity} and the global minima $xy=1$,
    \begin{align*}
    (x_\infty-y_\infty)^2=u_\infty^\top u_\infty-2\le \frac{2}{h}-2=\frac{2}{C}(x_0^2+y_0^2-2x_0y_0)+\frac{4}{C}x_0y_0-2\le \frac{2}{C}(x_0-y_0)^2+2(x_0y_0-1).
\end{align*}
\end{proof}

\subsection{Supplementary lemmas}

We first show a summary of the behavior of $u_k^\top u_k$ in the following lemma, which is the main idea of the proof of convergence.
\begin{lemma}
\label{lem:general_decrease_for_u^2}
Under the assumption of Theorem~\ref{thm:good_regularity_convergence}, there exist an increasing sequence $\{2N_i\}_{i\in\mathbb{Z}}$ with $N_i\in\mathbb{Z}$ and $N_{i+1}>N_i$, s.t., ${u_{2N_{i+1}}}^\top u_{2N_{i+1}} \le {u_{2N_{i}}}^\top u_{2N_i} $. More precisely, consider the even number of iterations, i.e., $k=2i$ for $i\in\mathbb{Z}$, 
    \begin{itemize}
        \item Stage I (not necessary): $u_k^\top u_k$ increases but is bounded by $\frac{4}{h}+\mathcal{O}(h)$. This happens when $|x_ky_k-1|$ is too small at the early stage of iterations (see Lemma~\ref{lem:u^2_bounded_by_4/h+h});
        \item Stage II: $u_k^\top u_k$ decreases every two steps when $x_k,y_k$ is outside a `monotone decreasing region' (see Lemma~\ref{lem:u^2_decrease_every_two_step});
        \item Stage III: when $x_k,y_k$ is near the `monotone decreasing region', there exists $N$, s.t. $u_{k+2N}^\top u_{k+2N}$ is smaller than $u_k^\top u_k$ (see Lemma~\ref{lem:xy-1_decrease_final_phase});
        \item Stage IV: when $|x_ky_k-1|$ monotonically decreases, $u_k^\top u_k$ decreases every two steps (see Lemma~\ref{lem:u^2_decrease_every_two_step} and Lemma~\ref{lem:xy-1_decrease_final_phase})
    \end{itemize}
    
\end{lemma}

\begin{lemma}
\label{lem:u^2_bounded_by_4/h+h}
Under the assumption of Theorem~\ref{thm:good_regularity_convergence}, we have
    $$u_k^\top u_k\le \frac{4}{h}+\mathcal{O}(h),\forall k\in\mathbb{N}.$$
\end{lemma}
\begin{proof}

    By Lemma~\ref{lem:r_k_negative_for_all} and Lemma~\ref{lem:u^2_decrease_every_two_step}, we know that the increase of $u_k^\top u_k$ after two step can only happen when $r_k\le -1$ and $0<x_ky_k-1<1$.
    
    WOLG assume $N$ is the first step s.t. $u_N^\top u_N= 4/h+\mathcal{O}(h)$. By Lemma~\ref{lem:u^2_decrease_every_two_step}, we would like to show that there exists $k$ s.t. for $x_ky_k>1$, we have either $x_ky_k-1>1$ or $-1<r_k<0$.

    Let $\delta_k=|x_ky_k-1|$. Then $$|r_k|\ge 4(1-c_1\delta_k^2)-1+\mathcal{O}(h^2)\ge 4(1-\delta_k^2/4)-1+\mathcal{O}(h^2) $$ and therefore $$\delta_{k+1}\ge\delta_k(3-4c_1\delta_k^2+\mathcal{O}(h^2))\ge\delta_k(3-\delta_k^2+\mathcal{O}(h^2)).$$ If $\delta_k<1$, it takes GD $\mathcal{O}(-\log(\delta_0))$ steps to go to $\delta_k \ge 1$. When $\delta_k\ge 1$ $u_k^\top u_k$ already starts to decrease every two steps. Moreover, $$\delta_{k+1}\ge\delta_k(3-\delta_k^2+\mathcal{O}(h^2))=\delta_k+2\delta_k-\delta_k^3+\mathcal{O}(h^2).$$ Therefore, $\delta_k$ keeps increasing for a while, staying in $\delta_k\ge 1$. Therefore, $u_k^\top u_k$ stays in the decreasing region with increase of at most $\mathcal{O}(h)$ for each step (for more detailed version, see Lemma~\ref{lem:general_decrease_for_u^2}), i.e., $u_k^\top u_k\le \frac{4}{h}+\mathcal{O}(h)$ for all $k$. 
\end{proof}

\begin{lemma}
\label{lem:does_not_converge_outside_2/h}
    Under the assumption of Theorem~\ref{thm:good_regularity_convergence}, GD does not converge outside $\{(x,y)|x^2+y^2\le 2/h\}$.
\end{lemma}
\begin{proof}
     We first would like to remove the initial condition that can converge in finite steps to the minima outside of this region. It turns out that such initial conditions form a null set. The proof is almost the same as~\citet{wang2022large} except for some easy calculations and thus omitted. We further remove all the initial conditions that converges to the periodic orbits, i.e.,  $\prod_{i=1}^n r_k\cdots r_{k+2i-1}=1$ for all $k$ and $n$. By similar argument in \citet{wang2022large} (i.e., for each $k,n$, this is a null set, and therefore the union of countably many null sets is still a null set), such initial conditions also form a null set.

     Assume $u_k^\top u_k\ge 2/h+\epsilon_0$ for all $k$, where $\epsilon_0>0$. Consider the case when $|x_ky_k-1|<\epsilon_1=\frac{\sqrt{h\epsilon_0}}{2}$. Then
     \begin{align*}
    |r_k|&\ge|h(1-\mathcal{O}(\epsilon_1^2))(2/h+\epsilon_0)+\mathcal{O}(h^2\epsilon_1)-1|\\
    &=|1-2\epsilon_1^2+h\epsilon_0+\mathcal{O}(h^2\epsilon_1)|\\
    &=|1+\frac{h}{2}\epsilon_0|>0
   \end{align*}
Namely, GD cannot converge with $u_k^\top u_k>2/h$.  
\end{proof}

\begin{lemma}
\label{lem:r_k_negative_for_all}
    Under the assumption of Theorem~\ref{thm:good_regularity_convergence}, if $r_k<0$, then $r_{n}<0$ for all $n\ge k$.
\end{lemma}
\begin{proof}
From Lemma~\ref{lem:u_k+1_upper_bounded_by_u_k}, we have $|C_{k-1}-C_k|=\mathcal{O}(h^2)$. Also $C_k=\frac{1-r_k}{q(\delta_k)}$. If $r_k<0$,
    \begin{align*}
        r_{k+1}=1-C_{k+1}q(\delta_{k+1})&\le 1-\frac{1-r_k}{q(\delta_k)}q(\delta_{k+1})+\mathcal{O}(h^2)\\
        &= 1-\frac{1-r_k}{q(\delta_k)}q(r_k\delta_{k})+\mathcal{O}(h^2)
    \end{align*}
    By Proposition~\ref{prop:q}, it suffices to prove that 
    \begin{align*}   R_0(\delta_k,-r_k):=(1-r_k)q(-r_k\delta_k)-q(\delta_k)>0,\ \forall \delta_k>0
    \end{align*}
We abuse the notation and let $r_k(\delta)=1-C_k q(\delta)$ for fixed $C_k>0$. By Proposition~\ref{prop:q}, $r_k(\delta)$ monotonically increases for $\delta>0$ and there is only one root denoted as $\delta_0$ s.t. $r_k(\delta)=0$ (Since $r_k<0$, there exists $\delta$ s.t. $r_k(\delta)<0$ and when $\delta$ is large, $r_k(\delta)>0$; therefore, it has a root). When $r_k(\delta_0)=r_k=0$, by Proposition~\ref{prop:q},  $R_0(\delta_0,-r_k(\delta_0))=R_0(\delta_0,0)=1-q(\delta_0)>0$. Also
\begin{align*}
    R_0(\delta,-r_k(\delta))&=C_k q(\delta) q\left( (C_kq(\delta)-1 )\delta\right)-q(\delta)\\
    &=q(\delta)\left[ C_k  q\left( (C_kq(\delta)-1 )\delta\right)-1 \right]
\end{align*}
Therefore, we only need to show $K(\delta):=(C_kq(\delta)-1 )\delta<\delta_0$ for $\delta\in (0,\delta_0)$. The derivative of $K$ is
$$K'(\delta)=C_kq(\delta)-1+C_k\delta q'(\delta)<C_kq(\delta)-1. $$
Then the maximum point of $K(\delta)$ is achieved at $\delta_1<\delta_0$ by Proposition~\ref{prop:q}. Thus, $R_0(\delta,-r_k(\delta))>0$, and consequently $$R_0(\delta_k,-r_k)\ge q(\delta_k)(C_k q(K(\delta_1))-1) >0,$$ $$r_{k+1}\le -\frac{R_0(\delta_k,-r_k)}{q(\delta_k)}+\mathcal{O}(h^2)\le-(C_k q(R_1(\delta_1))-1)+\mathcal{O}(h^2)<0,$$
since this $\mathcal{O}(h^2)$ is indeed bounded by $c\, \delta_k h^2$ for some universal constant $c>0$ and can be controlled by the first term. 
\end{proof}

\begin{lemma}
\label{lem:u^2_decrease_every_two_step}
Under the assumption of Theorem~\ref{thm:good_regularity_convergence}, the following properties are the two step decrease of $u_k^\top u_k$ in different cases:
\begin{enumerate}
    \item If $r_k<0$ and $u_k^\top u_k\le \frac{4}{h}+\mathcal{O}(h)$, there exists $R_1(a,r_k)\le 1$, s.t., when  $x_ky_k-1>R_1(a,r_k)$, we have $$u_{k+2}^\top u_{k+2}\le u_k^\top u_k-\frac{1}{2}h+\mathcal{O}(h^3\ell_{k+1}^2).$$ 
    \item If $-1<r_k<0$ and $0<x_ky_k-1\le1$, then $$u_{k+2}^\top u_{k+2}\le u_k^\top u_k-3.2h(1+r_k)(x_ky_k-1).$$
    \item If $r_k\le-1$, $u_k^\top u_k\le \frac{3}{h}$, and $x_ky_k<1$, then $$ u_{k+2}^\top u_{k+2}
\le u_k^\top u_k-\min\{\mathcal{O}(h),\mathcal{O}(h(x_ky_k-1))\}.$$
\end{enumerate}
\end{lemma}
\begin{proof}
First, we consider the two-step update of $u_k^\top u_k$ in the following
\begin{align*}
    u_{k+2}^\top u_{k+2}&=(1+h^2\ell_{k+1}^2)u_{k+1}^\top u_{k+1}-4h\ell_{k+1} x_{k+1} y_{k+1}\\
    &=(1+h^2\ell_{k+1}^2)((1+h^2\ell_{k}^2)u_{k}^\top u_{k}-4h\ell_{k} x_{k} y_{k})-4h\ell_{k+1} x_{k+1} y_{k+1}\\
    &\le u_k^\top u_k
-4h[\ell_k x_k y_k-\ell_{k}^2+\ell_{k+1} x_{k+1}y_{k+1}-\ell_{k+1}^2]-4h^3\ell_{k+1}^2\ell_k x_k y_k+h^4\ell_{k+1}^2\ell_k^2u_k^\top u_k+\mathcal{O}(h^3\ell_{k+1}^2)\\
&=u_k^\top u_k
-4h[\ell_k x_k y_k-\ell_{k}^2+\ell_{k+1} x_{k+1}y_{k+1}-\ell_{k+1}^2]+\mathcal{O}(h^3\ell_{k+1}^2).
\end{align*}
It suffices to analyze $$\ell_k x_k y_k-\ell_{k}^2+\ell_{k+1} x_{k+1}y_{k+1}-\ell_{k+1}^2=L(\delta_k)+L(r_k\delta_{k}),\ \text{where }L(\delta)=\ell(\delta)(\delta+1-\ell(\delta)).$$ 
By checking the slope and values of the above function, we have $$L(\delta_k)+L(r_k\delta_{k})\ge L(1)+L(r_k)>0.14$$ for all $r_k<0$, $0<a\le 1$, and $\delta_k\ge 1$.

 When $0<x_ky_k-1\le 1$, in order to make $-1<r_k<0$, we should at least have $u_k^\top u_k\le \frac{4}{h}$. Also we have $\ell_k\le x_ky_k-1\le 1$ in this region by Proposition~\ref{prop:q}. Therefore, by Proposition~\ref{prop:L(delta)}
 \begin{align*}
    u_{k+2}^\top u_{k+2}&=(1+h^2\ell_{k+1}^2)u_{k+1}^\top u_{k+1}-4h\ell_{k+1} x_{k+1} y_{k+1}\\
    &=(1+h^2\ell_{k+1}^2)((1+h^2\ell_{k}^2)u_{k}^\top u_{k}-4h\ell_{k} x_{k} y_{k})-4h\ell_{k+1} x_{k+1} y_{k+1}\\
    &\le u_k^\top u_k
-4h[\ell_k x_k y_k-\ell_{k}^2+\ell_{k+1} x_{k+1}y_{k+1}-\ell_{k+1}^2]-4h^3\ell_{k+1}^2\ell_k x_k y_k+h^4\ell_{k+1}^2\ell_k^2u_k^\top u_k\\
&\le u_k^\top u_k
-4h[\ell_k x_k y_k-\ell_{k}^2+\ell_{k+1} x_{k+1}y_{k+1}-\ell_{k+1}^2]\\
&\le u_k^\top u_k-3.2h(1+r_k)(x_ky_k-1).
\end{align*}

When $r_k\le-1$, $u_k^\top u_k\le \frac{3}{h}$, and $x_ky_k<1$, by Taylor series and simple calculation, we have $L(\delta)+L(r\delta)\ge \min\{ 0.5(1+r)\delta, L(-1)+L(1) \}$, where $L(-1)+L(1)\ge 0.14$. Then
\begin{align*}
    u_{k+2}^\top u_{k+2}
    &\le u_k^\top u_k
-4h[\ell_k x_k y_k-\ell_{k}^2+\ell_{k+1} x_{k+1}y_{k+1}-\ell_{k+1}^2]-h(\ell_k^2+\ell_{k+1}^2)+\mathcal{O}(h^3\ell_{k+1}^2)\\
&\le u_k^\top u_k-4h\min\{ 0.5(1+r_k)(x_ky_k-1), L(-1)+L(1) \}.
\end{align*}

\end{proof}

\begin{lemma}
\label{lem:guaranteed_decrease_xy-1_in_R2}
    Under the assumption of Theorem~\ref{thm:good_regularity_convergence}, if $-1\le r_k< 0$, there exists $R_2(a,r_k)\ge1-\mathcal{O}(h^2)$, s.t., for $|x_ky_k-1|\le R_2(a,r_k)$, we have $r_{k+1}>-1$.
\end{lemma}
\begin{proof}
    Since $r_k=1-C_kq(\delta_k)$, we have $$r_{k+1}=1-C_{k+1}q(\delta_{k+1})= 1-\frac{1-r_k}{q(\delta_k)}q(r_k\delta_{k})\pm\mathcal{O}(h^2).$$ 
    We denote $$\delta r(r_k,\delta_k):=2-\frac{1-r_k}{q(\delta_k)}q(r_k\delta_{k}).$$ When $\delta_k>0$, $
    \delta r(r_k,\delta_k)$ monotonically decreases as $\delta_k$ increases. Consider $$\delta r(r_k,1)=2-\frac{1-r_k}{q(1)}q(r_k).$$ This function $\delta r(r_k,1)$ monotonically decreases when $r_k$ increases between -1 and 0 and $\delta r(-1,0)>0$. Therefore, $r_{k+1}>-1\pm \mathcal{O}(h^2)$ and the conclusion follows from series expansion. 
\end{proof}

\begin{lemma}
\label{lem:R2>r1}
    $R_2(a)-R_1(a)>c>0$ for some constant $c$.
\end{lemma}
\begin{proof}
    Instead of analyzing the expressions of $R_1$ and $R_2$, we will just check this property by dividing the interval of $a$, i.e., $(0,1]$, into two subintervals and analyze a relaxation of $R_1$ and $R_2$. Consider $a\in[0.75,1]$ and $a\in (0,0.75)$. For the former one, $a\in[0.75,1]$, $R_1\le 1$ and $R_2\ge 1.5$; for the latter one, $a\in (0,0.75)$, $R_1\le 0.85$ and $R_2\ge1$. The detailed analysis is by checking the trend of the following two quantities in the intervals bounded by the above-mentioned values, the decrease of $u_k^\top u_k$ every two steps and the value of $r_{k+1}$ once $-1<r_k<0$ in previous lemmas.
\end{proof}

\begin{lemma}
\label{lem:xy-1_decrease_final_phase}
Under the assumption of Theorem~\ref{thm:good_regularity_convergence}, if $r_k<0$ for all $k\ge n$ given some $n\ge 0$, $|x_ky_k-1|$ will eventually start to decrease in the region $|x_ky_k-1|<R_2(a)$.

\end{lemma}
\begin{proof}

By Lemma~\ref{lem:u^2_decrease_every_two_step}, when $|x_ky_k-1|\ge R_1(a,r_k)$, $u_k^\top u_k$ keeps decreasing every two step with certain amount away from 0. Eventually, by the expression of $|r_k|$, $u_k^\top u_k$ will decrease until GD enters $|x_ky_k-1|\le R_2(a,r_k)$. Note by Lemma~\ref{lem:R2>r1}, $R_2(a,r_k)>R_1(a,r_k)$. However, $|x_ky_k-1|$ may not keep decreasing inside this region. WOLG, consider a lower bound of $R_2(a,r_k)$ to be $R_2(a)\ge 1-\mathcal{O}(h^2)$ and we will use $R_2(a)$ independent of iterations. From the expression of $|r_k|$, we have at least $u_k^\top u_k\le \frac{3}{h}$ in this case.

First consider $|x_ky_k-1|>R_2(a)$. According to Lemma~\ref{lem:u^2_decrease_every_two_step} and Proposition~\ref{prop:q}, $u_k^\top u_k$ decreases every two steps and the function $q$ decreases when $|x_ky_k-1|$ increases. Therefore there exists $k$ s.t. $|r_k|<1$. Moreover, for $|r_k|<1$, there exists $n$, s.t. $x_{k+2n}y_{k+2n}<x_k y_k$ and $|r_{k+2n}|<1$ if $x_{k+2n}y_{k+2n}-1\ge R_1(a)$.

Next, assume $k$ is such that $|x_ky_k-1|>R_2(a)$ but $|x_{k+2}y_{k+2}-1|\le R_2(a)$. According to the initial condition, there is no periodic orbit in the trajectory (see details in the proof of Lemma~\ref{lem:does_not_converge_outside_2/h}). We claim that there exists $n$, s.t., $|r_{k+2n}|<1$ and $|x_{k+2n}y_{k+2n}-1|\le R_2(a)$. Otherwise, if $u_{k+2}^\top u_{k+2}$ will still decrease every two steps, then it returns to the above cases when $|x_{k}y_{k}-1|> R_2(a)$. Then we analyze the case where $u_{k+2}^\top u_{k+2}$ will increase after two steps.
For the first $n$ s.t. $x_{k+2n}y_{k+2n}>R_2(a)$, if we have $|r_{k+2n}(R_2(a))|<|r_{k}(R_2(a))|$, this implies $u_{k+2n}^\top u_{k+2n}<u_k^\top u_k$ which can be either absorbed in the above cases, or eventually fall into the following case. For the first $n$ s.t. $x_{k+2n}y_{k+2n}>R_2(a)$, consider $|r_{k+2n}(R_2(a))|>|r_{k}(R_2(a))|$, with $|r_{k+2i}|\ge 1$ for all $i<n$ and $u_{k+2n}^\top u_{k+2n}>u_k^\top u_k$. Then if such process repeats, $|r_{k+2n}(R_2(a))|$ will be larger and larger until GD either starts to decrease in $|xy-1|\le R_2(a)$ or enters the two-step decreasing region of $u_k^\top u_k$ in $|xy-1|\le R_2(a)$ (we can use $|xy-1|\ge R_1(a)$ to represent this region; however, $R_1(a)$ is just a bound, meaning the actual region is larger), which will lead to $|r_{k+2n}|<1$ inside this region due to the same reasoning as the previous paragraph. Detailed characterization of this stage can be seen in the proof of limiting sharpness.

\end{proof}

\section{Proofs of results for functions~\eqref{eqn:functions} with $b=1$}

\begin{proof}[Proof of Theorem~\ref{thm:balancing_b=1}]
By Theorem 3.1 in \citet{wang2022large},
$$(x-y)^2\le \frac{2}{h}-2.$$
Then since $h=\frac{C}{x_0^2+y_0^2+4}$, we have
    $$(x-y)^2\le \frac{2}{h}-2=\frac{2}{C}(x_0^2+y_0^2-2x_0y_0+2x_0y_0+4)-2=\frac{2}{C}(x_0-y_0)^2+\frac{4}{C}(x_0y_0+2)-2.$$
\end{proof}

\section{Proofs of results for functions~\eqref{eqn:functions} with $b\ge 3$}
\label{app:proof_bad_regularity}

Consider function $$f(x,y)=(1-(xy)^b)^2/(2b^2),\text{ with }b=2n+1\text{ for }n\in\mathbb{Z}.$$ Let $p(s)=\sum_{i=0}^{b-1}s^i$. Then $$1-(xy)^b=(1-xy)p(xy).$$ Apart from the equations in Appendix~\ref{app:preparation_for_proofs}, we will also use the following two equations in our proofs
\begin{align*}
    x_{k+1}y_{k+1}-1=(x_ky_k-1)\left(1-\frac{h}{b}(x_ky_k)^{b-1}p(x_ky_k)u_k^\top u_k+\frac{h^2}{b^2}(x_ky_k-1)(x_ky_k)^{2b-2}p(x_ky_k)^2x_ky_k \right),
\end{align*}
\begin{align*}
    u_{k+1}^\top u_{k+1}&=u_k^\top u_k-\frac{h}{b}(x_ky_k-1)\\
    &\times\left( (x_ky_k)^{b-1}p(x_ky_k)x_ky_k\left(4-\frac{h}{b}(x_ky_k)^{b-1}p(x_ky_k)u_k^\top u_k\right)+\frac{h}{b}(x_ky_k)^{2b-2}p(x_ky_k)^2 u_k^\top u_k\right).
\end{align*}

\subsection{Proof of convergence}

\begin{proof}[Proof of Theorem~\ref{thm:bad_regularity_convergence}]

By Lemma~\ref{lem:does_not_converge_outside_2/h} (with a different choice of null set based on the functions), GD can only converge to the point s.t. $x^2+y^2\le 2/h$. Also we remove all the points (which is in a null set) s.t. they converge in finite step. Therefore, $r_k\ne 0$ for all $k\ge 0$.

If $0<r_0<1$, we have $x_{1}y_1>1$ and for any $x_ky_k>1$,
\begin{align*}
    u_{k+1}^\top u_{k+1}&=u_k^\top u_k-\frac{h}{b}(x_ky_k-1)\bigg( (x_ky_k)^{b-1}p(x_ky_k)x_ky_k\left(4-\frac{h}{b}(x_ky_k)^{b-1}p(x_ky_k)u_k^\top u_k\right)\\&\qquad+\frac{h}{b}(x_ky_k)^{2b-2}p(x_ky_k)^2 u_k^\top u_k\bigg)\\
    &\le u_k^\top u_k.
\end{align*}
Assume $r_i>0$ for $i=0,\cdots,k$. 
Consider
\begin{align*}
    r_{k+1}=\left(1-h\frac{\ell_{k+1}}{x_{k+1}y_{k+1}-1}(u_{k+1}^\top u_{k+1}-h\ell_{k+1} x_{k+1} y_{k+1})\right)
\end{align*}
where $\ell_k=\frac{(x_k y_k)^{b-1} \left((x_k y_k)^b-1\right)}{b}$ and therefore $q_k=\frac{\ell_k}{x_ky_k-1}=\frac{1}{b}(x_ky_k)^{b-1}p(x_ky_k)$. Moreover, $W(s)=\frac{1}{b}s^{b-1}p(s)$ monotonically increases when $s>1$, which implies $q_{k+1}\le q_k$. Also,
\begin{align*}
    0<u_{k+1}^\top u_{k+1}-h\ell_{k+1}x_{k+1}y_{k+1}&=(1+h^2\ell_{k}^2)u_{k}^\top u_{k}-4h\ell_{k} x_{k} y_{k}-h\ell_{k+1}x_{k+1}y_{k+1}\\
    &\le u_{k}^\top u_{k}-h\ell_{k} x_{k} y_{k}+h\frac{C}{b}\ell_k x_ky_k-3h\ell_{k} x_{k} y_{k}-h\ell_{k+1}x_{k+1}y_{k+1}\\
    &\le u_{k}^\top u_{k}-h\ell_{k} x_{k} y_{k}.
\end{align*}
Therefore, $r_{k+1}\ge r_k>0$. Moreover, we have
\begin{align*}
    r_{k+1}&=1-h\frac{\ell_{k+1}}{x_{k+1}y_{k+1}-1}(u_{k+1}^\top u_{k+1}-h\ell_{k+1} x_{k+1} y_{k+1})\\
    & \le 1-h\frac{\ell_{k+1}}{x_{k+1}y_{k+1}-1}(2x_{k+1} y_{k+1}-h\ell_{k+1} x_{k+1} y_{k+1})\\
    &\le 1-h(2-h\ell_{0})<1
\end{align*}
where the second inequality follows from the value at $x_{k+1}y_{k+1}=1$. Then $x_ky_k-1$ exponentially decreases until it converges.

If $r_0<0$, from the upper bound of $h$, we have $0<r_1<1$, $0<x_1y_1<1$, and $u_1^\top u_1\le u_0^\top u_0$ (according to the above discussion for $x_0y_0>1$). Then if $0<x_ky_k<1$,
\begin{align*}
    u_{k+1}^\top u_{k+1}&=(1+h^2\ell_{k}^2)u_{k}^\top u_{k}-4h\ell_{k} x_{k} y_{k}\\
    &=u_k^\top u_k +h|\ell_k|(4x_ky_k+h|\ell_k|u_k^\top u_k)\\
    &\le u_k^\top u_k +(4+Cu_k^\top u_k/(b u_0^\top u_0))h|\ell_k|x_ky_k
\end{align*}
where the inequality follows from $x_ky_k<1$ and $x_0y_0>1$.

Also, when $0<x_ky_k<1$,
\begin{align*}
    x_{k+1}y_{k+1}&=x_ky_k+(1-r_k)(1-x_ky_k)\\
    &=x_ky_k-h{\ell_{k}}(u_{k}^\top u_{k}-h\ell_{k} x_{k} y_{k})\\
    &\ge x_ky_k-h{\ell_{k}}x_{k} y_{k}(2-h\ell_{k} )\\
    &=x_ky_k+h{|\ell_{k}|}x_{k} y_{k}(2+h|\ell_{k}| )\\
    &\ge x_ky_k+2h|{\ell_{k}}|x_{k} y_{k}.
\end{align*}
We claim that when $u_k^\top u_k\le u_0^\top u_0$ and $0<x_ky_k<1$, we have $u_n^\top u_n\le u_0^\top u_0+C_1$ for all $n\ge k$ and for constant $\frac{10}{3}<C_1\le \frac{7}{2}$. Otherwise, consider $N$ s.t. $u_n^\top u_n\le u_0^\top u_0+C_1$ for $k\le n\le N$ and $u_{N+1}^\top u_{N+1}\ge u_0^\top u_0+C_1$. If $0<x_ny_n<1$, 
\begin{align*}
    r_n&=1-h\frac{\ell_n}{x_ny_n-1}(u_n^\top u_n-h\ell_n x_n y_n)\\
    &\ge 1-h(u_n^\top u_n-h\ell_n x_n y_n)\\
    &\ge 1-\frac{C}{4(u_0^\top u_0+4)}(u_n^\top u_n-h\ell_n x_n y_n)\\
    &\ge 1-\frac{u_0^\top u_0+C_1+h|\ell_n| x_n y_n}{u_0^\top u_0+4}\\
    &\ge 1-\frac{u_0^\top u_0+15/4}{u_0^\top u_0+4}>0
\end{align*}
where the first inequality follows from $q_n\le 1$ for $0<x_ny_n<1$; the second inequality follows from the initial condition and the requirement of $h$; the last inequality follows from $h|\ell_n|x_ny_n< \frac{1}{4}$ for $0<x_ny_n<1$ (can be easily checked from the initial condition that $h\le \frac{1}{8}$, and the rest follows from analyzing the expression of the function).
Therefore, $0<x_{n+1}y_{n+1}<1$ and consequently, $0<r_n<1$ for $n=k,\cdots,N$. Especially, consider $n=N$. Then iteratively from the lower bound of $x_{N+1}y_{N+1}$ above, we have
$$\sum_{i=k}^N2h|{\ell_{i}}|x_{i} y_{i}<1$$ and thus $$\sum_{i=k}^N(4+Cu_i^\top u_i/(b u_0^\top u_0))h|\ell_i|x_iy_i\le (4+C(u_0^\top u_0+C_1)/(b u_0^\top u_0))\sum_{i=k}^N h|\ell_i|x_iy_i<\frac{20}{3}\sum_{i=k}^N h|\ell_i|x_iy_i<\frac{10}{3}< C_1.$$ Contradiction.

Then, we have $0<r_n<1$ for all $n\ge k$. Take $k=1$, and then we have the monotone decreasing of $1-x_ny_n$. Also,
\begin{align*}
    r_n&=1-h\frac{\ell_n}{x_ny_n-1}(u_n^\top u_n-h\ell_n x_n y_n)\\
    &\le 1-h q_1 (2x_ny_n-h\ell_n x_ny_n) \\
    &\le 1-h q_1 x_1y_1 (2-h\ell_n )\\
    &\le 1-h(2-h) q_1 x_1y_1 <1.
\end{align*}
Thus, GD will converge to $xy=1$.

\end{proof}

\subsection{Proof of non-EoS}

\begin{proof}[Proof of Theorem~\ref{thm:non_eos_bad_regularity}]

From the proof of convergence, we know $r_k>0$ for $k\ge 1$. Thus when $x_ky_k$ is very close to 1, $r_k>0$. Take the limit and we have
\begin{align*}
    \lim_{k\to\infty} r_k=1-hu_\infty^\top u_\infty\ge0.
\end{align*}
Therefore,
\begin{align*}
    S_\infty=u_\infty^\top u_\infty\le \frac{1}{h}
\end{align*}

\end{proof}

\subsection{Proof of non-balancing}

Below is a more detailed version of Theorem~\ref{thm:no_balancing_bad_regu}.
\begin{theorem}[no Balancing, formal version]
\label{thm:no_balancing_bad_regu_Formal_version}

Consider ${b}=2n+1\text{ for }n\in\mathbb{Z}$. Assume the initial condition satisfies $$(x_0,y_0)\in \{(x,y):xy>4^{\frac{1}{2b-2}}, x^2+y^2\ge 4 \}\backslash \mathcal{B}_b\text{, where }\mathcal{B}_b\text{ is a Lebesgue measure-0 set.}$$ Let the learning rate be $$h= \frac{C}{(x_0^2+y_0^2+4)(x_0 y_0)^{2b-2}}\text{ for } 2\le C\le M_3(b,x_0,y_0)\text{, where }3<M_3(b,x_0,y_0)\le 4,$$ and the precise definition of $M_3$ is given in Theorem~\ref{thm:bad_regularity_convergence}. 
Then GD converges to a global minimum  $(x_\infty,y_\infty)$, and we have
\begin{align*}
    (x_\infty-y_\infty)^2\ge
    \begin{cases}
       (x_0-y_0)^2-\frac{2C}{4b-C}(x_0y_0-1),& \text{ if }r_0>0,\\
       (x_0-y_0)^2+2(x_0y_0-1)-\frac{2C}{b}x_0y_0,& \text{otherwise}
    \end{cases}
\end{align*}

\end{theorem}

\begin{proof}[Proof of Theorem~\ref{thm:no_balancing_bad_regu_Formal_version} (and Theorem~\ref{thm:no_balancing_bad_regu})]

Let $h=\frac{C}{(x_0^2+y_0^2+4)(x_0y_0)^{2b-2}}$. Before the proof, let first consider
\begin{align*}
    h\ell_0=\frac{C}{(x_0^2+y_0^2+4)(x_0y_0)^{2b-2}}\frac{(x_0 y_0)^{b-1} \left((x_0 y_0)^b-1\right)}{b}\le \frac{Cx_0y_0}{b(x_0^2+y_0^2+4)}\le \frac{C}{2b}.
\end{align*}

If $r_0>0$, from the proof of convergence, we know: $r_k>0$ for all $k$, and $\ell_k\ge 0$ for all $k$. Moreover, we have $x_ky_k\ge x_{k+1}y_{k+1}$, and consequently $\ell_k\le \ell_0$ for all $k$ (by the monotone decreasing of this function; see details in the proof of convergence). Then
\begin{align*}
    x_{k+1}y_{k+1}
    &=x_ky_k-h{\ell_{k}}(u_{k}^\top u_{k}-h\ell_{k} x_{k} y_{k})\le x_ky_k-h{\ell_{k}}x_{k} y_{k}(2-h\ell_{k} )\\
    &\le x_ky_k-h{\ell_{k}}x_{k} y_{k}(2-\frac{C}{2b} )
\end{align*}
Therefore
\begin{align*}
    x_0y_0-1=\sum_{k=0}^\infty x_ky_k-x_{k+1}y_{k+1}\ge \sum_{k=0}^\infty h{\ell_{k}}x_{k} y_{k}(2-\frac{C}{2b}).
\end{align*}

Also, we have
\begin{align*}
\label{appeqn:bad_regu_balancing_u_onestep_decrease}
    u_{k+1}^\top u_{k+1}&=(1+h^2\ell_{k}^2)u_{k}^\top u_{k}-4h\ell_{k} x_{k} y_{k}\\
    &\ge u_k^\top u_k-4h\ell_k x_ky_k
\end{align*}
Then 
\begin{align*}
    u_0^\top u_0-u_\infty^\top u_\infty=\sum_{k=0}^\infty u_{k}^\top u_k-u_{k+1}^\top u_{k+1}\le \sum_{k=0}^\infty4h\ell_k x_ky_k\le \frac{8b}{4b-C}(x_0y_0-1)\le \frac{2b}{b-1}(x_0y_0-1)
\end{align*}
Then
\begin{align*}
    (x_\infty-y_\infty)^2=u_\infty^\top u_\infty-2x_\infty y_\infty
    &\ge u_0^\top u_0-2x_\infty y_\infty-\frac{8b}{4b-C}(x_0y_0-1)\\
    &=u_0^\top u_0-2x_0y_0+2x_0y_0-2-\frac{8b}{4b-C}(x_0y_0-1)\\
    &= (x_0-y_0)^2-\frac{2C}{4b-C}(x_0y_0-1).
\end{align*}

If $r_0<0$, from the proof of convergence, we have: $r_k>0$ for $k\ge 1$, and $u_{k+1}^\top u_{k+1}\ge u_k^\top u_k$ for $k\ge 1$. Thus 
\begin{align*}
    u_0^\top u_0-u_\infty^\top u_\infty\le u_0^\top u_0-u_1^\top u_1\le 4h\ell_0x_0y_0\le \frac{2C}{b}x_0y_0
\end{align*}
Then
\begin{align*}
    (x_\infty-y_\infty)^2=u_\infty^\top u_\infty-2x_\infty y_\infty
    &\ge u_0^\top u_0-2x_\infty y_\infty-\frac{2C}{b}x_0y_0\\
    &=u_0^\top u_0-2x_0y_0+2x_0y_0-2-\frac{2C}{b}x_0y_0\\
    &= (x_0-y_0)^2+2(x_0y_0-1)-\frac{2C}{b}x_0y_0.
\end{align*}

\end{proof}

\subsection{Supplementary lemmas}

\begin{lemma}
\label{lem:M3>3}
Under the assumption of Theorem~\ref{thm:bad_regularity_convergence}, $M_3>3$.
\end{lemma}
\begin{proof}
    By the assumption, we have
\begin{align*}
    &\left(1+\left(\frac{C}{(x_0^2+y_0^2+4)(x_0 y_0)^{2b-2}}\right)^2\ell_0^2\right)x_0 y_0-\frac{C}{(x_0^2+y_0^2+4)(x_0 y_0)^{2b-2}}\ell_0 (x_0^2+y_0^2)\\
    &=\left(1+\left(\frac{C \left((x_0 y_0)^b-1\right)}{b(x_0^2+y_0^2+4)(x_0 y_0)^{b-1}}\right)^2\right)x_0 y_0-\frac{C \left((x_0 y_0)^b-1\right)}{b(x_0^2+y_0^2+4)(x_0 y_0)^{b-1}} (x_0^2+y_0^2)\\
    &\ge x_0y_0-\frac{C(x_0^2+y_0^2) \left((x_0 y_0)^b-1\right)}{b(x_0^2+y_0^2+4)(x_0 y_0)^{b-1}} 
\end{align*}
Since $b\ge 3$, when $C=3$, we have
    $$x_0y_0-\frac{C(x_0^2+y_0^2) \left((x_0 y_0)^b-1\right)}{b(x_0^2+y_0^2+4)(x_0 y_0)^{b-1}} >0.$$
\end{proof}

\begin{lemma}[stepsize]
\label{lem:bad_regu_lr_2/L_4/L}
    Let $C_1$=4. When $x_0^2+y_0^2\ge4 \left(\sqrt{2}+2\right)C_1$, the learning rate bound in Theorem~\ref{thm:bad_regularity_convergence}
    $$\frac{2}{(u_0^\top u_0+C_1)(x_0 y_0)^{2b-2}} \le h\le \frac{M_3(x_0,y_0)}{(u_0^\top u_0+C_1)(x_0 y_0)^{2b-2}}$$
     satisfies 
    $$\frac{2}{S_0}<\frac{2}{(u_0^\top u_0+C_1)(x_0 y_0)^{2b-2}},\text{ and }  \frac{4}{S_0}\le\frac{M_3(x_0,y_0)}{(u_0^\top u_0+C_1)(x_0 y_0)^{2b-2}}.$$

    \end{lemma}
\begin{proof}
    Consider the Hessian $\nabla^2f(x,y)$. The trace is
\begin{align*}
\Tr\nabla^2 f(x,y)&= (1 + (1 - 1/b) (1 - 1/(x y)^b)) (x^2 + y^2) (x y)^{2 b - 2}\\
&>(1 + (1 - 1/b) (1 - \frac{1}{4^{\frac{b}{2b-2}}}))(x^2 + y^2) (x y)^{2 b - 2}\\
&\ge\left(1+\frac{2}{3} \left(1-\frac{1}{2 \sqrt{2}}\right)\right)(x^2 + y^2) (x y)^{2 b - 2}\\
&> \frac{4}{3}(x^2 + y^2) (x y)^{2 b - 2},
\end{align*}
where the second inequality is via lower bounding of the expression at $b=3$.
    
   The determinant is 
    $$\det \nabla^2f(x,y)=-\frac{(x y)^{2 b} \left((x y)^b-1\right) \left(-(x y)^b+b \left(4 (x y)^b-2\right)+1\right)}{b^2 x^2 y^2}.$$ When $xy>1$, $\det\nabla^2f(x,y)<0$. Therefore, initial sharpness
    $$S_0>\Tr\nabla^2f(x_0,y_0).$$

   Then when $x_0^2+y_0^2\ge 4 \left(\sqrt{2}+2\right)C_1=16 \left(\sqrt{2}+2\right)$, we have
   $$  S_0> \left(1+\frac{2}{3} \left(1-\frac{1}{2 \sqrt{2}}\right)\right)(x_0^2 + y_0^2) (x_0 y_0)^{2 b - 2}\ge \frac{4}{3}(x_0^2 + y_0^2+C_1) (x_0 y_0)^{2 b - 2},$$
   and thus the lower bound of $h$ is greater than $\frac{2}{S_0}$, actually $\frac{8}{3S_0}$, and the upper bound of $h$ is greater than $\frac{4M_3}{3S_0}> \frac{4}{S_0}$ since $M_3>3$.
\end{proof}

\section{Additional experiments}
\label{app:experiments}

In this section, we provide the experiment details. 
We perform experiments on the CIFAR-10 and MNIST datasets with models of different regularities. Throughout the training process, we apply the full batch gradient descent with a group of constant learning rates to illustrate the presence/absence of large learning rate phenomena in different regularity settings. The algorithm does not use any weight decay or momentum.

\subsection{CIFAR-10}

To shorten the training time, we train our model on CIFAR-10-1k, which consists of $1000$ randomly chosen training samples from CIFAR-10. The data dimension is $N_0=3\times 32\times 32=3072$. The labels are transformed into their one-hot encoding, and the output dimension is $N_2=10$. 

We consider 2-layer neural networks in \eqref{eqn:neural_network_model}, i.e., $l=2$, with $N_1=200$ hidden neurons. Correspondingly, the weight matrices for the two layers are $W_1\in\RR^{N_0\times N_1}$ and $W_2\in\RR^{N_1\times N_2}$. We exclude any bias terms in our models. The loss function $\cL$ is chosen from $\ell^2$ (MSE) loss and the huber loss. The activation function $\sigma_1$ is chosen from $\mathrm{ReLU}$, $\tanh$ and $\mathrm{ReLU}^3$. 

For the weight initialization, we first use the default uniform initialization in PyTorch. Each weight in the $i$-th layer ($i=1,2$) follows the uniform distribution on $[-\frac{1}{\sqrt{N_{i-1}}}, \frac{1}{\sqrt{N_{i-1}}}]$. We then scale the weight matrices so that $\left\|{W_1}\right\|_{\rm F}=6$ and $\left\|{W_2}\right\|_{\rm F}=20$. 

We adopt the same Lanczos algorithm as in \cite{cohen2021gradient} to compute the leading eigenvalue of Hessian, i.e., the sharpness. For the settings with total epochs no greater than $20,000$, we compute the sharpness after every $40$ epochs. For those settings with total epochs greater than $20,000$, we compute the sharpness after every $400$ epochs.

The results of the above mentioned six cases are shown in Figure~\ref{figapp:huber_tanh}\ref{figapp:huber_relu}\ref{figapp:mse_tanh}\ref{figapp:mse_relu}\ref{figapp:mse_cubic_relu}\ref{figapp:huber_cubic_relu_2layer}. All the cases are consistent with our theory that bad regularity vanishes large learning rate phenomena, except for huber loss with $\mathrm{ReLU}^3$. 
The reason for this inconsistency potentially lies in the fact that the $\mathrm{ReLU}^3$ is only added to $W_1$, while $W_2$ is still linear under the corresponding objective function (and linear function, from our theory, has good regularity). The disagreement of the `order' of weights is beyond the scope of our theory and we will leave it for future exploration.

Nevertheless, to provide further support for our regularity theory, we consider the 3-layer model \eqref{eqn:dor_3layer_toy} with a fixed last layer for the combination of $\mathrm{ReLU}^3$ activation and huber loss. In this case, we add a fixed last layer $W_3\in\RR^{N_2\times N_2}$ whose entries are uniformly sampled from $\{+1, -1\}$. We set $\sigma_1$ be the linear activation, $\sigma_2=\mathrm{ReLU}^3$, and $\cL=\mathrm{Huber}$. To prevent the initial sharpness from being too large, we use a slightly smaller initialization where $\left\|{W_1}\right\|_{\rm F}=3$ and $\left\|{W_2}\right\|_{\rm F}=10$. The result is shown in Figure~\ref{figapp:huber_cubic_relu_3layer}, which is consistent with our claim.

For bad regularity models (degree greater than or equal to $2$ in Tab.~\ref{tab:dor_ml}), we also implement experiments with batch normalization to see if the large learning rate phenomena can be recovered. For the cases using 2-layer network (see Figure~\ref{figapp:mse_relu_batch_normalization}\ref{figapp:mse_cubic_relu_batch_normalization}), we add a standard \texttt{BatchNorm1d} layer in PyTorch between the first layer $W_1$ and the activation $\sigma_1$. We set \texttt{affine=False} so that the normalization layer would not contain any learnable parameters. 
For 3-layer network (see Figure~\ref{figapp:huber_cubic_relu_batch_normalization}), it is added between $W_2$ and $\sigma_2$.

\begin{figure}[ht]
    \centering
    \includegraphics[width=\textwidth]{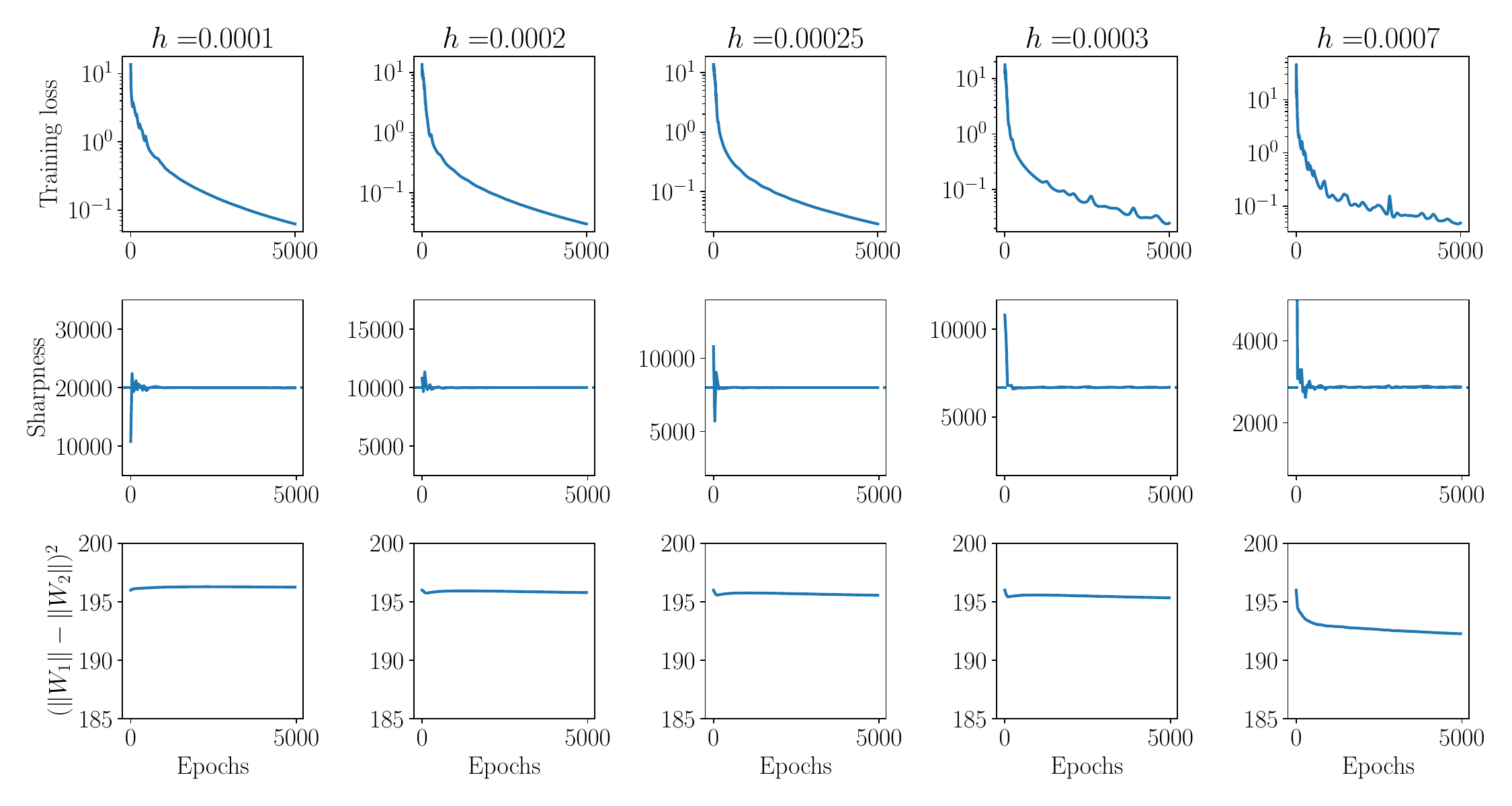}
    \caption{CIFAR-10: huber+tanh}
    \label{figapp:huber_tanh}
\end{figure}
\begin{figure}[ht]
    \centering
    \includegraphics[width=\textwidth]{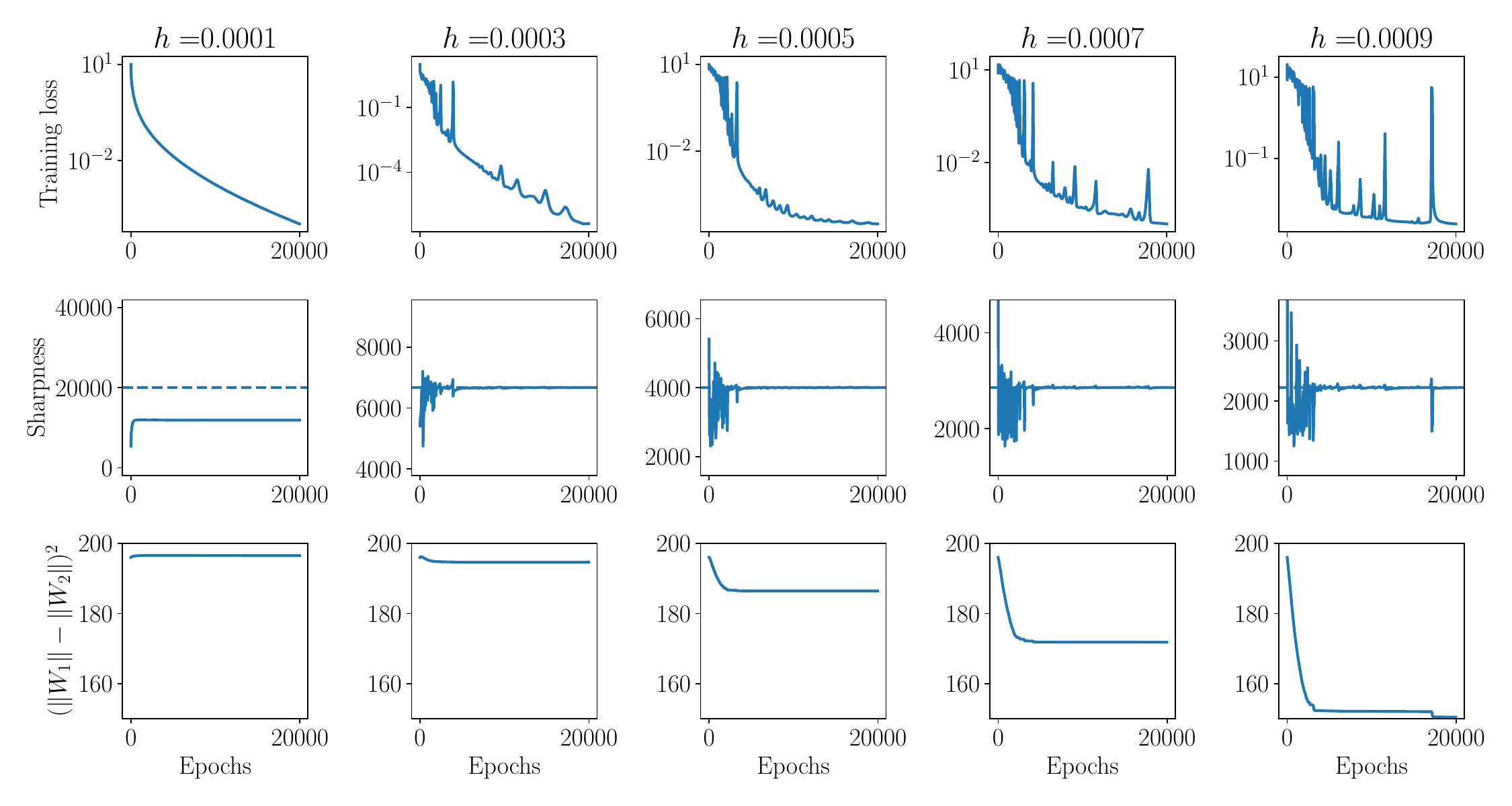}
    \caption{CIFAR-10: huber+ReLU}
    \label{figapp:huber_relu}
\end{figure}

\begin{figure}[ht]
    \centering
    \includegraphics[width=\textwidth]{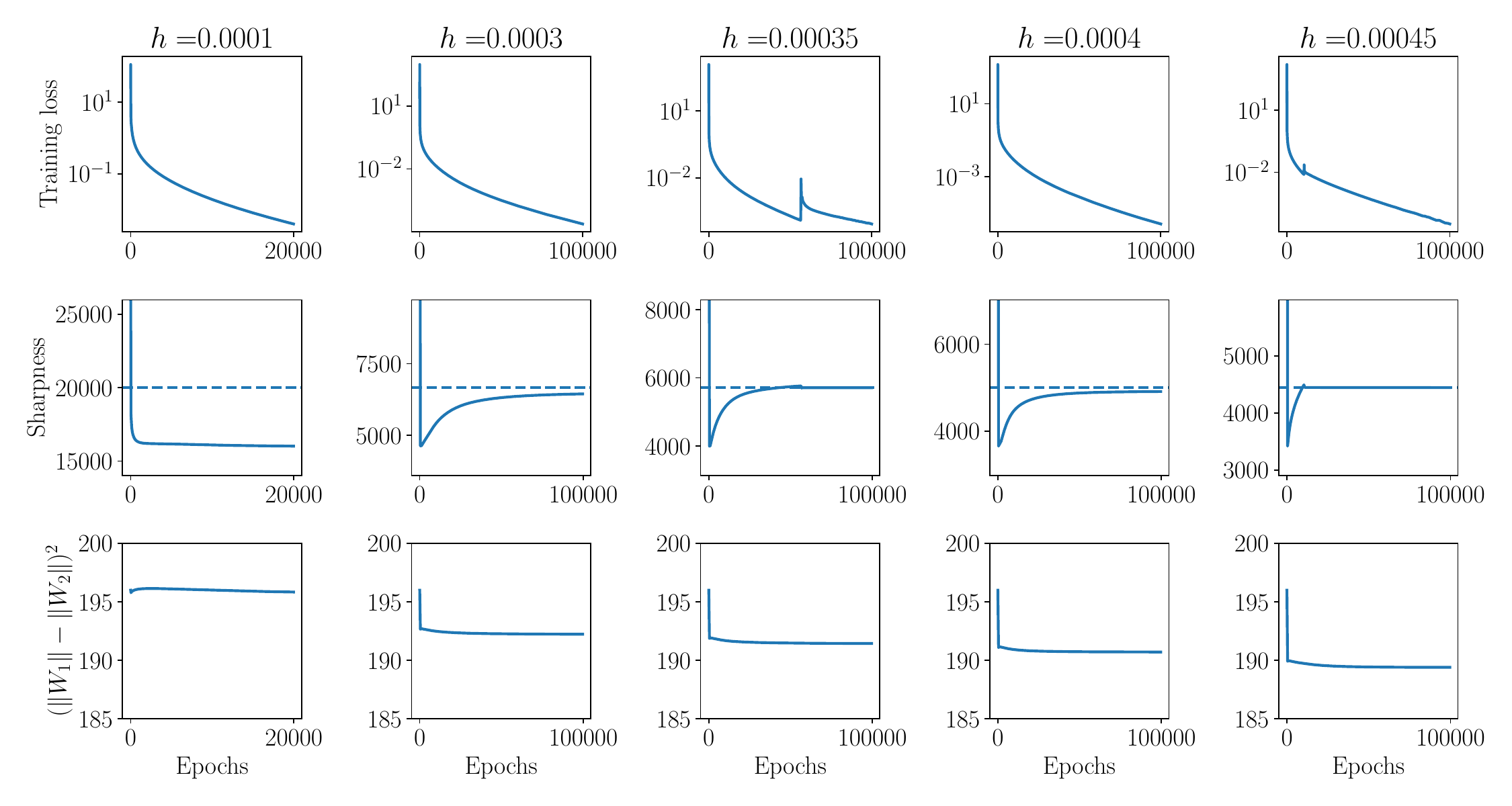}
    \caption{CIFAR-10: $\ell^2$+tanh}
    \label{figapp:mse_tanh}
\end{figure}
\begin{figure}[ht]
    \centering
    \includegraphics[width=\textwidth]{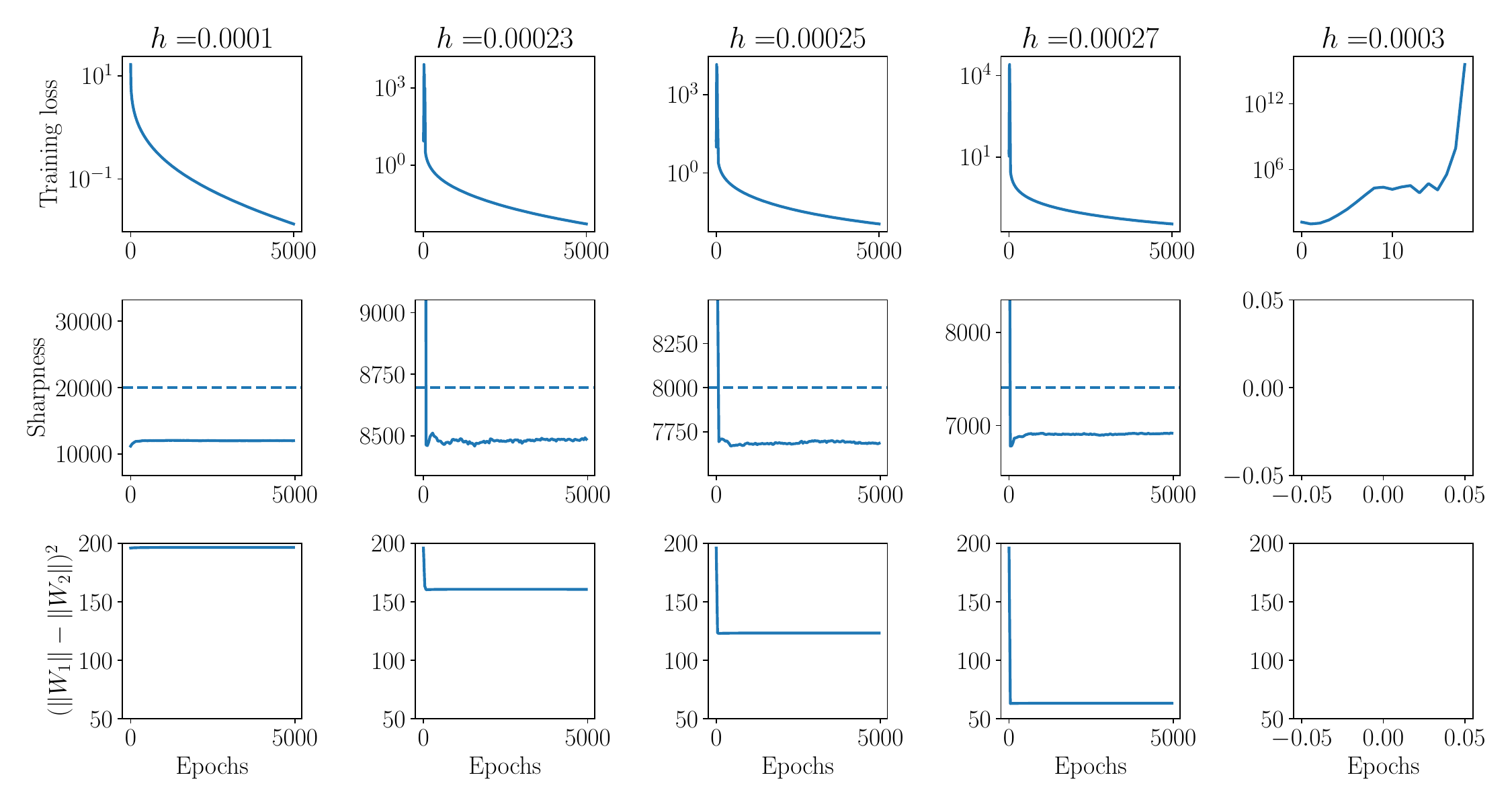}
    \caption{CIFAR-10: $\ell^2$+ReLU}
    \label{figapp:mse_relu}
\end{figure}
\begin{figure}[ht]
    \centering
    \includegraphics[width=\textwidth]{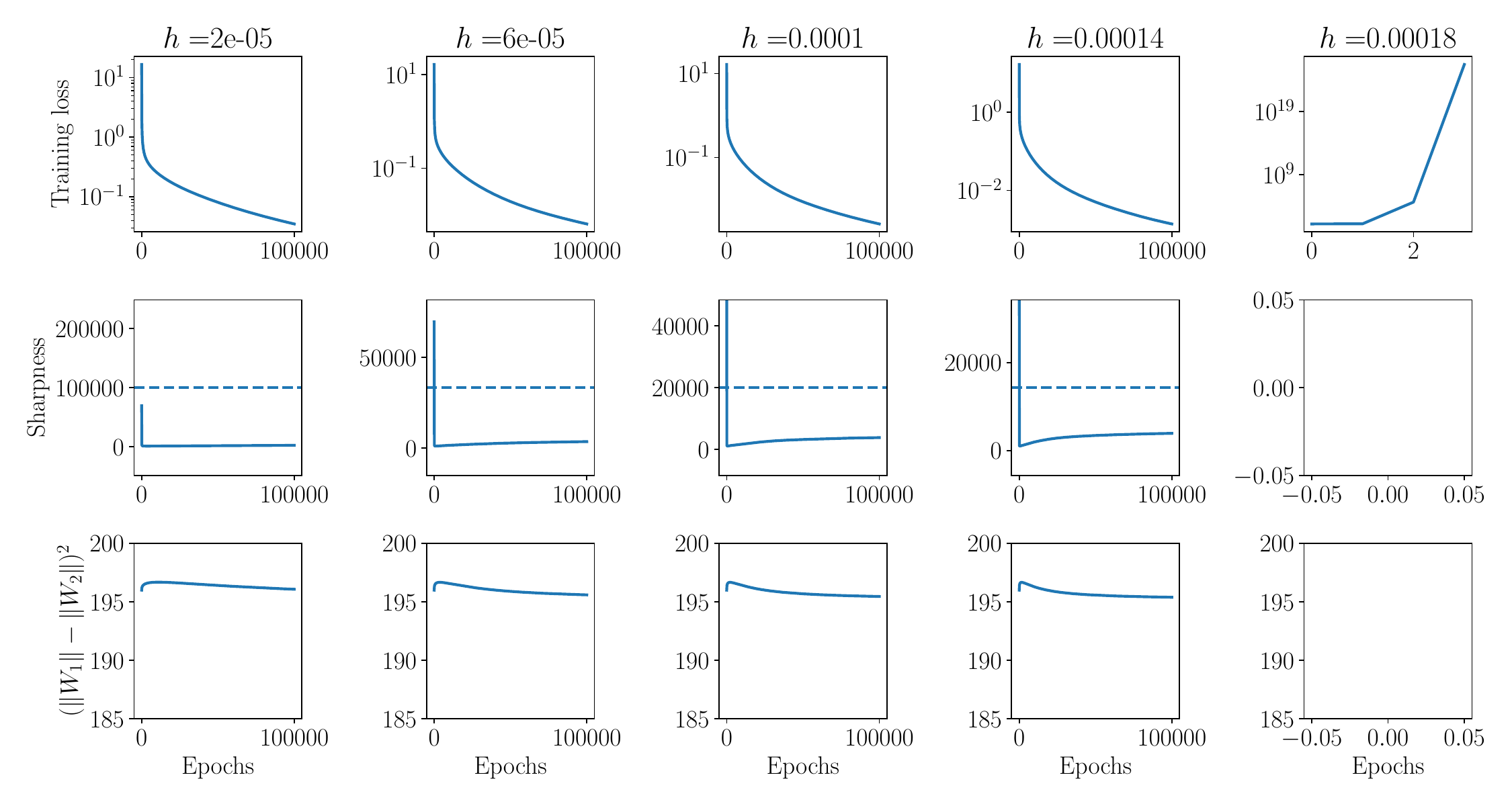}
    \caption{CIFAR-10: $\ell^2$+$\mathrm{ReLU}^3$}
    \label{figapp:mse_cubic_relu}
\end{figure}

\begin{figure}[ht]
    \centering
    \includegraphics[width=\textwidth]{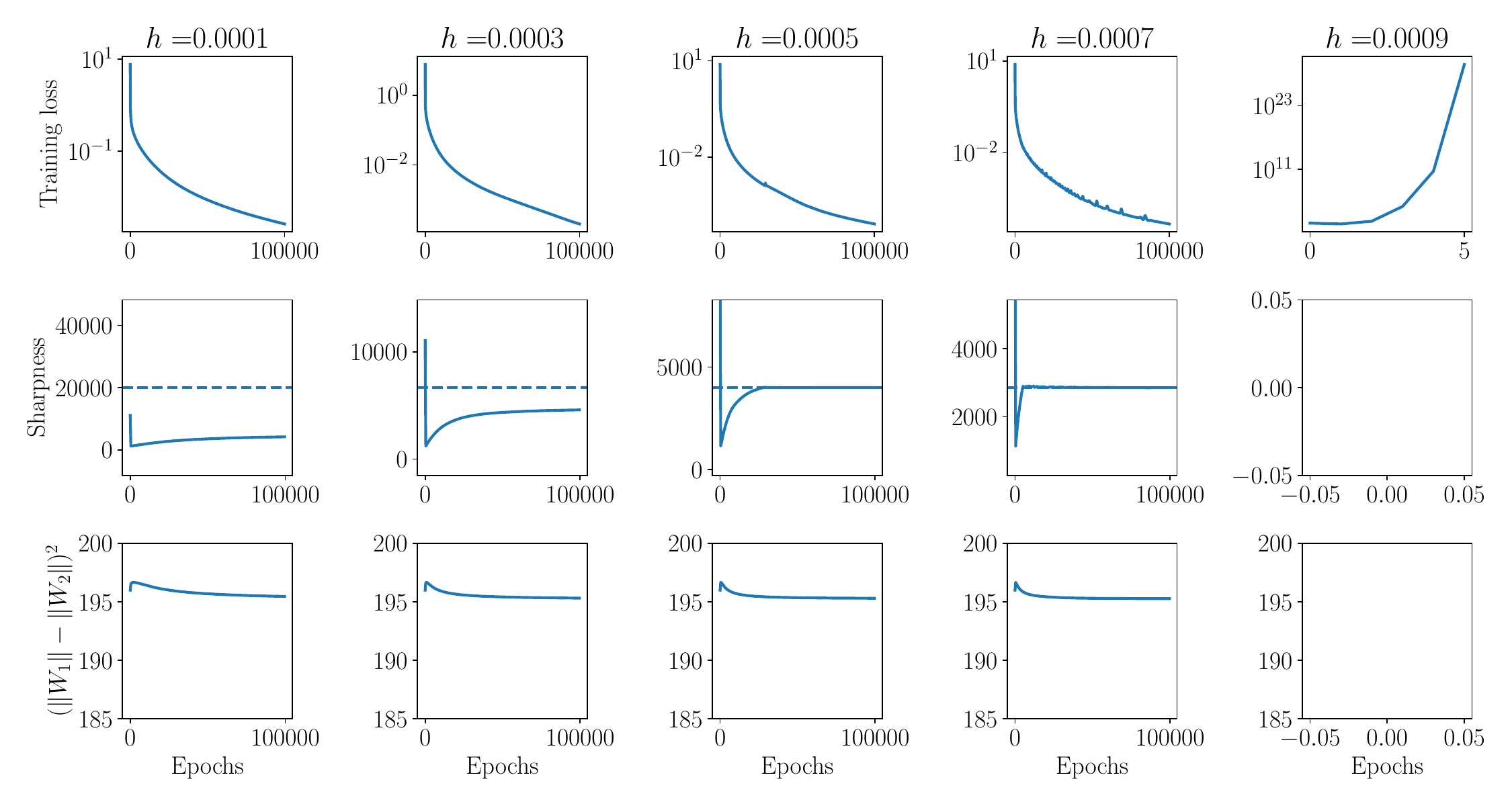}
    \caption{CIFAR-10: $\ell^2$+$\mathrm{ReLU}^3$, 2-layer network}
    \label{figapp:huber_cubic_relu_2layer}
\end{figure}
\begin{figure}[ht]
    \centering  \includegraphics[width=\textwidth]{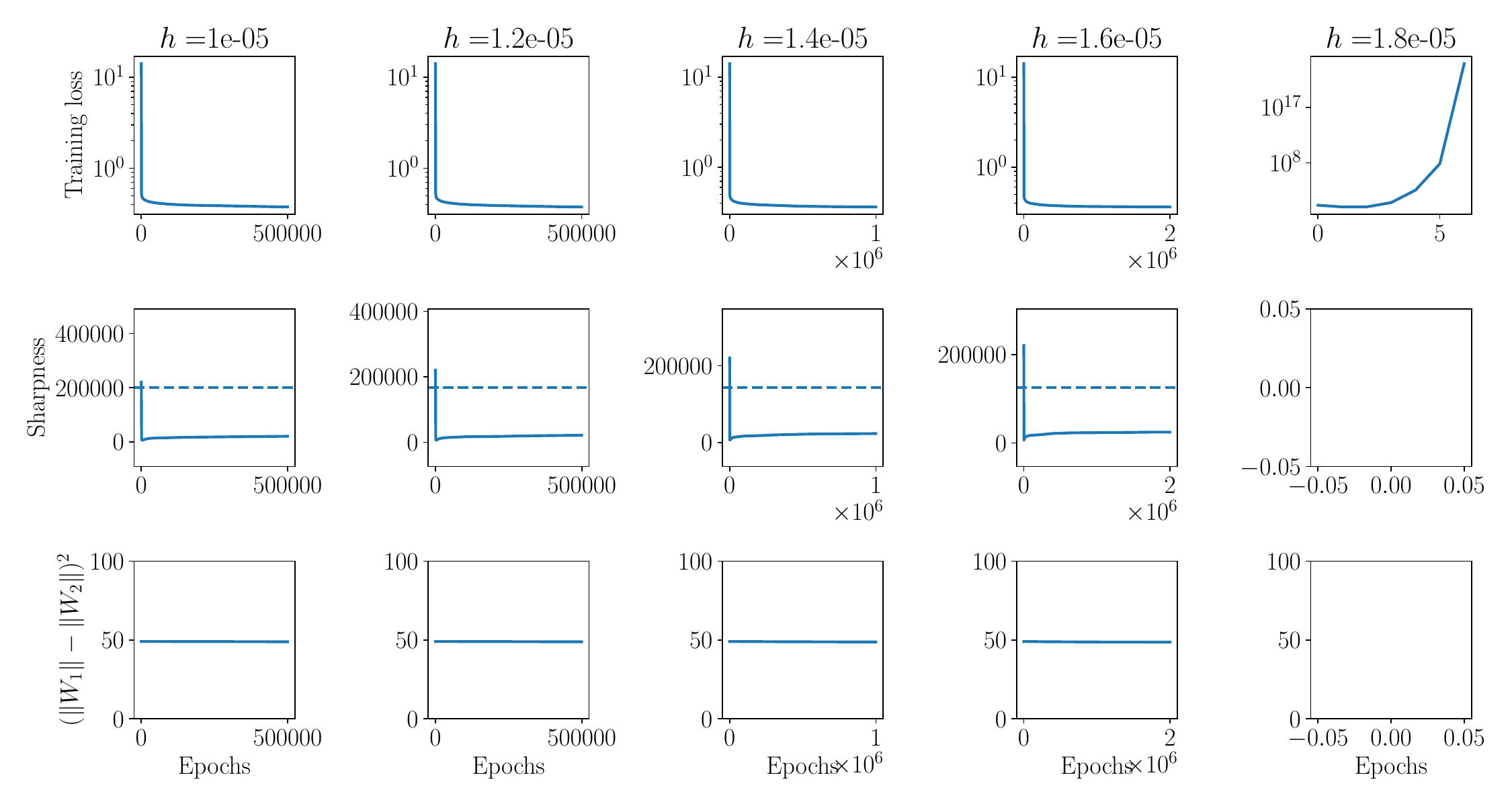}
    \caption{CIFAR-10: $\ell^2$+$\mathrm{ReLU}^3$, 3-layer network}
    \label{figapp:huber_cubic_relu_3layer}
\end{figure}

\begin{figure}[ht]
    \centering
    \includegraphics[width=\textwidth]{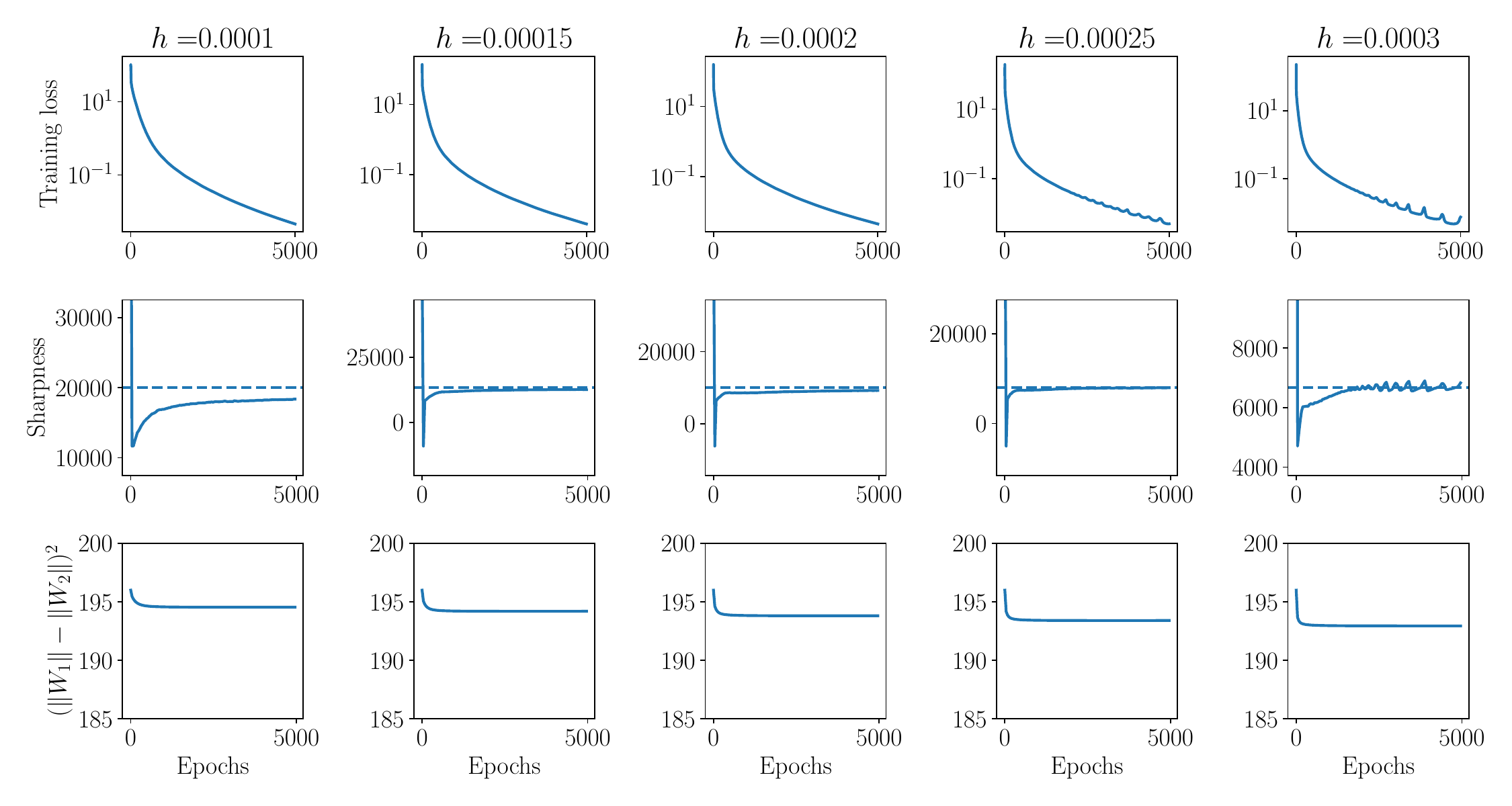}
    \caption{CIFAR-10: $\ell^2$+ReLU with batch normalization}
    \label{figapp:mse_relu_batch_normalization}
\end{figure}

\begin{figure}[ht]
    \centering
    \includegraphics[width=\textwidth]{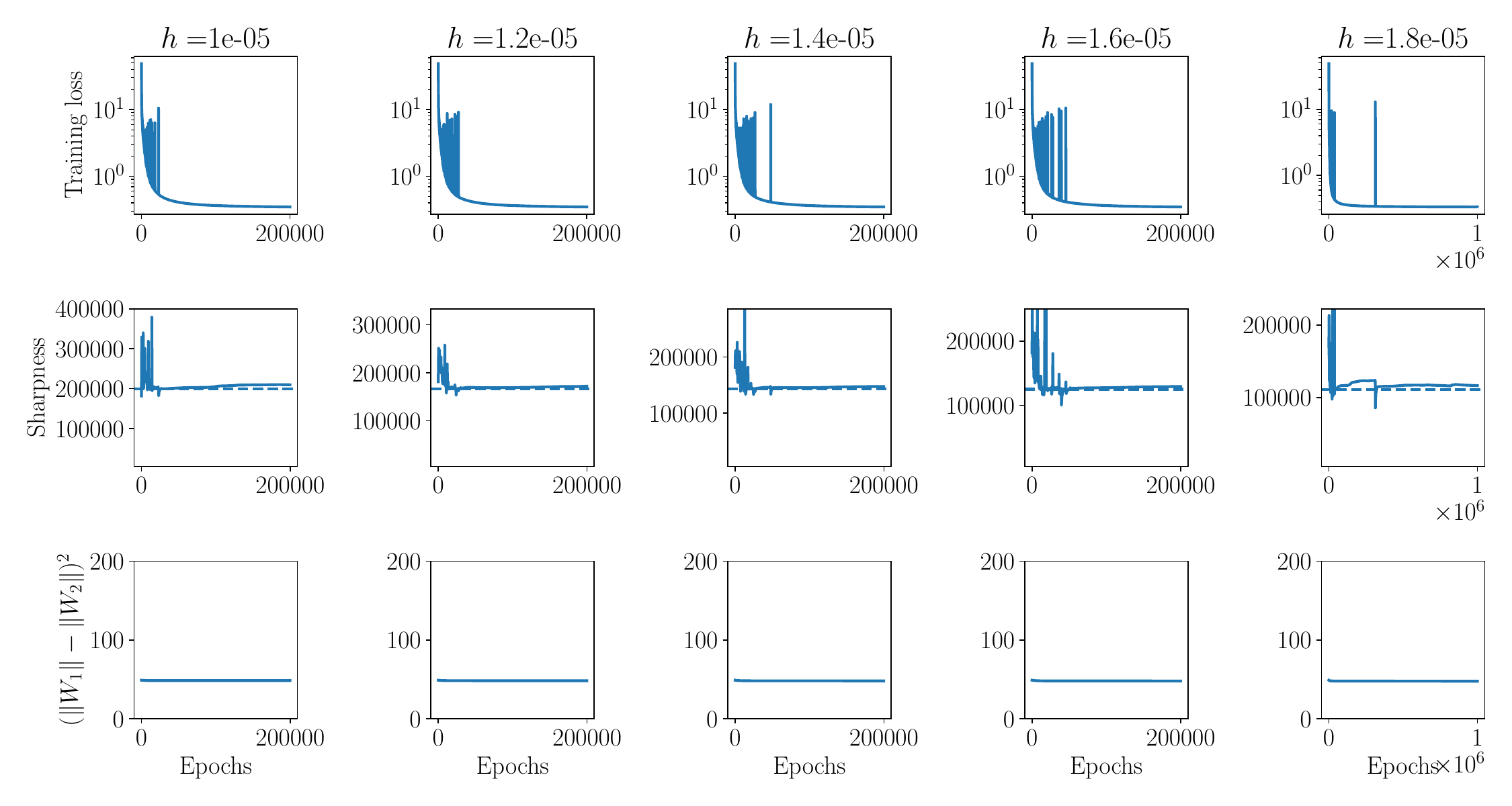}
    \caption{CIFAR-10: huber+$\mathrm{ReLU}^3$ (3-layer) with batch normalization}
    \label{figapp:huber_cubic_relu_batch_normalization}
\end{figure}

\begin{figure}[ht]
    \centering
    \includegraphics[width=\textwidth]{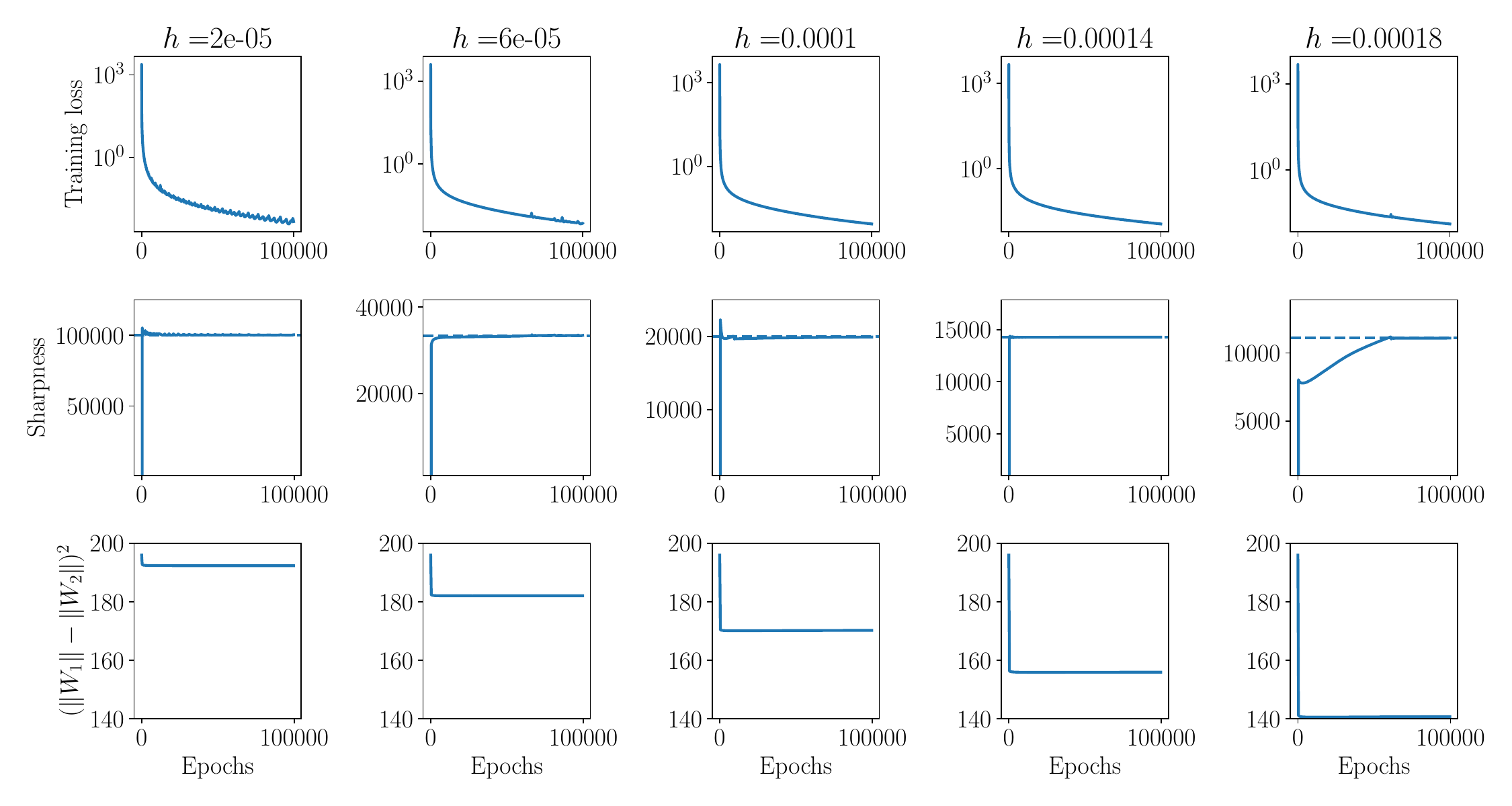}
    \caption{CIFAR-10: $\ell^2$+$\mathrm{ReLU}^3$ with batch normalization}
    \label{figapp:mse_cubic_relu_batch_normalization}
\end{figure}

\subsection{MNIST}

The setting in MNIST is similar to CIFAR-10. We again perform experiments on a $1000$ sample subset of MNIST. The input dimension now is $N_0=1\times 28\times 28=784$, and the output dimension is still $N_2=10$. 
For the weight initialization, we use $\left\|{W_1}\right\|_{\rm F}=7$ and $\left\|{W_2}\right\|_{\rm F}=40$ for huber+ReLU, $\left\|{W_1}\right\|_{\rm F}=10$ and $\left\|{W_2}\right\|_{\rm F}=40$ for all the other 2-layer models, and $\left\|{W_1}\right\|_{\rm F}=6$ and $\left\|{W_2}\right\|_{\rm F}=15$ for the 3-layer model with last layer fixed. 
All the cases (6 cases under 2-layer network, and huber+$\mathrm{ReLU}^3$ under 3-layer network) exhibit the same results on MNIST (see Figure~\ref{fig:mnist_huber_tanh}\ref{fig:mnist_huber_relu}\ref{fig:mnist_mse_tanh}\ref{fig:mnist_mse_relu}\ref{fig:mnist_huber_cubic_relu_2layer}\ref{fig:mnist_huber_cubic_relu_3layer}\ref{fig:mnist_mse_cubic_relu}). We also test batch normalization on an example of bad regularity, huber+$\mathrm{ReLU}^3$ (3-layer), and observe the restoration of large learning rate phenomena (see Figure~\ref{fig:mnist_huber_cubic_relu_BN}).

\begin{figure}
    \centering
    \includegraphics[width=0.8\textwidth]{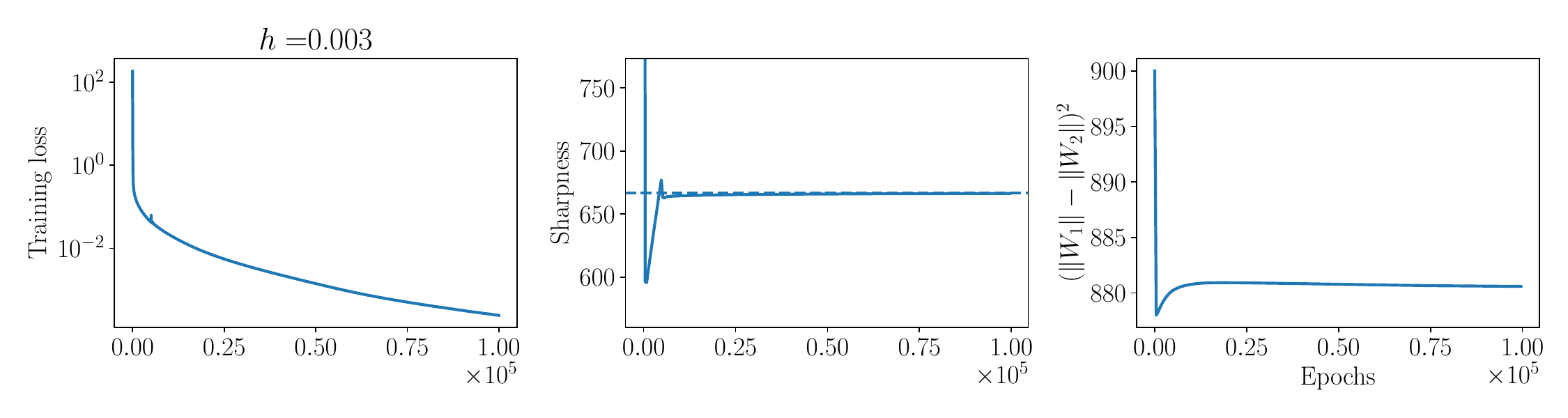}
    \caption{MNIST: huber+tanh}
    \label{fig:mnist_huber_tanh}
\end{figure}

\begin{figure}
    \centering
    \includegraphics[width=0.8\textwidth]{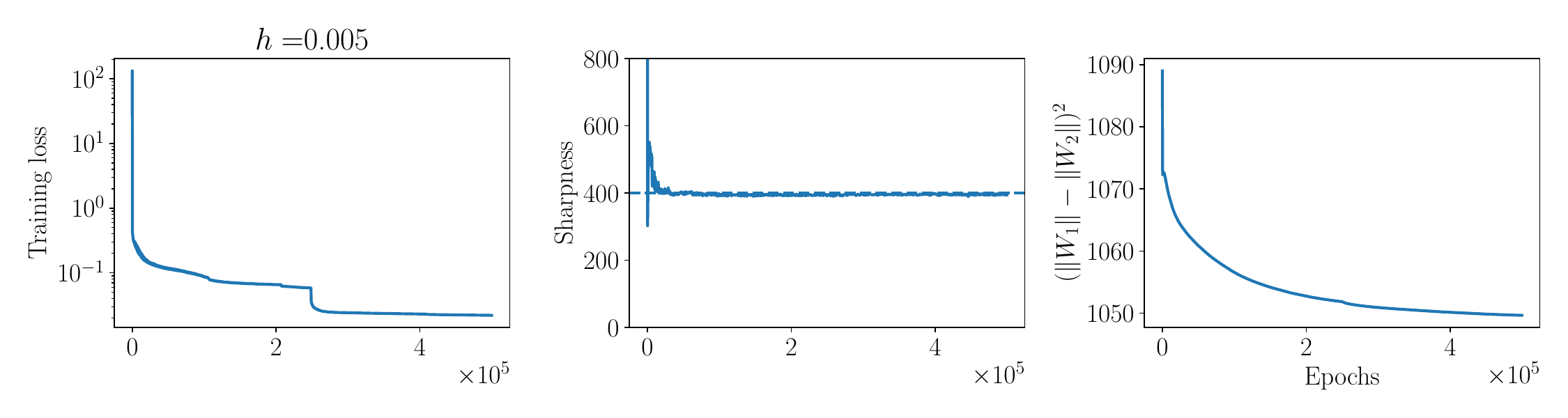}
    \caption{MNIST: huber+ReLU}
    \label{fig:mnist_huber_relu}
\end{figure}

\begin{figure}
    \centering
    \includegraphics[width=0.8\textwidth]{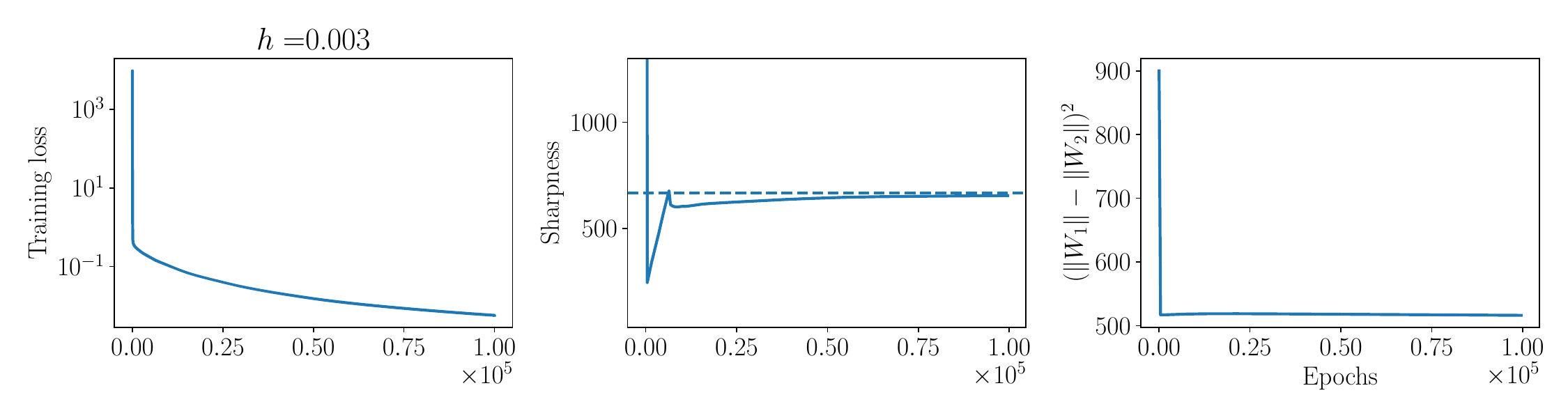}
     \caption{MNIST: $\ell^2$+tanh}
    \label{fig:mnist_mse_tanh}
\end{figure}

\begin{figure}
    \centering
    \includegraphics[width=0.8\textwidth]{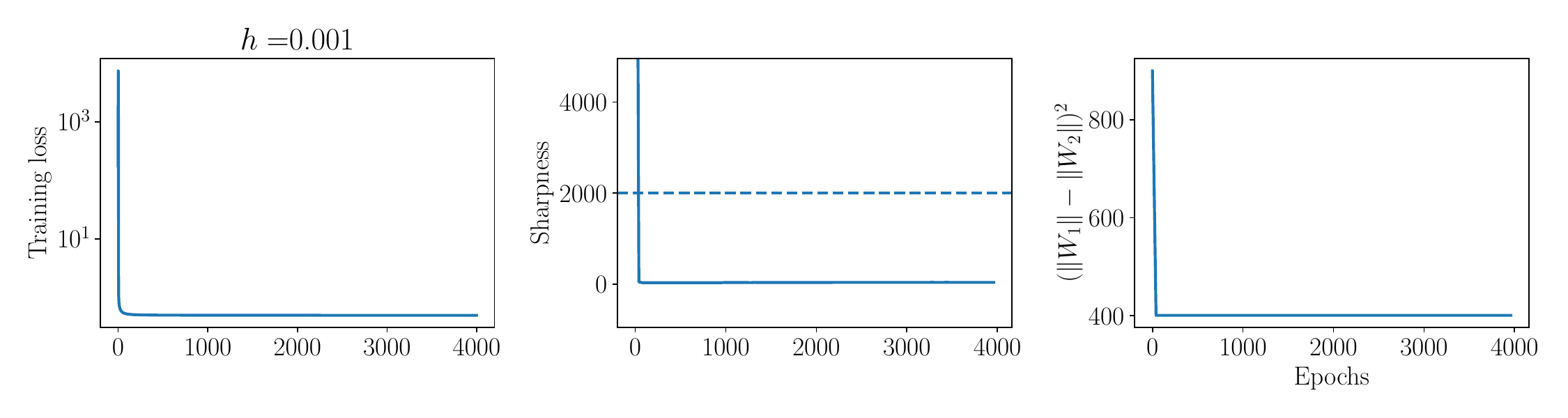}
    \caption{MNIST: $\ell^2$+ReLU}
    \label{fig:mnist_mse_relu}
\end{figure}

\begin{figure}
    \centering
    \includegraphics[width=0.8\textwidth]{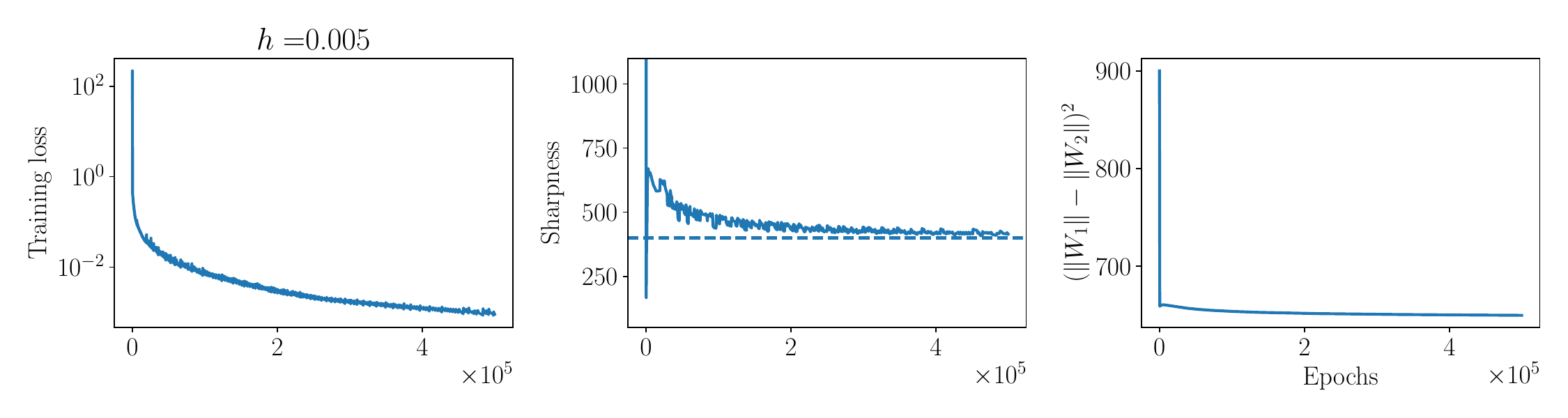}
    \caption{MNIST: huber+$\mathrm{ReLU}^3$ (2-layer)}
    \label{fig:mnist_huber_cubic_relu_2layer}
\end{figure}

\begin{figure}
    \centering
    \includegraphics[width=0.8\textwidth]{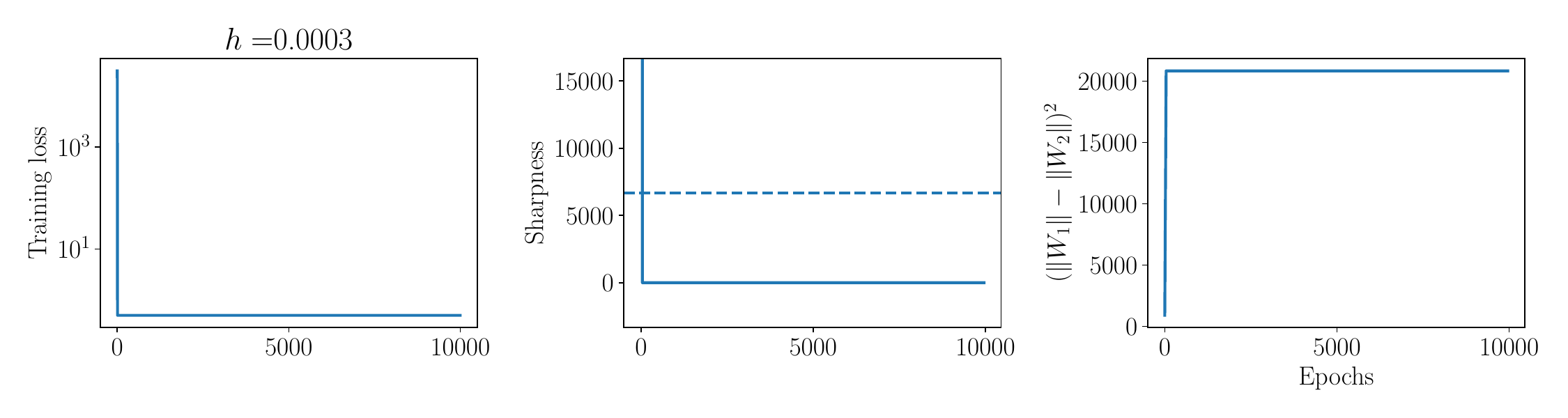}
    \caption{MNIST: huber+$\mathrm{ReLU}^3$ (3-layer)}
    \label{fig:mnist_huber_cubic_relu_3layer}
\end{figure}

\begin{figure}
    \centering
    \includegraphics[width=0.8\textwidth]{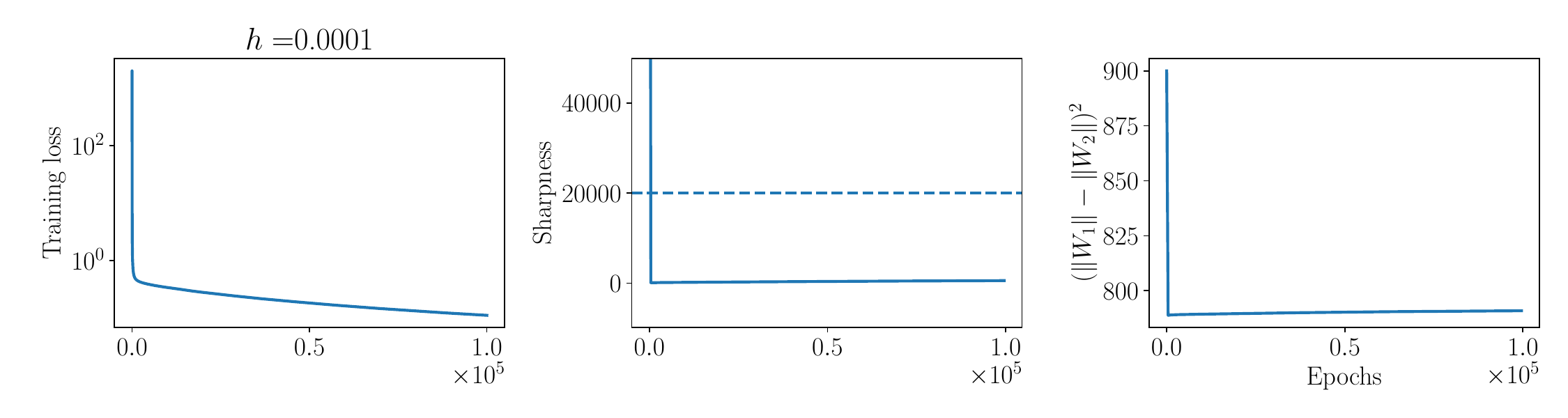}
    \caption{MNIST: $\ell^2$+$\mathrm{ReLU}^3$}
    \label{fig:mnist_mse_cubic_relu}
\end{figure}

\begin{figure}
    \centering
    \includegraphics[width=0.8\textwidth]{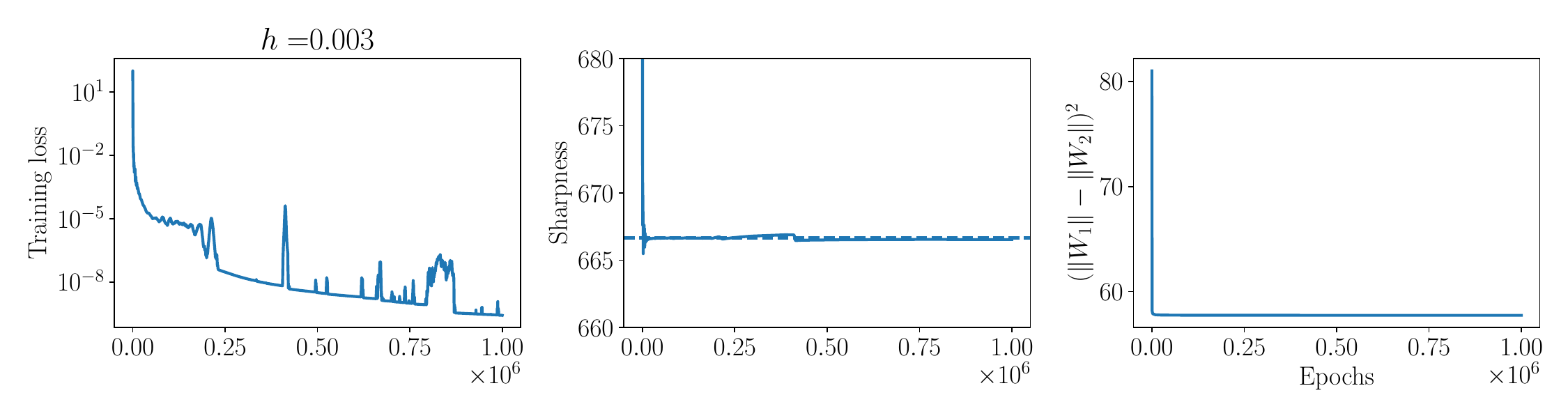}
    \caption{MNIST: huber+$\mathrm{ReLU}^3$ (3-layer) with batch normalization}
    \label{fig:mnist_huber_cubic_relu_BN}
\end{figure}

\end{document}